\documentclass[11pt]{article}

\usepackage{amsmath,amssymb}
\usepackage{amsthm}
\usepackage{mathtools}
\usepackage[noend]{algorithmic}
\usepackage[ruled,vlined]{algorithm2e}
\usepackage{setspace}
\usepackage{url}
\usepackage{fullpage}
\usepackage{makeidx}
\usepackage{enumerate}
\usepackage[top=1in, bottom=1.25in, left=1in, right=1in]{geometry}
\usepackage{graphicx,float,psfrag,epsfig,caption}
\usepackage[usenames,dvipsnames,svgnames,table]{xcolor}
\definecolor{darkgreen}{rgb}{0.0,0,0.9}
\usepackage[pagebackref,colorlinks=true,pdfpagemode=UseNone,citecolor=OliveGreen,linkcolor=BrickRed,urlcolor=BrickRed,pdfstartview=FitH]{hyperref}
\usepackage{epstopdf}
\usepackage{color}
\usepackage{xr}
\usepackage{subfig}
\usepackage{caption}
\usepackage{graphicx}
\usepackage[utf8]{inputenc}

\usepackage{cancel}

\usepackage{scalerel,stackengine}

\stackMath
\newcommand\reallywidehat[1]{%
\savestack{\tmpbox}{\stretchto{%
  \scaleto{%
    \scalerel*[\widthof{\ensuremath{#1}}]{\kern.1pt\mathchar"0362\kern.1pt}%
    {\rule{0ex}{\textheight}}
  }{\textheight}%
}{2.4ex}}%
\stackon[-6.9pt]{#1}{\tmpbox}%
}
\usepackage[mathscr]{euscript}

\DeclareSymbolFont{rsfs}{U}{rsfs}{m}{n}
\DeclareSymbolFontAlphabet{\mathscrsfs}{rsfs}

\numberwithin{equation}{section}

\newtheoremstyle{myexample} 
    {\topsep}                    
    {\topsep}                    
    {\rm }                   
    {}                           
    {\bf }                   
    {.}                          
    {.5em}                       
    {}  

\newtheoremstyle{myremark} 
    {\topsep}                    
    {\topsep}                    
    {\rm}                        
    {}                           
    {\bf}                        
    {.}                          
    {.5em}                       
    {}  

\newtheorem{claim}{Claim}[section]
\newtheorem{lemma}[claim]{Lemma}

\newtheorem{theorem}{Theorem}
\newtheorem{proposition}[claim]{Proposition}
\newtheorem{corollary}[claim]{Corollary}
\newtheorem{definition}[claim]{Definition}

\theoremstyle{myremark}
\newtheorem{remark}{Remark}[section]

\theoremstyle{myremark}

\theoremstyle{myexample}
\newtheorem{example}[remark]{Example}

\definecolor{darkgreen}{rgb}{0.0, 0.5, 0.0}

\newcommand{\bea}{\begin{eqnarray}}
\newcommand{\eea}{\end{eqnarray}}
\newcommand{\<}{\langle}
\renewcommand{\>}{\rangle}
\newcommand{\E}{{\mathbb E}}

\def\fr{\frac}
\def\fr12{\frac{1}{2}}

\def\lt{\left}
\def\rt{\right}
\def\sval{\mbox{{\rm\tiny val}}}
\def\strain{\mbox{{\rm\tiny train}}}
\def\ssurr{\mbox{{\rm\tiny surr}}}

\def\tsurr{{\sf surr}}
\def\tbias{{\sf bias}}

\def\aperp{\alpha_{\perp}}

\def\beperp{\beta_{\perp}}
\def\R{{\mathbb R}}
\def\obR{\overline{\mathbb R}}

\def\sP{{\sf P}}

\def\surr{{\mbox{\tiny\rm su}}}
\def\srand{{\mbox{\tiny\rm rand}}}

\def\sls{{\mbox{\tiny\rm ls}}}
\def\sMM{{\mbox{\tiny\rm MM}}}

\def\spn{{\rm span}}

\def\bG{{\boldsymbol G}}

\def\bX{{\boldsymbol X}}

\def\tbG{\tilde{\boldsymbol G}}
\def\tbH{\tilde{\boldsymbol H}}
\def\opi{\overline{\pi}}
\def\sF{{\sf F}}
\def\ophi{\overline{\varphi}}

\def\sTV{\mbox{\tiny \rm TV}}

\def\snew{\mbox{\tiny \rm new}}
\def\sexc{\mbox{\tiny \rm exc}}
\def\stest{\mbox{\tiny \rm test}}
\def\sunb{\mbox{\tiny \rm unb}}
\def\snr{\mbox{\tiny \rm nr}}

\def\Unif{{\sf Unif}}

\def\eps{{\varepsilon}}

\def\id{{\boldsymbol{I}}}
\def\bh{\boldsymbol{h}}

\def\cuA{\mathscrsfs{A}}

\def\cuP{\mathscrsfs{P}}
\def\cuK{\mathscrsfs{K}}

\def\cuL{\mathscrsfs{L}}
\def\hcuL{\widehat{\mathscrsfs{L}}}
\def\cuN{\mathscrsfs{N}}
\def\cuS{\mathscrsfs{S}}

\def\ocuA{\overline{\mathscrsfs{A}}}
\def\cuU{\mathscrsfs{U}}

\def\S{{\mathbb S}}

\def\sQ{{\sf Q}}

\def\hbxi{\hat{\boldsymbol{\xi}}}
\def\hbtheta{\hat{\boldsymbol{\theta}}}
\def\surrth{\hat{\boldsymbol{\theta}}^{\mbox{\tiny\rm su}}}
\def\btheta{{\boldsymbol{\theta}}}

\def\bxi{{\boldsymbol{\xi}}}

\def\oR{\overline{R}}

\def\bSigma{{\boldsymbol{\Sigma}}}

\def\bP{{\boldsymbol{P}}}

\def\bA{{\boldsymbol{A}}}
\def\bu{{\boldsymbol{u}}}

\def\bh{{\boldsymbol{{h}}}}

\def\bQ{{\boldsymbol{Q}}}
\def\bM{{\boldsymbol{M}}}

\def\bg{{\boldsymbol{g}}}

\def\bzero{{\mathbf 0}}

\def\cF{{\mathcal F}}
\def\cG{{\mathcal G}}

\def\cC{{\mathcal C}}

\def\op{\mbox{\tiny\rm op}}

\def\naturals{{\mathbb N}}
\def\reals{{\mathbb R}}

\def\normal{{\sf N}}

\def\sT{{\sf T}}

\def\bv{{\boldsymbol{v}}}
\def\bz{{\boldsymbol{z}}}
\def\bx{{\boldsymbol{x}}}
\def\ba{{\boldsymbol{a}}}

\def\bs{{\boldsymbol{s}}}

\def\bA{\boldsymbol{A}}
\def\bB{\boldsymbol{B}}

\def\bH{\boldsymbol{H}}
\def\hbH{\boldsymbol{\hat{H}}}

\def\de{{\rm d}}

\def\bX{\boldsymbol{X}}

\def\bW{\boldsymbol{W}}
\def\prob{{\mathbb P}}
\def\E{{\mathbb E}}

\def\<{\langle}
\def\>{\rangle}
\def\Tr{{\sf Tr}}

\def\Ball{{\sf B}}

\def\sign{{\rm sign}}

\def\diag{{\rm diag}}

\def\hR{\hat{R}}

\def\by{{\boldsymbol{y}}}

\def\cW{{\mathcal W}}
\def\obW{\overline{\boldsymbol W}}

\def\hp{\hat{p}}

\def\bu{{\boldsymbol{u}}}
\def\b0{{\boldsymbol{0}}}

\def\bfone{{\boldsymbol 1}}

\DeclareMathOperator*{\plim}{p-lim}

\def\bdelta{{\boldsymbol \delta}}

\def\cS{{\mathcal S}}

\def\cB{{\mathcal B}}

\def\argmin{{\rm arg\,min}}

\def\ess{{\rm ess}}

\def\balpha{{\boldsymbol \alpha}}
\def\balphap{\overline{\boldsymbol\alpha}_{\perp}}
\def\alphap{{\alpha}_{\perp}}

\def\err{{\sf err}}

\def\bfzero{\boldsymbol{0}}

\def\bbeta{{\boldsymbol\beta}}

\def\pr{{\pi}}

\def\rcoeff{\rho}

\def\code#1{\texttt{#1}}

\title{Towards a statistical theory of data selection under weak supervision}

\author{Germain Kolossov${}^*$ \and Andrea Montanari${}^*$ \and 
Pulkit Tandon\thanks{Granica Computing Inc. --- \href{www.granica.ai}{granica.ai}}}

\begin{document}

\maketitle

\begin{abstract}
Given a sample of size $N$, it is often useful to select a subsample of smaller size $n<N$ to 
be used for statistical estimation or learning. 
Such a data selection step is useful to reduce the requirements of data labeling and the computational complexity
of learning. 
We assume to be given $N$ unlabeled samples $\{\bx_i\}_{i\le N}$, 
and to be given access to a  `surrogate model' 
that can predict labels $y_i$ better than random guessing.
Our goal is to select a subset of the samples, to be denoted by $\{\bx_i\}_{i\in G}$, 
of size $|G|=n<N$.
We then acquire labels  for this set and we use them to train a model
via regularized empirical risk minimization.

By using a mixture of numerical experiments
on real and synthetic data, and mathematical derivations under low- and high- dimensional asymptotics,
we show that: $(i)$~Data selection can be very effective, in particular
beating training on the full sample in some cases; $(ii)$~Certain popular choices in data selection methods (e.g. unbiased reweighted subsampling,
or influence function-based subsampling) can be substantially suboptimal.
\end{abstract}

\tableofcontents

\section{Introduction}

Labeling is a notoriously laborious task in machine learning. 
A possible approach towards reducing this burden is to identify a
small subset of training samples that are most valuable for training. 
To be concrete, consider the standard supervised learning setting in which the whole
data consist of $N$ feature vector-label pairs 
$(\bx_i,y_i)\in\reals^p\times\reals$, $i\le N$, but now assume that the labels $y_i$ are not observed. We would like to label only the $n$ most important vectors
in this sample.

More explicitly, data selection-based learning consists of 
two steps:
\begin{enumerate}
\item \emph{Data selection.} Given feature vectors
$\bX:=(\bx_i)_{i\le N}$, select a subset $G\subseteq [N]$ of size $n$ (or close to $n$).
\item \emph{Training.} Having acquired labels for the selected subset 
$\{y_i\}_{i\in G}$, train a model $f(\, \cdot\, ;\btheta):\reals^p\to\reals$
(with parameters $\btheta$) on the labeled data $\{(\bx_i,y_i)\}_{i\in G}$.
\end{enumerate}
Throughout this paper we will focus on methods with the following structure.
In the first step, set $G$ is generated by selecting each datapoint $i$ independently with 
probability $\pr_i = \pr(\bx_i)$. Namely 
\begin{align}
\prob(i\in G|\by,\bX)=\pi_i, \;\;\; \mbox{independently for }i\le N.
\end{align}

The second step (training) is carried out by (weighted) empirical risk minimization (ERM) over the selected set 
$G$. Namely, given loss function $\ell:\reals\times\reals\to\reals$
and regularizer $\Omega:\reals^p\to\reals$, we solve
\begin{align}
\hbtheta & = \arg\min_{\btheta}\hR_N(\btheta)\, ,\label{eq:Mestimators}\\
\hR_N(\btheta)&:= \frac{1}{N}\sum_{i\in G}w_i\, \ell(y_i,f(\bx_i;\btheta))
+\lambda\, \Omega(\btheta)\, . \label{eq:ER}
\end{align}
The weights $w_i$ can depend on the feature vectors,
and are designed as to reduce the estimation or prediction error.
In particular, a popular choice is $w_i={\rm const}/\pr_i$ because this implies that
$\hR_N(\btheta)$ is an unbiased version of the full empirical risk (corresponding to $\pi_i=1$ for all $i$). 

Throughout this paper, we will assume data $(\bx_i,y_i)$ 
to be i.i.d. samples from a common distribution $\prob\in\cuP(\reals^d\times \reals)$ and will evaluate a data selection procedure via the resulting
test error.
\begin{align}
R_{\stest}(\btheta) = \E\, \ell_{\stest}\big(y_{\snew};f(\bx_{\snew};\btheta)\big)\, ,\label{eq:DefTestError}
\end{align}
where expectation is taken with respect to the test
sample $(\bx_{\snew},y_{\snew})\sim \prob$. We allow the test loss
$\ell_{\stest}$ to be different from the training loss $\ell$.
We will denote the (training) population risk
by $R(\btheta) = \E\, \ell\big(y;f(\bx;\btheta)\big)$.

Given its practical utility, data selection has been extensively
studied in a variety of settings, including experimental design, active learning, 
learning algorithms and data optimization. While some heuristics 
(e.g., focus on samples that are most difficult to predict) have resurfaced  from different viewpoints, their implementation
and effectiveness depends on the problem formulation. In particular, unlike the most common active learning scenario, 
it is important that we assume to be given a fixed data sample $\{\bx_i\}_{i\le N}$
(without labels) and carry out a single data selection step. (This is sometime
referred to as `pool-based active learning' \cite{settles2012active}.)
 
Before summarizing our contributions, 
it is useful to mention some of the existing
approaches. 

\paragraph{Bayesian methods.}   Within a Bayesian setting, one can quantify uncertainty about $\btheta$ by using its conditional entropy given the data, and is therefore natural to select the subsample $G$ as to minimize this conditional entropy. This is a first example of ``uncertainty sampling."  It was introduced in the seminal work of Lindley \cite{lindley1956measure}
and developed in a high-dimensional setting by Seung, Opper, Sompolinsky
\cite{seung1992query}. We refer to \cite{houlsby2011bayesian,gal2017deep} for recent 
pointers to this line of work, emphasizing that its focus is largely on online active learning.
(See \cite{cui2021large} for a recent exception.)

\paragraph{Heuristic approaches.} Several groups developed subsampling schemes based on
some measure of the impact of each samples on the trained model. Among others:
\cite{lewis94sequential} measures uncertainty using probabilities predicted by
a single current model;
\cite{vodrahalli2018all} selects the samples with largest norm of the
gradient $\nabla_{\btheta}\ell(y_i,f(\bx_i,\btheta))$;  \cite{jiang2019accelerating}
selects sample $i$ with probability increasing in the loss itself $\ell(y_i,f(\bx_i,\btheta))$
(more precisely, the probability depends on the order statistics of $\ell(y_i,f(\bx_i,\btheta))$).

\paragraph{Leverage scores} have been used for a long time in statistics  as a measure of importance of a data point $i\in \{1,\dots,N\}$ in linear regression \cite{chatterjee1986influential}. 
Recall that the leverage of sample $\bx_i$ is defined as $H_{ii}:=\bx_i^{\sT}(\bX^{\sT}\bX)^{-1}\bx_i$. Subsampling methods based on leverage scores suggest to
sample rows with probabilities $\pr_i \propto H_{ii}$  and weight them with $w_i=1/\pr_i$
to obtain an unbiased risk estimate
\cite{drineas2006sampling,drineas2011faster}. 

These approaches were initially developed from a  numerical linear algebra perspective.
As a consequence, their objective was to approximate the full sample linear regression solution, rather
than to achieve small estimation or generalization error.
Statistical analyses of leverage score-based data selection was developed in \cite{ma2014statistical,raskutti2016statistical,ma2022asymptotic}. These works all assume variations of the original scheme of   \cite{drineas2006sampling} possibly replacing 
$\pr_i \propto H_{ii}$  by $\pr_i \propto f(H_{ii})$ for some monotone increasing function $f$,
and  $w_i=1/\pr_i$ for unbiasedness.

\paragraph{Influence functions.} Let $\btheta_*= \argmin_{\btheta\in\reals^p} R(\btheta)$ denote
the population parameter values. 
For a number of estimators (in particular, for M-estimators of the form \eqref{eq:Mestimators}, under certain regularity conditions)
the following approximate linearity holds 
as $n\to\infty$ for $p$ fixed \cite{van2000asymptotic}
\begin{align}
\hbtheta-\btheta_* = \frac{1}{n}\sum_{i\in G} w_i\psi(\bx_i,y_i)+o_P(1/\sqrt{n})\, ,
\end{align}
where $\psi:\reals^p\times\reals\to\reals$ is the so-called influence 
function. This approximation can be used to select the probabilities $\pi_i$ and weights $w_i$.
Several authors have developed this approach in the context of generalized linear models,
resulting in the choice $\pi_i\propto \|\psi(\bx_i,y_i)\|_2$
\cite{ting2018optimal,wang2018optimal,ai2021optimal}. However this conclusion is derived assuming unbiased subsampling,
i.e. for $w_i=1/\pi_i$. (The recent paper \cite{wang2020less} shows potential for improvement by 
using biased schemes.) 

Further, most of these works \cite{wang2018optimal,wang2020less,ai2021optimal} assume a specific asymptotics 
in which size of the subsample $n$ diverges \emph{after} the full sample size $N$ diverges.
In other words, they focus on the regime $1\ll n\ll N$, in which the estimation (or generalization) error achieved
after data selection is much larger than the one on the full sample. In many applications,
we are most interested in the case in which the two errors (before and after selection) are comparable. In other words, we are practically interested in selecting a constant fraction 
of the data: $n=\Theta(N)$.

\paragraph{Margin-based selection.} The recent paper \cite{sorscher2022beyond} shows that, in the case of the binary perceptron, under a noiseless teacher-student  distribution,
selecting samples far from the margin
is beneficial in a high-dimensional or data-poor regime. 
This is in stark contrast with influence-function and related approaches that upsample data points whose labels are more difficult to predict.
While our results confirm these findings,  measuring
uncertainty in terms of distance from the margin is fairly specific to binary classification under certain distributional assumptions \footnote{We note that \cite{sorscher2022beyond} 
proposes schemes that are
empirically effective more generally, but the connection with their 
mathematical analysis is indirect.}.
%
%
\section{Definitions}

As mentioned in the introduction, an important application of data selection is to alleviate the
need to label data. We will therefore focus on data selection schemes that do not make use of the labels
$(y_i)_{i\le N}$.
On the other hand, we will assume to have access to a surrogate model
$\sP_{\surr}(\de y|\bx)$ which is a probability kernel
from $\reals^d$ to $\reals$. 
We will consider two settings:
\begin{itemize}
    \item \emph{Ideal surrogate.} In this case $\sP_{\surr}(\de y|\bx) = \prob(\de y|\bx)$ is the actual conditional distribution of the labels. Of course this is the very object we want to learn. Since however the overall
    data selection and learning scheme is constrained, the problem is nevertheless non-trivial.
    \item \emph{Imperfect surrogates.} In this case $\sP_{\surr}(\de y|\bx)$ is predictive of the actual value of the labels, but does not coincide with the Bayes predictor. We will study the dependence of optimal data-selection on the accuracy of the surrogate model.
\end{itemize}

The selection probabilities and weights can depend on $\bx_i$ and  $\sP_{\surr}(\, \cdot\, |\bx_i)$
(the conditional distribution of $y_i$ given $\bx_i$ under the surrogate). Formally
\begin{align}
\pi_i=\pi(\bx_i,\sP_{\surr}(\, \cdot\, |\bx_i))\, ,\;\;\;\; w_i=w(\bx_i,\sP_{\surr}(\, \cdot\, |\bx_i))\, ,
\end{align}
for some functions $\pi(\,\cdots\,)$,
$w(\,\cdots\,)$. In order to simplify notations, we will omit the dependence  on the
surrogate predictions, unless needed.

It is convenient to encode
the selection process into random variables $S_i(\bx_i,\sP_{\surr}(\, \cdot\, |\bx_i))\ge 0$ which depend on
$\bx_i,\sP_{\surr}(\, \cdot\, |\bx_i)$, and also on additional  randomness 
independent across different 
samples\footnote{Formally, $S_i(\bx_i,\sP_{\surr}(\, \cdot\, |\bx_i)) = s(U_i,\bx_i,\sP_{\surr}(\, \cdot\, |\bx_i))$, where $s$ 
is a deterministic function and $(U_i)_{i\le N}\sim_{i.i.d.}\Unif([0,1])$. }.
Again, we will omit the dependence on
$\sP_{\surr}(\, \cdot\, |\bx_i)$.
We recover the original formulation by setting
\begin{align}
S_i(\bx_i) = w(\bx_i) \,  \bfone_{i\in G}\, ,\;\;\; \prob(i\in G|\bX,\by) = \pi(\bx_i)\, .
\end{align}
We can thus rewrite the empirical risk \eqref{eq:ER} as:
\begin{align}
\hR_N(\btheta)&= \frac{1}{N}\sum_{i=1}^N\,S_i(\bx_i) \;  L(\btheta;y_i,\bx_i) 
+\lambda\, \Omega(\btheta)\, ,\label{eq:ER2}\\
L(\btheta;y,\bx) & := \ell(y,f(\bx;\btheta))\, .
\end{align}
(Analogously, we introduce the notation  $L_{\stest}(\btheta;y,\bx)  := \ell_{\stest}(y,f(\bx;\btheta))$
for the test loss.)
We also recall the notation for full sample (not reweighted) population risk and 
its minimizer
\begin{align}
R(\btheta)&:= \E\big\{ L(\btheta;y,\bx)\big\} \,,\;\;\;\;\;\;
\btheta_*:= \argmin_{\btheta} R(\btheta)\, .
\end{align}

We will enforce that the target sample size $n$ is achieved in expectation, namely
\begin{align}
n= \sum_{i=1}^N\E\, \pi(\bx_i)\, ,
\;\;\;\;\;\; \pi(\bx_i) :=\prob\big(S_i(\bx_i)>0\big|\bX,\by\big)\, .\label{eq:Normalization}
\end{align}
%
A special role is played of course by unbiased schemes.
\begin{definition}\label{def:Unbias}
We denote the set of data selection schemes defined above
(namely, randomized functions $(\bx,\sP_{\surr}(\, \cdot\, |\bx)) \mapsto S(\bx,\sP_{\surr}(\, \cdot\, |\bx))$) by $\cuA$.

A data selection scheme is \emph{unbiased} if $\E\{S(\bx)|\bX,\by\}=1$. We denote the set of unbiased 
data selection schemes by $\cuU$. 
\end{definition}

Throughout, we consider $n,N\to\infty$, with converging ratio 
\begin{align}
\frac{n}{N}\to \gamma\in (0,1)\, \label{eq:GammaDef}.
\end{align}
We refer to $\gamma$ as to the `subsampling fraction.' For the sake of simplicity, we will think of
sequences of instances indexed by $N\to\infty$ and it will be understood that $n(N) \to\infty$ as well.
As mentioned above, we are interested in 
$\gamma$ bounded away from $0$ because we want to
keep the accuracy close to the full sample accuracy.

\subsection*{Notations}

The unit sphere in $d$ dimensions is denoted by 
$\S^{d-1}$. Given two symmetric  matrices $\bA,\bB\in\reals^{d\times d}$,
we write $\bA\succeq \bB$ if $\bA-\bB$ is positive semidefinite, and
 $\bA\succ \bB$ if $\bA-\bB$ is strictly positive definite.
We denote by $\plim_{n\to\infty}X_n$ the limit in probability of a sequence of random variables 
$(X_n)_{n\ge 1}$. We use $O_P(\,\cdot\,)$, $o_P(\,\cdot\,)$  and so on for the 
standard Oh-notation in probability. For instance, given a sequence of random variables $X_N$,
and deterministic quantities $b_N$, we have
\begin{align}
    X_N = o_P(b_N)\;\; \Leftrightarrow \;\; \plim_{N\to\infty}\frac{|X_N|}{b_N} = 0\, .
\end{align}
%
%
\section{Summary of results}
\label{sec:Summary}

\begin{figure}
\begin{center}
\includegraphics[width=0.7\linewidth]{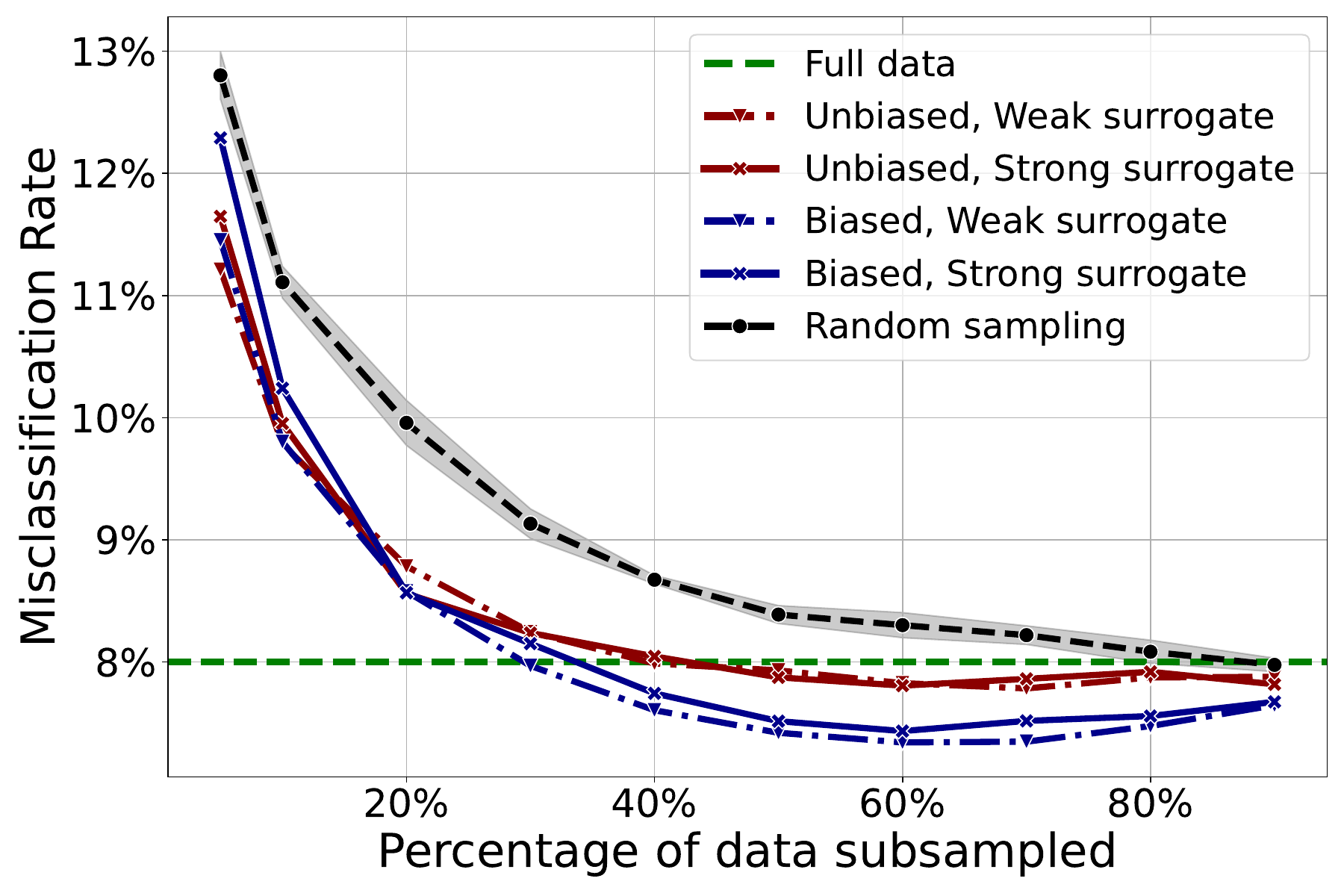}
\end{center}
\caption{Test (misclassification) error in a binary image classification problem. We train logistic regression after data selection on feature vectors obtained from SwAV embeddings 
($N=34345$ samples in $p=2048$ dimensions). The data selection schemes use  surrogate models that were trained on a small separated fraction of the 
data ($N_{\surr}= 14720$ for the `strong surrogate' and $N_{\surr}= 1472$ for the `weak surrogate'). See Section \ref{sec:NumericalReal} for further explanations.}
\label{fig:Summary}
\end{figure}

Our theory is based on two types of asymptotics, that capture complementary
regimes. In Sections \ref{sec:LowDim}, \ref{sec:LowDimImperf} we will study the low-dimensional asymptotics, whereby $p$ is fixed as $n,N\to\infty$.
This analysis applies to a fairly general class of data distributions and 
parametric models $f(\,\cdot\,;\btheta)$.
In Section \ref{sec:HighDim}, we will instead assume 
$p\to\infty$ as $n,N\to\infty$, with $n/p\to \delta\in (0,\infty)$. In this case
we restrict ourselves to the case of convex M-estimators, 
for models that are linear in the feature vectors $\bx_i$.
Namely, we assume $L(\btheta;y_i,\bx_i) = 
L_*(\<\btheta,\bx_i\>;y_i)$, with $L_*$ convex in its first argument.
Further, we assume the $\bx_i$ to be standard Gaussian. 
 In this setting, we can use techniques based on Gaussian comparison inequalities.

We next summarize some of the insights emerging from our work. Some of these points are 
illustrated in Figures \ref{fig:Summary} and \ref{fig:Summary_Opt}, which report the results of data selection experiments with real data, described in Section \ref{sec:NumericalReal}. 
\begin{description}
\item[Unbiased subsampling can be suboptimal.] As mentioned in the introduction, the 
vast majority of (non-Bayesian) theoretical studies assumes unbiased subsampling schemes,
whereby $w_i\propto 1/\pi_i$. We show both theoretically and empirically that this can lead to 
potentially unbounded loss in accuracy with respect to biased schemes. Low-dimensional
asymptotics provide a particularly compelling explanation of this phenomenon. 

In the low-dimensional regime, the estimation error is ---roughly speaking--- inversely proportional to the curvature of the expected risk. Unbiased methods do not change the  curvature, while biased methods can increase the curvature.
In fact, we will prove that unbiased methods are  
suboptimal under a broad set of conditions conditions.

This result is illustrated in Figures \ref{fig:Summary} and \ref{fig:Summary_Opt}. In Figure \ref{fig:Summary}  we compare a biased and an unbiased scheme
(with selection probabilities which approximate the influence-function scheme). In Figure \ref{fig:Summary_Opt}
we select, for each value of the undersampling ratio 
$n/N$, the best among the biased and unbiased scheme. We observe that the biased scheme often achieves better test error than the unbiased one.
\item[Optimal selection depends on the subsampling fraction.] 
Popular data selection methods are independent of subsampling fraction in the sense that the selection probabilities $\pi_i$ can be computed without knowing the target sampling fraction $n/N$, except from the overall normalization that enforces the constraint \eqref{eq:Normalization}. In particular, both leverage score and influence function methods compute scores that depend on the data, but not on the target sample size.

In contrast, we will see that optimal schemes depend in a
non-trivial way on $n/N$.
This is intuitively very natural. Consider, for instance the leverage score
$\bx_i^{\sT}(\bX^{\sT}\bX)^{-1}\bx_i$.
This measures how much datapoint $\bx_i$ differs from the the other data. However, 
it makes more sense to measure the difference from other data in the selected set $G$. This suggests the use of
$\bx_i^{\sT}(\bX_G^{\sT}\bX_G)^{-1}\bx_i$ to select a new data point to add to $G$. This 
intuition will be reflected in the optimal methods derived below.
\begin{figure}
\centering
\phantom{A}\hspace{-1.cm}
\includegraphics[width=0.53\linewidth]{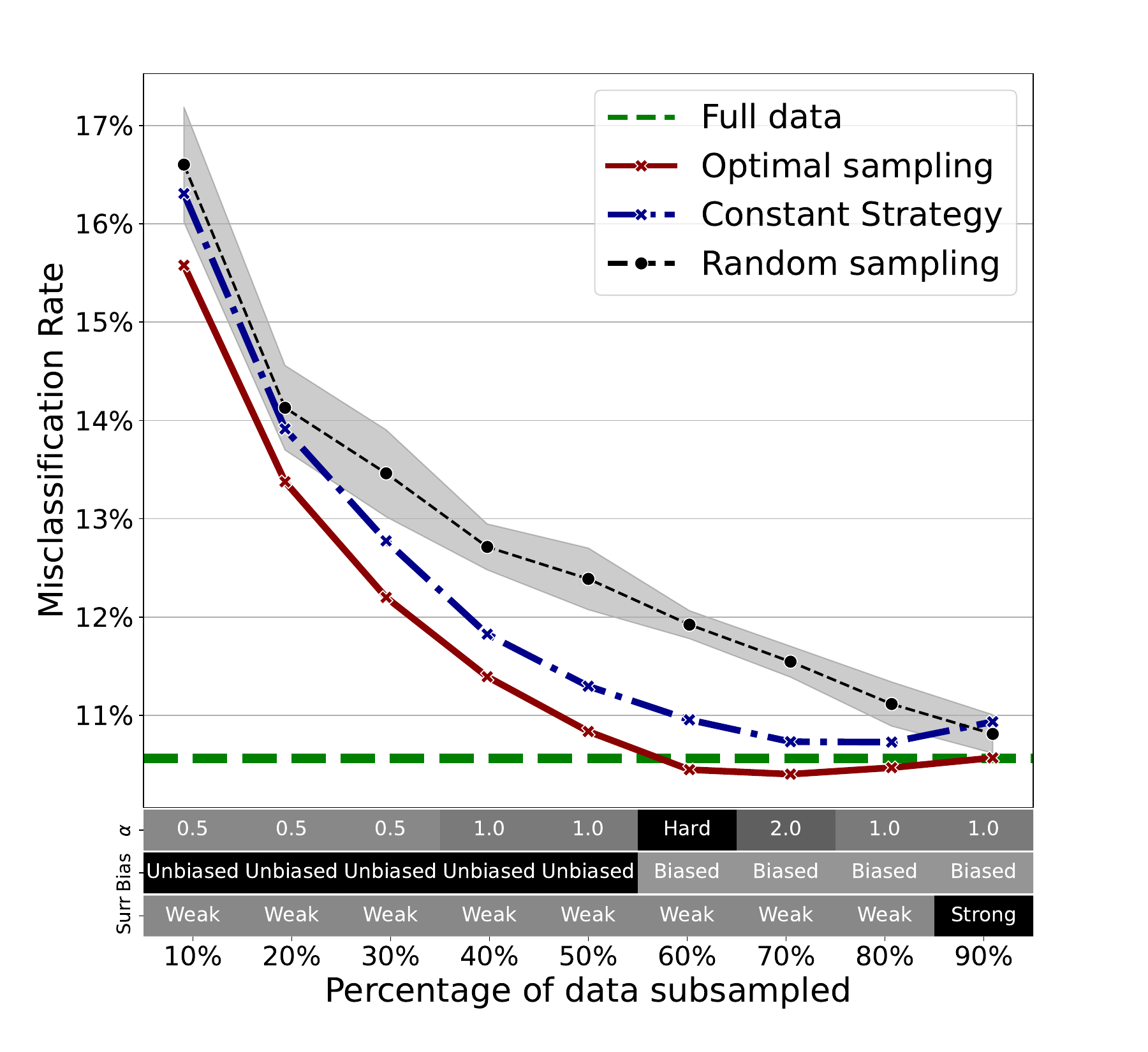}
\hspace{-1cm}
 \includegraphics[width=0.53\linewidth]{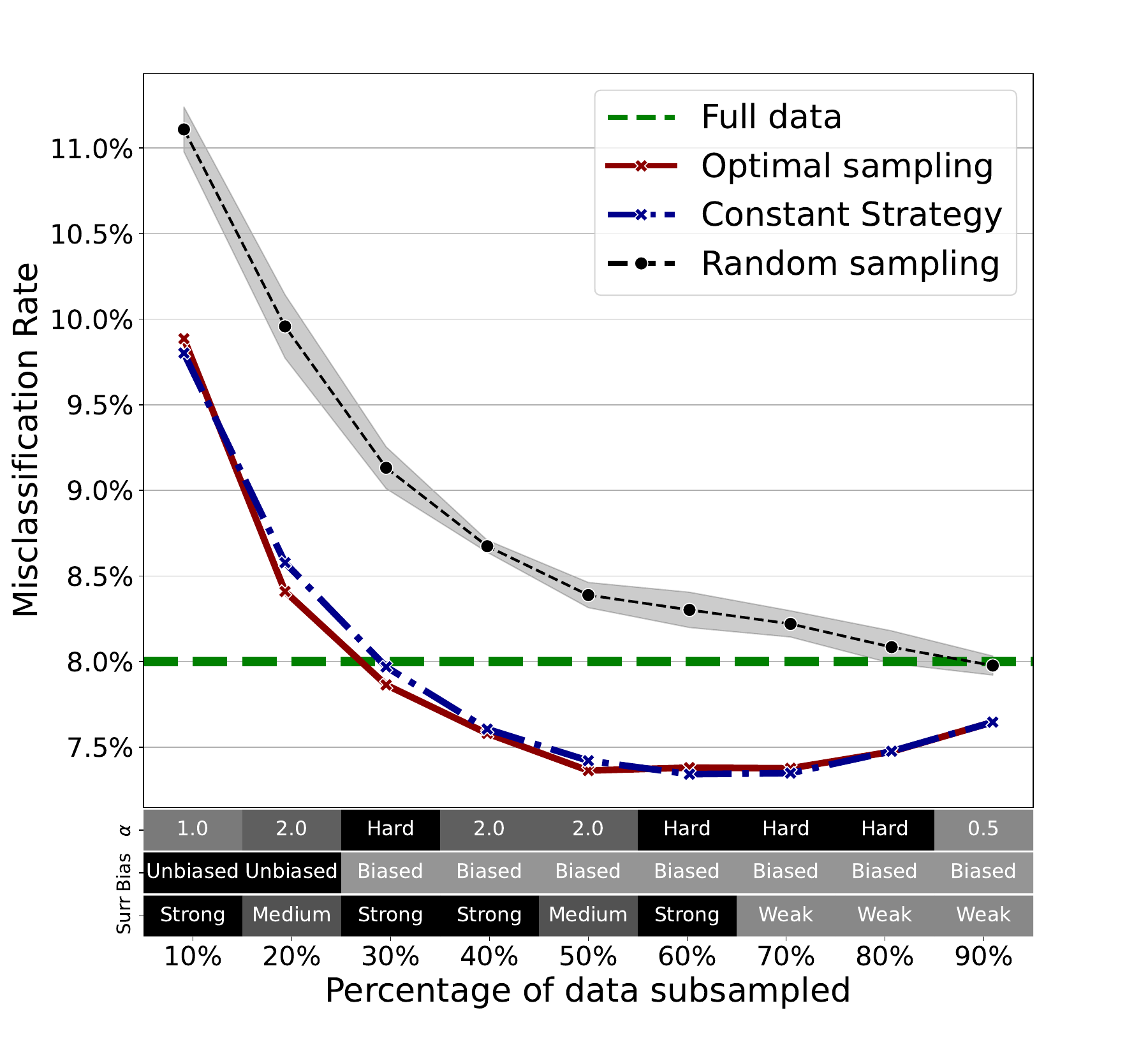}
\caption{Test error in the same binary classification problem as in Figure
\ref{fig:Summary} (see Section \ref{sec:NumericalReal}). At each value of the 
subsampling rate, we optimize our scheme over the following parameters: 
ridge regularization $\lambda$; exponent $\alpha$ in the data selection scheme
(larger $\alpha$ corresponds to selecting `harder' samples); biased or unbiased subsampling; and
strength of the surrogate model. Constant strategy refers to biased subsampling with parameters $\lambda=0.01$, $\alpha=0.5$, and weak surrogate models. Left plot: $N=3434$, $p=2048$; right plot $N=34345$, $p=2048$.}
\label{fig:Summary_Opt}
\end{figure}

\item[Plugin use of the surrogate can be suboptimal.]
How to construct selection schemes in absence of the labels $y_i$? A natural 
idea would be to start from selection probabilities that use the labels, 
$\pi_0(y_i,\bx_i)$, and then take a conditional expectation with respect to $y_i$,
to get $\pi(\bx_i)= \E\{\pi_0(y_i,\bx_i)|\bx_i\}$. 
We can then replace this conditional
expectation by conditional expectation with respect to the surrogate:
$\pi_{\surr}(\bx_i)= {\sf E}_{\surr}\{\pi(y_i,\bx_i)|\bx_i\}$.  

We will show that this is not always the best option, and in particular
better surrogate models do not always lead to better data selection.
This is (partially) illustrated  by Figure \ref{fig:Summary}, which reports test error achieved by selecting data on the basis of surrogate models of different accuracy. We observe that the test error is 
insensitive to the difference of surrogate models.
\item[Uncertainty-based subsampling is effective.]
The simplest rule of thumb emerging from our work broadly confirms earlier research: subsampling schemes based on the sample `hardness,'
i.e. on how uncertain surrogate predictions are, 
perform well in a broad array of settings.
One prominent example is provided by influence function-based 
selection, which is biased towards most uncertain
samples. 

While influence-function based subsampling is often a good
starting point, it is not generally optimal. 
First: in the low-dimensional asymptotics, stronger bias
towards hard examples is beneficial, cf. Section \ref{sec:BiasedFirst}. 
Second: in certain high-dimensional cases, this bias
must be reversed, and ``easier'' examples should be selected,
see Section \ref{sec:HighDim}.
This is in agreement with the results of \cite{sorscher2022beyond}.
Third, balancing the previous point, selecting the `hardest' or `easiest'  even in high dimension, is not always optimal.
For instance, in some cases biasing towards hard samples
can be beneficial in high-dimension, but selecting the hardest 
ones, e.g. those that have largest value of influence,
can lead to significant failures, see e.g. Section \ref{sec:NumericalSynth}.
\item[Data subsampling can improve generalization.] As stated several times, our objective is to select a fraction $\gamma\in (0,1)$ of the data such that, training a model on the selected subset yields test error close to the one achieved by training on the whole subset.
To our surprise, we discover that good data selection, not only can keep test error 
close to the original one down to $\gamma$ as small as $0.4$, but can actually reduce the 
test error below the one obtained from the full sample.
This is illustrated by Figure \ref{fig:Summary}.

The reader might wonder whether this happens because 
information is being passed from the
surrogate model to the trained model via data selection. However, things are more interesting than this.
Indeed, we achieve a reduction in test error even 
when the surrogate model is trained on significantly than $(1-\gamma)N$ random samples from the same population (so that overall, we are not using the entire dataset).

In Section \ref{sec:NonMono}, we will construct explicit examples 
for which we prove that learning after data selection achieves smaller
test error than learning on the full sample.
\end{description}
%
%
\section{Low-dimensional asymptotics: Ideal surrogate}
\label{sec:LowDim}

In this section we study the classical asymptotics in which $p$ is fixed as $n,N\to\infty$.
This setting is simpler and allows for weaker 
assumptions on the data distribution.

\subsection{General asymptotics}

The asymptotics of $\hbtheta$  depends on the population risk associated to sampling scheme $S$
(that is the expectation of the empirical risk \eqref{eq:ER2}):
\begin{align}
R_S(\btheta) : =\E\big\{S(\bx)\, L(\btheta;y,\bx)\big\}\, .
\end{align}
(In this notation, the argument $S$ indicates the dependence on the 
\emph{function} that defines the subsampling procedure.)
The conditional gradient covariance, and conditional Hessian will play an important role,
and are defined below:
\begin{align}
\bG(\bx):= \E\big\{\nabla_{\btheta}L(\btheta_*;y,\bx)
 \nabla_{\btheta}L(\btheta_*;y,\bx)^{\sT}|\bx\big\}\, ,\;\;\;\;\;\;
 \bH(\bx):= \E\big\{\nabla^2_{\btheta}L(\btheta_*;y,\bx)|\bx\big\}\, .\label{eq:GHDef}
\end{align}
Our first result characterizes the error under weighted quadratic losses
$\|\hbtheta-\btheta_*\|_{\bQ}^2:=\<\hbtheta-\btheta_*,\bQ(\hbtheta-\btheta_*)\>$,
for $\bQ\in\reals^{p\times p}$. This covers both the standard $\ell_2$ estimation error 
(by setting $\bQ=\id$) and, under smoothness conditions, the test error $R_{\stest}(\hbtheta)$.

We define the asymptotic error coefficient via
%
\begin{align}
\rcoeff(S;\bQ) :=\lim_{M\to\infty}\lim_{N\to\infty}\E\big\{
\big(N\|\hbtheta-\btheta_*\|_{\bQ}^2\big)\wedge M\big\}\, ,\label{eq:RhoCoeffDef}
\end{align}
whenever the limit exists. The next result is an application of textbook asymptotic statistics,
as detailed in Appendix \ref{sec:ProofLowDim}.
\begin{proposition}\label{propo:LowDimAsymptotics}
Assume the following:
\begin{itemize}
\item[{\sf A1.}]  $R_S(\btheta)$ is uniquely minimized at $\btheta=\btheta_*$.
\item[{\sf A2.}] $\btheta\mapsto L(\btheta;y,\bx)$ is non-negative, lower semicontinuous. Further, for every 
$\bu\in \S^{p-1}$, 
define $L_{\infty}(\bu;y,\bx)\in [0,\infty]$
\begin{align}
L_{\infty}(\bu;y,\bx) = 
\liminf_{\substack{\btheta\to+\infty\\
\btheta/\|\btheta\|\to\bu}} L(\btheta;y,\bx)\, ,
\end{align}
and assume $\inf_{\bu\in\S^{p-1}}\E \{S(\bx)L_{\infty}(\bu;y,\bx)\} > R_S(\btheta_*)$.
\item[{\sf A3.}] $\btheta\mapsto L(\btheta;y,\bx)$ is differentiable at $\btheta_*$
for $\prob$-almost every  $(y,\bx)$.
Further, for  $B$  a neighborhood of $\btheta_*$,
\begin{align*}
\E\sup_{\btheta_1\neq\btheta_2\in  B}\lt\{S(\bx)^2\frac{|L(\btheta_1;y,\bx)-L(\btheta_2;y,\bx)|^2}
{\|\btheta_1-\btheta_2\|^2_2}\rt\}<\infty
\end{align*}
\item[{\sf A4.}] $\btheta\mapsto R_S(\btheta)$ is twice differentiable at $\btheta_*$,
with $\nabla^2R_S(\btheta_*) = \E\{S(\bx)\bH(\bx)\}\succ \bzero$. 
\end{itemize}
Then, for any $\bQ\in\reals^{p\times p}$ symmetric, the limit of Eq.~\eqref{eq:RhoCoeffDef}
exists and is given by
\begin{align}
\rho(S;\bQ) & =  \frac{\E\{S(\bx)^2\}}{\E\{S(\bx)\}^2}\, \Tr\big(\bG_S\bH_S^{-1}
\bQ\bH_S^{-1}\big)\, ,\label{eq:GeneralAsymp}
\end{align}
where 
\begin{align}
\bG_S:= \frac{\E\big\{S(\bx)^2\, \bG(\bx)\big\}}{\E \{S(\bx)^2\}}\, ,
\;\;\;\;\;\;\;\;
 \bH_S:= \frac{\E\big\{S(\bx)\, \bH(\bx)\big\}}{\E \{S(\bx)\}}\, .
\end{align}
Further, $\lim_{M\to\infty}\lim_{N\to\infty}\prob\big\{N\|\hbtheta-\btheta_*\|_{\bQ}^2>M\big\}=0$. 

In particular, if $\btheta\mapsto R_{\stest}(\btheta)$ is twice continuously differentiable at $\btheta_*$, with $\nabla R_{\stest}(\btheta_*) = \bzero$, then
\begin{align}
R_{\stest}(\hbtheta) &= R_{\stest}(\btheta_*) + \frac{1}{N}\cdot \rho(S;\bQ_{\stest})\cdot Y_S +o_P(1/N)\, ,
\end{align}
where $\bQ_{\stest} := \frac{1}{2}\nabla^2R_{\stest}(\btheta_*)$, and $\E Y_{S}=1$. (More explicitly,
$Y_S=\|\bg_S\|^2_{\bQ_{\stest}}/\E\{\|\bg_S\|^2_{\bQ_{\stest}}\}$, where $\bg$ is a centered Gaussian vector with covariance 
$\bH_S^{-1}\bG_S\bH_S^{-1}$.)
\end{proposition}
\begin{remark}
Key  for Proposition \ref{propo:LowDimAsymptotics} to hold is condition
{\sf A1}, which requires the minimizer of the subsampled population risk
$R_S$ to coincide with the minimizer of the original risk $R$. This amounts to say that the data selection
scheme is not so biased as to make empirical risk minimization inconsistent. If it does not hold,
then the resulting error $\|\hbtheta-\btheta_*\|_{\bQ}^2$ will be, in general, of order one.
\end{remark}

\subsection{Unbiased data selection}
\label{sec:UnbiasedFirst}

In this section and the next one, we discuss optimal choices of the subsampling 
procedure under the low-dimensional asymptotics.

We begin by considering unbiased schemes, i.e. schemes satisfying Definition \ref{def:Unbias}.
 While this case is covered by earlier work \cite{ting2018optimal,wang2018optimal,ai2021optimal}, 
 it is somewhat simpler than the biased case and provides useful background for the development
in the next sections.

The key simplification is that, in the unbiased case, the matrix $\bH_S$ does not depend on $S$,
namely $\bH_S= \bH$, where
\begin{align}
\bH& = \nabla^2 R(\btheta_*)\, ,\;\;\;\;\;
R(\btheta) := \E L(\btheta;y,\bx)\, .
\end{align}
(In general $R$ can be different from the test error $R_{\stest}(\btheta)$ because of the different loss.)
We thus recover the standard subsampling scheme based on influence functions. (For the reader's convenience, we present a proof of this specific formulation in Appendix \ref{sec:ProofBasicUnbiased}.)
\begin{proposition}\label{propo:Unbiased}
Under the assumptions of Proposition \ref{propo:LowDimAsymptotics}, further assume
$\bQ\succeq\bzero$, $\bH\succ\bzero$. Then $\rcoeff(S;\bQ)$ is minimized among unbiased data selection schemes 
by a scheme of the form
\begin{align}
S_{\sunb}(\bx) =
\begin{cases}
\pi_{\sunb}(\bx)^{-1} &\mbox{with probability  }\pi_{\sunb}(\bx),\\
0 &\mbox{otherwise}\, ,
\end{cases}
\end{align}
Further, an optimal choice of $\pi(\bx)$ is given by 
\begin{align}
\pi_{\sunb}(\bx) &= 
\min\Big(1; c(\gamma)\, Z(\bx)^{1/2}\Big)\, ,\label{eq:UnbFirst}\\
Z(\bx)&:=\Tr\big(\bG(\bx) \bH^{-1}\bQ\bH^{-1}\big)=\E\big\{\big\|\nabla_{\btheta} L(\btheta;y,\bx)\big\|^2_{\bH^{-1}\bQ\bH^{-1}}\big|\bx\big\} \, .\nonumber
\end{align}
Here the constant $c(\gamma)$ is defined so that $\E \pi_{\sunb}(\bx) =\gamma$.

Finally the the optimal coefficient $\rcoeff_{\sunb}(\bQ):=\inf_{S\in\cuU}\rcoeff(S;\bQ)
= \rcoeff(S_{\sunb};\bQ)$ is given by
\begin{align}
\rcoeff_{\sunb}(\bQ) &  =\inf_{\E\pi(\bx)= \gamma}\E\Big\{\frac{Z(\bx)}{\pi(\bx)}\Big\}\label{eq:OptimalUnbiased}\\
&=\E
\max\Big(\frac{1}{c(\gamma)}Z(\bx)^{1/2};Z(\bx)\Big) \, .\nonumber
\end{align}
\end{proposition}

Denote by $S_{\srand}$ random subsampling,
namely $S_{\srand}(\bx) = 1/\gamma$ with probability $\gamma$, and $S_{\srand}(\bx) =0$ otherwise
(which is of course is unbiased). The next proposition establishes some basic properties of 
$\rcoeff_{\sunb}(\bQ;\gamma)$.
\begin{proposition}\label{propo:UnbiasedMono}
 Write $\rcoeff_{\sunb}(\bQ;\gamma)$ for the optimal unbiased coefficient
when the subsampling fraction equals $\gamma$. Then:
$(1)$~$\rcoeff_{\sunb}(\bQ)\le \rcoeff_{\srand}(\bQ)$ with the inequality being strict unless 
$Z(\bx)$ is almost surely constant or $\gamma=1$; $(2)$~$\gamma\mapsto\rcoeff_{\sunb}(\bQ;\gamma)$
is monotone non-increasing (for $\bQ\succeq \bfzero$).
\end{proposition}
\begin{proof}
For claim $(1)$, it is easy to compute $\rcoeff_{\srand}(\bQ) = \gamma^{-1}\E(Z)$.
Therefore, the gain derived from optimal subsampling can be written as
\begin{align}
\frac{\rcoeff_{\sunb}(\bQ)}{\rcoeff_{\srand}(\bQ)}  = \frac{\E(Z/\pi_{\sunb})\E(\pi_{\sunb})}{\E((Z/\pi_{\sunb})\cdot \pi_{\sunb})}\le 1\, . 
\end{align}
The inequality above is due to the fact that $\pi_{\sunb}$ and $Z/\pi_{\sunb}$ are both non-decreasing 
functions of $Z$.  It is strict unless the following happens with probability one, for i.i.d. 
copies $Z_1,Z_2$ of $Z$:
$\min(1;cZ_1^{1/2})= \min(1;cZ_2^{1/2})$ or $\max(Z_1;Z_1^{1/2}/c)= \max(Z_2;Z_2^{1/2}/c)$
However, this happens only if either $Z$ is almost surely a constant or if $cZ\le 1$ almost surely.
The latter implies $\gamma=1$.

To prove claim $(2)$, 
note that the mapping $\pi\mapsto \E\{Z(\bx)/\pi(\bx)\}$ appearing in Eq.~\eqref{eq:OptimalUnbiased}
is monotone with respect to the partial ordering of pointwise 
domination. Therefore, we can replace the constraint $\E\pi(\bx)= \gamma$ by $\E\pi(\bx)\le  \gamma$,
whence monotonicity trivially follows.
\end{proof}
While monotonicity may appear intuitively obvious, we will see in the Section \ref{sec:NonMono}
that it does not hold for biased subsampling.
\begin{remark}
Besides computing expectations with respect to the conditional distribution of $y$
given $\bx$ (a problem which will be addressed in the next section), evaluating
the score requires to be able to approximate $\bH= \E\{\bH(\bx)\}$. However, this
can be consistently estimated e.g. using the empirical mean $\hbH:= N^{-1}\sum_{i=1}^N\bH(\bx_i)$.
We will neglect errors involved in such approximations.
\end{remark}

\begin{remark}[Connection with influence functions]
Under the assumptions of Proposition \ref{propo:Unbiased} the influence function is given by
\begin{align}
\psi(\bx,y) = -\bH^{-1}\nabla_{\btheta} L(\btheta_*;y,\bx)\, ,
\end{align}
and therefore we can rewrite (omitting for simplicity the capping at $\pi=1$,
which is only needed at large $\gamma$),
\begin{align}\label{eq:IF}
\pi(\bx_i)\propto \E\big\{\big\|\psi(\bx_i,y_i)\big\|^2_{\bQ}\big|\bx_i\big\}^{1/2}\, .
\end{align}
We recovered influence function-based subsampling 
\cite{ting2018optimal,wang2018optimal,ai2021optimal} (with the choice of the norm
depending on the estimation loss).  Note that the fact that we do not 
use the label $y_i$ results in the appearance of a conditional expectation of $y_i$ given $\bx_i$.
Interestingly, the optimal way to deal with 
unknown $y_i$ is to take
conditional expectation of the \emph{square} of the sampling probability. Namely, we compute \eqref{eq:IF} 
instead of $\E\big\{\big\|\psi(\bx_i,y_i)\big\|_{\bQ}\big|\bx_i\big\}$.
\end{remark}

\begin{example}[Linear regression]\label{example:FirstLinear}
In this case we take $L(\btheta;y,\bx)=L_{\stest}(\btheta;y,\bx) = (y-\<\btheta,\bx\>)^2/2$.
With little loss of generality, we assume $\bx$ to be centered i.e. $\E(\bx) = \bfzero$, with finite 
covariance $\E\{\bx\bx^{\sT}\}=\bSigma$. We then have $\nabla^2R(\btheta)=\bSigma$ and
$\nabla_{\btheta}L(\btheta_*;y_i,\bx_i) =-\eps_i\,  \bx_i$ where $\eps_i = y_i-\<\btheta_*,\bx_i\>$
is the `noise' for observation $i$.

If we further assume $\E(\eps_i^2|\bx_i)=\tau^2$ independent of $\bx_i$, we finally get (dropping a constant multiplicative factor)
\begin{align}
    Z(\bx) = \<\bx,\bSigma^{-1}\bQ\bSigma^{-1}\bx\>\, .\label{eq:GeneralizedLeverage}
\end{align}
In the special case of prediction error, $\bQ=\bSigma$ and therefore 
$Z(\bx) = \<\bx,\bSigma^{-1}\bx\>$
and we recover a population version of the leverage score. 
Note that the optimal sampling probabilities
are proportional to the \emph{square root} of the score (capped at $1$).
\end{example}

\begin{example}[Generalized linear models]\label{ex:FirstGLM}
Again, we assume $\bx_i$ to be a centered vector and
\begin{align}
\prob(\de y_i|\bx_i)  = \exp\big\{y_i\<\btheta_*,\bx_i\>-\phi(\<\btheta_*,\bx_i\>)\big\}\, \nu_0(\de y_i)\, .
\end{align}
Here $\nu_0$ is the reference measure, e.g. $\nu_0=\delta_{+1}+\delta_{-1}$ for logistic regression.
We fit a model of the same type and train via the loss
$L(\btheta;y,\bx)= -y\<\btheta,\bx\>+\phi(\<\btheta,\bx\>)$. Hence 
$\nabla_{\btheta} L(\btheta;y,\bx)=-(y-\phi'(\<\btheta,\bx\>))\bx^{\sT}$ and 
$\nabla^2 R(\btheta_*) = \E\{\phi''(\<\btheta,\bx\>)\bx\bx^{\sT}\}:=\bH$.

A simple calculation yields:
\begin{align}
Z(\bx) = \phi''(\<\btheta_*,\bx\>)\, \cdot 
\<\bx,\bH^{-1}\bQ\bH^{-1}\bx\>\, .\label{eq:GLM_Score_Unbiased}
\end{align}
\end{example}

\subsection{Biased data selection}
\label{sec:BiasedFirst}

We now consider  biased subsampling schemes $S$. Among these, a special role is
played by those schemes in which the only randomness is the choice of
whether or not to select a certain sample $i$.
We refer to these schemes as `simple.'
\begin{definition}[Simple data selection schemes] 
A scheme is \emph{simple} if there exist function $w, \pi$ such that
\begin{align}
S(\bx) =
\begin{cases}
w(\bx) &\mbox{with probability  }\pi(\bx),\\
0 &\mbox{otherwise}\, ,
\end{cases}
\end{align}
We denote the set of simple sampling schemes by $\cuS\subseteq \cuA$
(recall that $\cuA$ is the set of all data selection schemes).
\end{definition}

We can identify any simple data selection scheme with the 
pair $(\pi,w)$, hence we will occasionally abuse notation and write $\rho(\pi,w;\bQ)$ for $\rho(S;\bQ)$.
\begin{lemma}\label{lemma:Simple}
For $\bQ\succeq \bfzero$, we have
\begin{align}
\inf_{S\in \cuA}\rho(S;\bQ) = \inf_{S\in\cuS} \rho(S;\bQ)\, .
\end{align}
\end{lemma}
We can therefore restrict ourselves to simple schemes without loss (See Appendix \ref{sec:ProofSimple} for the proof).
Among these, we focus on `non-reweighting schemes' that are defined by the additional condition $w=1$
and denote them by $\cuN\subseteq \cuS$. Non-reweighting schemes are practically convenient since they do not require to modify 
the training procedure.

It turns out that optimal non-reweighting schemes have a simple structure, as stated below. 
(See Appendix \ref{sec:ProofBiased} for a proof.)
\begin{proposition}\label{propo:Biased}
Under the assumptions of Proposition \ref{propo:LowDimAsymptotics},
further assume $\bQ\succeq\bfzero$, $\bH\succ\bzero$, $\E\{\nabla_{\btheta}L(\btheta_*;y,\bx)|\bx\}=\bzero$,
and $\bx\mapsto \bG(\bx)$, $\bH(\bx)$ to be continuous.
Further define
\begin{align}
\bG_\pi :=\E_{\pi}\bG(\bx)\, ,\;\;\;\;\;\;
\bH_\pi := \E_{\pi}\bH(\bx)\, ,\;\;\;\mbox{where}\;\;\;\; \E_{\pi}f(\bx) := \frac{\E\big\{f(\bx)\,\pi(\bx)\big\}}{\E\big\{\pi(\bx)\big\}}
\end{align}
Finally define the function
\begin{align}
Z(\bx;\pi) := -\Tr\big\{\bG(\bx)\bH_{\pi}^{-1}\bQ\bH_{\pi}^{-1}\big\}+
2 \Tr\big\{\bH(\bx)\bH_{\pi}^{-1}\bQ\bH_{\pi}^{-1}\bG_\pi\bH_{\pi}^{-1}\big\}\, .
\end{align} 
Then there exists $\pi_{\snr}$ achieving the minimum of the asymptotic error over non-reweighting schemes
$\rho(\pi_{\snr},1;\bQ)=\inf_{S\in\cuN}\rho(S;\bQ)$.
Further, if $\bH_{\pi_{\snr}}\succ \bfzero$ strictly, then 
$\pi_{\snr}$ takes the form
\begin{align}
    \pi_{\snr}(\bx) = \begin{cases}
    1 & \mbox{ \textup{if} }Z(\bx;\pi_{\snr}) > \lambda\, ,\\
    0 & \mbox{ \textup{if} }Z(\bx;\pi_{\snr}) < \lambda\, ,\\
    b(\bx)\in [0,1] & \mbox{ \textup{if} }Z(\bx;\pi_{\snr}) = \lambda\, .
    \end{cases}\label{eq:piNR}
\end{align}
where $\lambda$ and $b(\bx)$  are chosen  so that $\E\pi_{\snr}(\bx) = \gamma$.
The resulting  optimal asymptotic  error $\rho_{\snr}(\bQ)  = \rho(\pi_{\snr},1;\bQ)$  is
\begin{align}
\rho_{\snr}(\bQ) &=\frac{1}{\gamma} \Tr(\bG_{\pi_{\snr}}\bH_{\pi_{\snr}}^{-1}\bQ\bH_{\pi_{\snr}}^{-1}) 
\label{eq:VariationalOptBiased}\\
&= 
\frac{1}{\gamma}\inf_{\pi:\E\pi=\gamma} \Tr(\bG_{\pi}\bH_{\pi}^{-1}\bQ\bH_{\pi}^{-1}) \, .
\end{align}
\end{proposition}
In many cases of interest, the set $\{\bx : Z(\bx;\pi_*) = \lambda\}$ has zero measure
(for instance, this is typically the case if the distribution of $\bx$ is absolutely continuous with respect to Lebesgue).
In this case, the optimal $\pi$ selects the samples deterministically to be those
satisfying $Z(\bx_i;\pi_*) > \lambda$.  

In some cases, we can interpret $Z(\bx_i;\pi_*)$ as a score measuring the uncertainty  in predicting the label of sample $\bx_i$
(see examples below). Then this proposition
establishes that, under the non-reweighing scheme (and in the low-dimensional asymptotics) we should select the `hardest' $n$
examples.
\begin{remark}
    The condition $\E\{\nabla_{\btheta}L(\btheta_*;y,\bx)|\bx\}=\bzero$
    (almost surely) could be eliminated at the cost of  enforcing the constraint $\E\{\nabla_{\btheta}L(\btheta_*;y,\bx)\pi(\bx)\}=\bzero$. 
    This would result in a more complicated expression for the data selection rule: we will not pursue this generalization. 
\end{remark}

In the next two subsections, we will specialize non-reweighting subsampling
to linear models and generalized linear models. 
In particular, the examples presented in Section \ref{sec:BiasedLinear}
prove that biased data selection can beat unbiased selection by an arbitrarily large factor, as stated formally below.
\begin{theorem}\label{thm:BiasedBeats}
Consider least-square regression under the model $y_i=\<\btheta_*,\bx_i\>+\eps_i$,
with  $\E(\eps_i|\bx_i)=0$, $\E(\eps_i^2|\bx_i)=\tau^2$. For any $C$, there is
a distribution of the $\bx_i$'s and a $\gamma\in (0,1)$ such that 
$\rho_{\sunb}(\bSigma;\gamma)/\rho_{\snr}(\bSigma;\gamma)>C$.
\end{theorem}

\subsubsection{Linear regression (Proof of Theorem \ref{thm:BiasedBeats})}
\label{sec:BiasedLinear}

Continuing from Example~\ref{example:FirstLinear}, note that the condition 
$\E\{\nabla_{\btheta}L(\btheta_*;y,\bx)|\bx\}=\bzero$ holds.
We now have
$\bG(\bx)=\tau^2\bx\bx^{\sT}$, $\bH(\bx)=\bx\bx^{\sT}$. Since rescaling $\bG$ does not change
the selection rule, we can redefine $\bG(\bx)=\bx\bx^{\sT}$, whence $\bH_{\pi}=\bG_{\pi}=\bSigma_{\pi}$,
where $\bSigma_{\pi}$ 
is the population covariance of the subsampled feature vectors.

A simple calculation shows 
\begin{align}
Z(\bx;\pi)& = \<\bx,\bSigma_{\pi}^{-1}\bQ\bSigma_{\pi}^{-1}\bx\>\, ,\label{eq:ModifiedLeverage}\\
    \pi_{\snr}(\bx) &= \begin{cases}
    1 & \mbox{ if } \<\bx,\bSigma_{\pi_{\snr}}^{-1}\bQ\bSigma_{\pi_{\snr}}^{-1}\bx\> > \lambda\, ,\\
    0 & \mbox{ if } \<\bx,\bSigma_{\pi_{\snr}}^{-1}\bQ\bSigma_{\pi_{\snr}}^{-1}\bx\>< \lambda\,  .
    \end{cases}\label{eq:piLR_NR}
\end{align}
In other words, this scheme selects all data that lay outside a certain ellipsoid.
The shape of the ellipsoid is determined self-consistently by the covariance 
$\bSigma_{\pi_{\snr}}$ of points outside the ellipsoid.

Let emphasize two differences with respect to the standard leverage-score approach of Eq.~\eqref{eq:GeneralizedLeverage}:
\begin{enumerate}
\item[$(i)$]~As for general non-reweighing schemes, selection is essentially deterministic, $\pi_{\snr}(\bx)\in\{0,1\}$ (in particular, if the distribution of $\bx$ has a density, 
then $\pi_{\snr}(\bx)\in\{0,1\}$ with probability one).
\item[$(ii)$]~The original covariance $\bSigma$ is replaced by the covariance of selected data $\bSigma_{\pi}$. As anticipated in Section \ref{sec:Summary}, the selected set depends on $\gamma$ 
in a nontrivial way.
\end{enumerate}

Given that the leverage score of datapoint $\bx_i$ measures how different is $\bx_i$ from the other data,
the modified score of Eq.~\eqref{eq:ModifiedLeverage} can be interpreted as measuring how different is 
$\bx_i$ from selected data.

\begin{example}[One-dimensional covariates]
In order to get a more concrete understanding of the difference
with respect to unbiased data selection, consider 
one-dimensional  covariates $x_i \sim\sP_{X}$ with $\sP_{X}$ of mean zero.
We study prediction error, i.e. $\bQ=\bSigma$.
Let  $X_M$ be the supremum of the support of $\sP_{|X|}$, 
which we assume finite, and $X$ a sample from $\sP_X$.
In this case $Z(x) =\tau^2 x^2/\Sigma$. Provided $\gamma\le \E|X|/X_M$, we have $c(\gamma) = \gamma\E(X^2)^{1/2}/(\tau \E|X|)$
and the  optimal 
asymptotic error for unbiased subsampling is 
\begin{align}
\rcoeff_{\sunb}(\Sigma) = \frac{\tau^2}{\gamma}\cdot \frac{(\E|X|)^2}{\E(X^2)}\, .
\end{align}

On the other hand, assuming $\sP_X$ has a density, the optimal non-reweighting rule takes the form $\pi_*(x) = \bfone(|x|\ge r(\gamma))$.
The coefficient $r=r(\gamma)$ is fixed by $\gamma = \prob(|X|\ge r)$.
The optimal 
asymptotic error for non-reweighting subsampling is
\begin{align}
 \rcoeff_{\snr}(\Sigma)= 
\tau^2\cdot \frac{\E(X^2)} {\E(X^2\bfone_{|X|\ge r})}\, .
\end{align}

The ratio of biased to unbiased error then takes the form
\begin{align}
\frac{\rcoeff_{\snr}(\Sigma)}{\rcoeff_{\sunb}(\Sigma)}= \frac{\E(X^2)^2\prob(|X|\ge r)} 
{(\E|X|)^2\E(X^2\bfone_{|X|\ge r})}\, ,\;\;\;\;\;\; \gamma\le \frac{\E|X|}{X_M}\, .
\end{align}

 It is easy to come up with examples in which this ratio is smaller than one. For instance if 
 $\sP_X = \Unif([-X_M,X_M])$, then
\begin{align}
\frac{\rcoeff_{\snr}(\Sigma)}{\rcoeff_{\sunb}(\Sigma)} = \frac{4\gamma} 
{3(1-(1-\gamma)^3)}\, ,\;\;\;\;\;\;\; \gamma\le \frac{1}{2}\, .
\end{align}
More generally, as $\gamma\to 0$, we get
\begin{align}
\frac{\rcoeff_{\snr}(\Sigma)}{\rcoeff_{\sunb}(\Sigma)}= \frac{\E(X^2)^2} 
{(\E|X|)^2X_M^2}+o(1)\,  ,\;\;\; \gamma\to 0,
\end{align}
Notice that this ratio is always smaller than one (by H\"older's inequality) and can be arbitrarily small.
For instance $\sP_X(\de x)= C_{M,\alpha}|x|^{-\alpha}\bfone_{1\le |x|\le X_M}$, $\alpha>3$
then the above ratio is 
\begin{align}
\frac{\rcoeff_{\snr}(\Sigma)}{\rcoeff_{\sunb}(\Sigma)}= 
\Big(\frac{\alpha-2}{\alpha-3} \cdot \frac{1-X_M^{-\alpha+3}}{1-X_M^{-\alpha+2}}\Big)^2 \frac{1}{X_M^2} +o(1)\, .
\end{align}
To be concrete, for $\alpha=4$, $X_M=10$, unbiased  is suboptimal by a factor larger than $30$.
My choosing $X_M$ large, we can make this ratio arbitrarily small,
hence proving Theorem \ref{thm:BiasedBeats}.
\end{example}

\begin{example}[Elliptical covariates]
The one-dimensional example above is easily generalized to higher dimensions. 
Consider $\bx_i = \bSigma^{1/2}\bz_i$ where $\bz_i$ are spherically symmetric, namely 
$\bz_i= r_i\bu_i$ with $(r_i,\bu_i)\sim\sP_R \otimes \Unif(\S^{d-1}(\sqrt{d}))$.
If we consider test error (and therefore $\bQ=\bSigma$), then it is a symmetry argument 
shows that, for optimal $\pi$ the modified leverage score \eqref{eq:ModifiedLeverage}
coincides with the original one
\begin{align}
Z(\bx;\pi_*) = \<\bx,\bSigma^{-1}\bx\> = \|\bz\|^2_2 = r^2\, .
\end{align}
We therefore recover the one-dimensional case with $\sP_{|X|}$ replaced by $\sP_R$.
\end{example}

\subsubsection{Generalized linear models}
\label{sec:GLM2}

\begin{figure}
\begin{center}
\includegraphics[width=\linewidth]{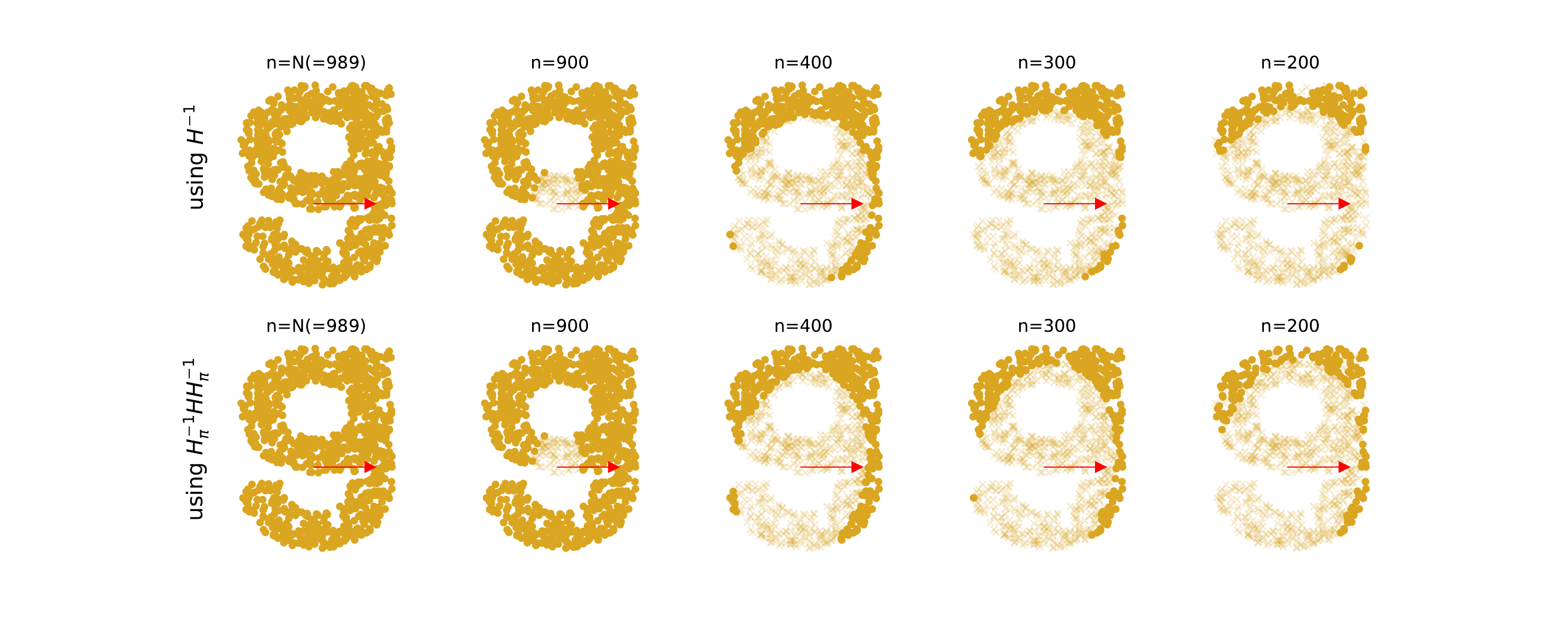}
\end{center}
\vspace{-1cm}
\caption{Data selection in a logistic model
(here we optimize test error with respect to log-loss). Covariates are bi-dimensional and uniformly distributed  
over the letter {\bf g}. The red arrow corresponds to the true parameter vector $\btheta^*$.
Selected points are dark yellow (circles) and non selected ones are light yellow (crosses). Top row: selecting
data with largest value of the influence function. Bottom: optimal 
 non-reweighting selection scheme.}
\label{fig:Opt-Biased}
\end{figure}
Continuing from Example \ref{ex:FirstGLM}, we note that $\E\{y|\bx\}=\phi'(\<\btheta_*,\bx\>)$. Therefore $\E\{\nabla_{\btheta}L(\btheta_*;y,\bx)|\bx\}=\bzero$.
Further 
\begin{align}
\bG(\bx) = \E\{(y-\phi'(\<\btheta_*,\bx\>))^2|\bx\}\cdot \bx\bx^{\sT}= \phi''(\<\btheta_*,\bx\>)
\cdot \bx\bx^{\sT} =\bH(\bx) \, .
\end{align}
whence we have the following generalization of the score \eqref{eq:ModifiedLeverage}:
\begin{align}
Z(\bx;\pi)= \phi''(\<\btheta_*,\bx\>) \cdot \<\bx,\bH_{\pi}^{-1}\bQ\bH_{\pi}^{-1}\bx\>\, .\label{eq:ModifiedLeverage_GLM}
\end{align}
Note that this is similar to the score derived in the unbiased case, cf. Eq.~\eqref{eq:GLM_Score_Unbiased},
with two important differences that we already encountered for linear regression: 
$(i)$~The selection process is essentially deterministic: a datapoint is selected if $Z(\bx_i;\pi)>\lambda$ and not selected if $Z(\bx_i;\pi)<\lambda$;
$(ii)$~The score is computed with respect to the selected data, namely $\bH$ is replaced by $\bH_{\pi}$.
The effect of replacing $\bH$ by $\bH_{\pi}$ is illustrated
on a toy data distribution in  Figure \ref{fig:Opt-Biased}. 

We observe that at high $\gamma$, the two selection procedures are very similar. 
In contrast, at low $\gamma$, selection based on $\bH$ is always biased towards selecting ``hard samples" (in directions roughly orthogonal to $\btheta_*$), whereas selection based on $\bH_{\pi}$ takes into account geometry of selected subset 
and keeps samples in a more diverse set of directions.

We also note that, for GLMs, the optimal asymptotic error coefficient 
\eqref{eq:VariationalOptBiased} takes the particularly simple form
\begin{align}\label{eq:RhoGLM_Biased}
\rho_{\snr}(\bQ) &=
\frac{1}{\gamma}\inf_{\pi:\, \E\pi=\gamma} \Tr(\bH_{\pi}^{-1}\bQ) \, .
\end{align}
 In the case of well-specified GLMs, biased data selection cannot 
improve over full sample estimation. Indeed, $\rho_{\snr}(\bQ) =
\Tr(\E\{\bH(\bx)\pi_{\snr}(\bx)\}^{-1}\bQ)$, and 
$\E\{\bH(\bx)\pi_{\snr}(\bx)\}\preceq \E\{\bH(\bx)\}$.
On the other hand, it is clear that $\rho_{\snr}(\bQ)\le \rho_{\srand}(\bQ)$
(because  $\rho_{\srand}(\bQ)$ corresponds to a special choice of $\pi$ on the
right-hand side of Eq.~\eqref{eq:RhoGLM_Biased}).
Further, the inequality is strict 
(namely, $\rho_{\snr}(\bQ)< \rho_{\srand}(\bQ)$) except (possibly) on a the degenerate
case in which $Z(\bx;\pi)$  is constant on a set of positive measure.

%
%
\subsection{Data selection can beat full-sample estimation}
\label{sec:NonMono}

\begin{figure}
\label{fig:full_sample}
\begin{center}
\includegraphics[width=0.5\linewidth]{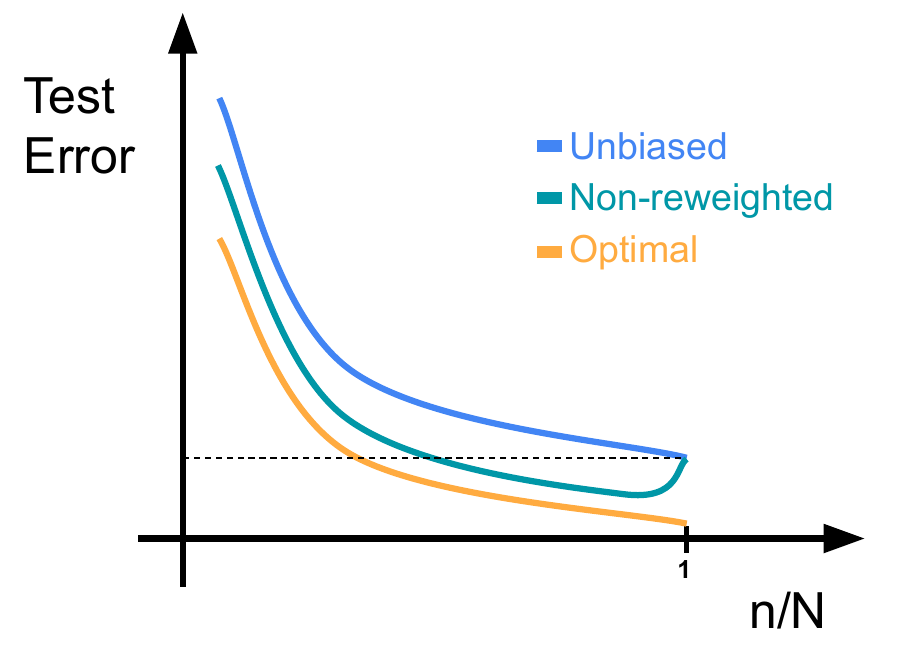}
\end{center}
\vspace{-0.5cm}
\caption{Cartoon of the monotonicity properties of various data-selection schemes, with the selected
sample size.}
\label{fig:Opt-Biased-Cartoon}
\end{figure}
One striking phenomenon illustrated in Figure \ref{fig:Summary} is that
data selection can reduce generalization error. Namely, running empirical risk minimization
(ERM)
with respect to a selected subset of $n<N$  samples out of $N$ total produces
a better model than running ERM on the full dataset. This is the case even when data selection
was based on a relatively weak surrogate, e.g. one trained on $n_{\surr}$ random samples,
so that $n_{\surr}+n$ is significantly below $N$.

Is this non-monotonic behavior compatible with theoretical expectations?
The answer depends on the data selection scheme (see Figure \ref{fig:Opt-Biased-Cartoon} for a cartoon illustration):
\begin{itemize}
\item For unbiased (reweighted) data selection, we saw in Section \ref{sec:UnbiasedFirst} 
that the asymptotic test error is always monotone in the size of the subsample $n$
(at least within the low-dimensional setting studied here). In particular,
full sample data ERM cannot have worse test error than ERM on selected data.
\item Consider then the optimal data selection scheme with reweighting.
We claim that this is also monotone.
Indeed given target sample sizes $n_1<n_2$, one
can simulate data selection at sample size $n_1$ by first selecting $n_2$ samples and then setting to $0$
the weights of $n_2-n_1$ samples. 

However, for $n=N$, this scheme does not reduce to unweighted ERM, but to optimally weighted ERM.
As a consequence, this monotonicity property does not imply that original unweighted
ERM on the full sample has better test error than weighted ERM on a data-selected subsample.
\item The main result of this section will be a proof that non-monotonicity is possible  in a neighborhood of $\gamma=1$ for non-reweighting schemes.
\end{itemize}
\begin{theorem}\label{thm:GeneralDerivative}
Under the   setting of Proposition \ref{propo:Biased}, further assume 
$\bH \succ \bfzero$, $\E\{\|\bG(\bx)\|_{\op}^4\}<\infty$,
$\E\{\|\bH(\bx)\|_{\op}^4\}<\infty$. Then, there exists a constant $C$ such that, for any $\lambda\in \reals$ such that $\prob(Z_{\bQ}(\bx;1)<\lambda)>0$, we have
\begin{align}
\rho_{\snr}(\bQ;\gamma) &\le \rho_{\snr}(\bQ;1)-(1-\gamma)
\E\big\{Z_{\bQ}(\bx;1)\big |Z_{\bQ}(\bx;1)<\lambda\big\} +C(1-\gamma)^{3/2}\, ,
\label{eq:MainGenDeriv}\\
Z_{\bQ}(\bx;1) &:= -\Tr\big\{\bG(\bx)\bH^{-1}\bQ\bH^{-1}\big\}+
2 \Tr\big\{\bH(\bx)\bH^{-1}\bQ\bH^{-1}\bG\bH^{-1}\big\}\, .
\end{align}
Further
\begin{itemize}
\item[$(a)$] If $\partial_{\gamma}\rho_{\snr}(\bQ;1)=-\ess\inf Z_{\bQ}(\bx;1)$. (Note that this is potentially equal
to $+\infty$.)
\item[$(b)$] If $\prob(Z_{\bQ}(\bx;1)<0)>0$, then there exists $\gamma_0=\gamma_0(d)<1$  such that
\begin{align}
\gamma\in (\gamma_0(d),1)\;\; \Rightarrow \;\;\; \rho_{\snr}(\bQ;\gamma) <\rho_{\snr}(\bQ;1)\, .
\label{eq:NonMonoCLaim}
\end{align}
\end{itemize}
\end{theorem}

We next construct specific cases in which $\prob(Z_{\bQ}(\bx;1)<0)>0$.
We consider misspecified linear model. Namely $(y_i,\bx_i)\in\reals\times \reals^d$
with
\begin{align}
\prob(y_i\in A|\bx_i) = \sP(A|\<\btheta_0,\bx_i\>)\, ,\label{eq:Misspecified}
\end{align}
where $\btheta_0\in\reals^d$ is a fixed  vector. 
We will show that both in the case of linear regression and logistic regression, 
there are choices of the conditional distribution $\sP$, for which
$\gamma\mapsto \rho(\bQ;\gamma)$ is strictly increasing near $\gamma =1$.
In this cases, training on a selected subsample provably helps.
\begin{theorem} \label{thm:NonMono}
Assume $\bx_i\sim\normal(0,\id_d)$, and $y_i$ distributed according to
Eq.~\eqref{eq:Misspecified}. Let $\bQ=\bH:=\nabla^2 R(\btheta_*)$,
(as before, $\btheta_*:=\argmin_{\btheta}R(\btheta)$).
Then in each of the following cases, there exists $ \sP(A|t)$ such that
Eq.~\eqref{eq:NonMonoCLaim} holds:
\begin{enumerate}
\item[$(a)$] Least squares, i.e. $L(\btheta;y,\bx) = (y-\<\btheta,\bx\>)^2/2$, $L_{\stest}=L$.
\item[$(b)$] Logistic regression, whereby $y_i\in\{+1,-1\}$, 
$L(\btheta;y_i,\bx_i) = -y_i\<\btheta,\bx_i\>+\phi(\<\btheta,\bx_i\>)$, 
$\phi(t) = \log(e^t+e^{-t})$,  $L_{\stest}=L$.
\end{enumerate}
\end{theorem}
The proofs of Theorem \ref{thm:GeneralDerivative} and Theorem \ref{thm:NonMono} are presented in Appendix \ref{sec:ProofNonMono}.

\subsection{Is unbiased subsampling ever optimal?}

In the Section \ref{sec:BiasedFirst}, we presented examples in which unbiased subsampling is inferior to a
special type of biased subsampling: non-reweighting subsampling.
Given this, it is natural to ask when is unbiased subsampling optimal. 
We will next prove that, in a natural setting, unbiased subsampling is always suboptimal
below a certain subsampling ratio.
\begin{theorem}\label{thm:NeverOptimal}
Consider the setting of Proposition \ref{propo:LowDimAsymptotics}, and consider the prediction error 
under test loss equal to training loss (i.e. $\bQ=\bH:= \nabla^2 R(\btheta_*)$). Further assume
\begin{itemize}
\item[{\sf A1.}]
For $\bG(\bx)$ and $\bH(\bx)$  defined as in Eqs~\eqref{eq:GHDef}, we have $\bG(\bx)=\bH(\bx)$.
(This is for instance the case in maximum likelihood). 
\item[{\sf A2.}] $\bG(\bx)$ is almost surely non-zero has an almost sure upper bound\footnote{Namely, there exists $M<\infty$ such that $\prob\{\|\bG(\bx)\|_2\in (0, M]\} =1$}.
\item[{\sf A3.}] The distribution of $\bG(\bx)/\Tr(\bG(\bx)\bH^{-1})^{1/2}$ is not not supported on any strict affine subspace of 
the set of $d\times d$ symmetric matrices. 
\item[{\sf A4.}] $\E\{\nabla_{\btheta} L(\btheta_*;y,\bx)|\bx\}=\bzero$.
\end{itemize}
Then there exists $\gamma_0>0$ such that, for any $\gamma\in(0,\gamma_0)$ there
is a biased subsampling scheme asymptotically outperforming the best unbiased scheme. 
Namely, there exists a subsampling scheme $S_b$ such that the following inequality holds strictly
\begin{align}   
\rcoeff(S_b;\bH) < \inf_{S\in\cuU} \rcoeff(S;\bH)\, .
\end{align}
\end{theorem}
The proof of this theorem is deferred to Appendix \ref{app:NeverOptimal}.
We construct a perturbation of the unbiased scheme $w(\bx) = (1+\eps\varphi(\bx))/\pi_{\sunb}(\bx)$, and show that it leads to an improvement of the test error for small $\eps$.
Let us emphasize that, as a consequence, we do not prove that a non-reweighing scheme necessarily beats unbiased subsampling.

\begin{remark}
The restriction $\gamma\in (0,\gamma_0)$  in Theorem \ref{thm:NeverOptimal}
is a proof artifact. Under the theorem's assumptions, it implies that the 
unbiased scheme \eqref{eq:UnbFirst} take the simpler form $\pi_{\sunb}(\bx) = c(\gamma) Z(\bx)^{1/2}$, which is more amenable to analysis.
\end{remark}

\begin{example}[Generalized linear models]
    Consider again the GLM model of Example \ref{ex:FirstGLM} (log-likelihood loss).
    Recall that in this case 
    \begin{align*}
    \bG(\bx) = \bH(\bx)=\phi''(\<\btheta_*,\bx\>)\, \bx\bx^{\sT}\, .
    \end{align*}
    Hence condition {\sf A1} is satisfied, and {\sf A4} is also always satisfied.
    If $\prob(\de\bx)$ is supported on $\|\bx\|_2\le M$, and $\prob(\bx=0)=0$
     then condition {\sf A2} is also satisfied. (Note that $\phi''(t)>0$ for any $t>0$ unless 
     $\nu_0=\delta_c$ is a point mass, which is a degenerate case.)

     Finally, for condition {\sf A3} note that
    \begin{align*}
    \frac{\bG(\bx)}{\Tr(\bG(\bx)\bH^{-1})^{1/2}} = \frac{\bx\bx^{\sT}}{\<\bx,\bH\bx\>}\, .
    \end{align*}
    Therefore, a sufficient condition for {\sf A3} is that the support of $\prob(\de \bx)$
    contains an arbitrarily small    ball $\Ball^d(\bx_0;\eps)\subseteq \reals^d$, $\eps>0$.
\end{example}
%
%
\section{Low-dimensional asymptotics: Imperfect surrogates} 
\label{sec:LowDimImperf}

In this section we consider the more realistic setting in which the surrogate model
$\sP_{\surr}(\de y|\bx)$ does not coincide with the actual conditional distribution of $y_i$
given $\bx_i$. 

We will model this situation using a minimax point of view. Namely, we will assume that
the actual conditional distribution is in a neighborhood of the surrogate model,
and will study the worst case error in this neighborhood. 
The minimax theorem then implies that we should use data selection schemes that are optimal for the 
worst data distribution in this neighborhood.
We will not pursue the construction of minimax optimal data selection schemes
in this paper. However, we point out that the broad conclusion
is consistent with some of our empirical findings in Section \ref{sec:NumericalReal}.
Namely, in certain cases using a less accurate surrogate model yields better data selection.

Throughout, we will use the notation $\E_{\surr}\{F(y,\bx)|\bx\} =
\int F(y,\bx)\, \sP_{\surr}(\de y|\bx)$, and similarly for $\prob_{\surr}(\,\cdot\,|\bx)$.

\subsection{Plugin schemes}
 
The simplest approach to utilize an imperfect surrogate 
 proceeds as follows: $(i)$~Choose a data selection scheme under the assumption of ideal surrogate; $(ii)$~Replace the conditional expectations 
$\E\{F(y,\bx)|\bx\}$ in that scheme by expectations with respect to the 
surrogate model $\E_{\surr}\{F(y,\bx)|\bx\}$; $(iii)$~Replace expectations over $\bx$
by expectation over the data sample. 

In particular, revisiting the schemes of Sections \ref{sec:UnbiasedFirst} and \ref{sec:BiasedFirst},
we obtain the following:

\noindent\emph{Plugin unbiased data selection.} We form
\begin{align}
\bG_{\surr}(\bx) & := \E_{\surr}\big\{\nabla_{\btheta}L(\hbtheta^{\surr};y,\bx)
\nabla_{\btheta}L(\hbtheta^{\surr};y,\bx)^{\sT}|\bx\big\}\, ,\\
\bH_{\surr}(\bx) & := \E_{\surr}\big\{\nabla^2_{\btheta}L(\hbtheta^{\surr};y,\bx)|\bx\big\}\, ,
\end{align}
and  subsample according to 
\begin{align}
\pi(\bx) &= \min\Big(1; c(\gamma)\, Z_{\surr}(\bx)^{1/2}\Big)\, ,\\
Z_{\surr}(\bx)&:=\Tr\big(\bG_{\surr}(\bx)\bH_{1,\surr}^{-1}\bQ\bH_{1,\surr}^{-1}\big) \, ,\\
\bH_{1,\surr} & := \E\{\bH_{\surr}(\bx)\} \, .
\end{align}
We then reweight each selected sample proportionally to $1/\pi(\bx)$.
Note that:
\begin{itemize}
\item The `true' parameters vector $\btheta_*$ appearing in $\nabla_{\btheta} L(\,\cdot\,;y,\bx)$,
 $\nabla^2_{\btheta} L(\,\cdot\,;y,\bx)$ 
was replaced by an estimate obtained from the surrogate model.  In certain applications $\hbtheta^{\surr}$
can be `read off' the surrogate model itself. In general, we can define it via 
$\hbtheta^{\surr}:=\arg\min \sum_{i=1}^{N} 
L(\btheta;y^{\surr}_i,\bx_i)$, where $(y^{\surr}_i)_{i\le N}$ are drawn independently according to 
$y_i \sim \sP_{\surr}(\,\cdot\,|\bx_i)$.
\item  The matrix $\bH_{1,\surr}$ can be replaced  its empirical version $\hbH_{1,\surr}:= N^{-1}\sum_{i=1}^{N} 
\bH_{\surr}(\bx_i)$. 
\end{itemize}

\noindent\emph{Plugin non-reweighting data selection.} In this case we select samples such that
$Z_{\surr}(\bx;\pi)>\lambda$, cf. Eq.~\eqref{eq:piNR}, where 
\begin{align}
Z(\bx;\pi) := -\Tr\big\{\bG_{\surr}(\bx)\bH_{\surr,\pi}^{-1}\bQ\bH_{\surr,\pi}^{-1}\big\}+
2 \Tr\big\{\bH_{\surr}(\bx)\bH_{\surr,\pi}^{-1}\bQ\bH_{\surr,\pi}^{-1}\bG_{\surr,\pi}\bH_{\surr,\pi}^{-1}\big\}\, ,
\end{align} 
and $\bH_{\surr,\pi}:=\E\{\bH_{\surr}(\bx)\pi(\bx)\}/\E\{\pi(\bx)\}$ and similarly for $\bG_{\surr,\pi}$. Again, expectations over $\bx$
are replaced by averages over the $N$ samples.

Plugin approaches are natural and easy to define, and in fact we will use them in our simulations.
However, we will show that they can be suboptimal.
Before doing that, we need to make more explicit the notion of optimality that is relevant here.

\subsection{Minimax formulation}

We want formalize the idea that we do not know the conditional distribution of $y$
given $\bx$, but we have some information about it coming from the surrogate model $\sP_{\surr}$.
With this in mind, we introduce a set of probability kernels
\begin{align}
\cuK_d \subseteq \cuK^0_d := \big\{\sP:\cB_{\reals}\times \reals^d\to [0,1] \mbox{ probability kernel } \big\}\, .
\end{align}
Informally, $\cuK_d$ is a neighborhood of the surrogate model $\sP_{\surr}$,
and captures our uncertainty about the actual conditional distribution:
we know that $\prob_{y |\bx}\in \cuK_d$.
For instance, we could consider, for some $r\in [0,1)$,
\begin{align}
\cuK_d(\sP_{\surr};r) := \big\{\sP:\;\; \E_{\bx}\|\sP(\,\cdot\, |\bx )-\sP_{\surr}(\,\cdot\, |\bx )\|_{\sTV}\le r\big\}\, .
\end{align}
We will assume $\cuK_d$ (and its variant $\cuK_{N,d}$ introduced below)
to be convex. Namely, for all $\lambda\in [0,1]$,
\begin{align}
    \sP_0,\sP_1\in \cuK_d  \;\; \Rightarrow\;\;  \sP_{\lambda}(\de y|\bx) = (1-\lambda) \sP_{0}(\de y|\bx) +
    \lambda  \sP_{1}(\de y|\bx) \in  \cuK_d \, .
\end{align}
We are interested in a data selection scheme that works well uniformly over the uncertainty 
encoded in $\cuK_d$.

Given a probability kernel $\sP(\de y|\bx)$
(i.e. $\sP:\cB_{\reals}\times\reals^d\to [0,1]$),
we write $\prob(\sP)$ for the data distribution (on $\reals^d\times \reals$)
induced by $\sP$. Namely $\E_{\prob(\sP)} F(y,\bx) := \E_{\bx}\{\int F(y,\bx)\,\sP(\de y|\bx)\}$.

Let $R_{\stest}(\btheta)= \E L_{\stest}(\btheta; y,\bx)$
be the test error with respect to a certain target distribution $\prob$. Given an estimator $\hbtheta$, we let  $\hbtheta_S(\by,\bX)$
denote its output when applied to data $\by,\bX$ in conjunction with data
selection scheme $S$.
For clarity of notation, we define
\begin{align}
R_{\#}(S;\by,\bX) : = R_{\stest}(\hbtheta_S(\by,\bX))\, .
\end{align}

We  define the minimax risk $R_{\sMM}(\cuK_d)$ by
(here we recall that $\cuA$ is the set of all data selection methods)
\begin{align}
R_*(S;\cuK_d) &:= \sup_{\sP\in  \cuK_d} \E_{\by,\bX\sim \prob(\sP)}
R_{\#}(S;\by,\bX) \, ,\label{eq:Rstar}\\
R_{\sMM}(\cuK_d) &:= \inf_{S\in \cuA} R_*(S;\cuK_d)\, .\label{eq:RMM}
\end{align}
We seek near optimal data selection schemes $S$, namely
schemes such that $R_*(S;\cuK_d) \approx R_{\sMM}(\cuK_d)$.
We note that the expectation in Eq.~\eqref{eq:Rstar} includes expectation over the 
randomness in $S$.

\begin{remark}
The set $\cuK_d$  provides information about $\btheta_*$. This information can be exploited in other ways than 
via data selection. For instance, we could restrict the empirical risk minimization
of Eq.~\eqref{eq:Mestimators} to a set that is ``compatible'' with $\cuK_d$.
However, we are only interested in procedures that follow the general data selection framework 
defined in the previous sections and hence are not necessarily optimal against this
broader set. 
\end{remark}

\subsection{Duality and its consequences}

We can apply Sion's minimax theorem to a relaxation of
$R_{\sMM}(\cuK_d)$. Namely,
\begin{itemize}
\item We replace $\cuK_{d}$ by a set of probability kernels 
$\reals^{N\times d}$ to $\reals^N$:
\begin{align}
\cuK_{N,d} \subseteq \cuK^0_{N,d} := \big\{\sP:\cB_{\reals^N}\times \reals^{N\times d}\to [0,1] \mbox{ probability kernel } \big\}\, .
\end{align}
such that
each marginal of $\sP\in \cuK_{N,s}$ is in $\cuK_d$. In
other words, we allow for entries of $\by$ to be conditionally dependent, 
given $\bX$.
Generalizing the notations above, $\prob(\sP_N)$
denotes the induced distribution on $\by,\bX$.
\item We replace the space of data selection schemes $\cuA$ by a  the set $\ocuA$  of probability kernels
$\sQ$ such that, for any $A\subseteq [N]$, and any $\bX\in\reals^{N\times d}$, $\sQ(A|\bX)$ is
the conditional probability of selecting data in the set $A$ given covariate vectors $\bX$. In other words, we consider more
general data-selection schemes in which the selected set is allowed to depend on all the data points.
\end{itemize}

The following result is an  application of the standard minimax 
theorem in statistical decision theory, 
see e.g. \cite[Section 3.7]{liese2008statistical}.
\begin{theorem}\label{thm:Minimax}
Assume that 
any $\sP_N\in \cuK_{d,N}$ is supported on $\|\by\|\le M$,
and that $(\by,\bX) \mapsto R(\hbtheta_A(\by,\bX))$ is continuous
for any $A$.
Define 
\begin{align}
\oR_{\sMM}(\cuK_d) := \inf_{S\in \ocuA} \oR_*(S;\cuK_d):=
\inf_{S\in \ocuA}\sup_{\sP_N\in  \cuK_{d,N}} \E_{\by,\bX\sim \prob(\sP_N)}
R_{\#}(S;\by,\bX)\, .
\end{align}
Then we have
\begin{align}
\oR_{\sMM}(\cuK_d) = \sup_{\sP_N\in  \cuK_{d,N}}
\inf_{S\in \ocuA} \E_{\by,\bX\sim \prob(\sP_N)}
R_{\#}(S;\by,\bX)\, .
\end{align}
Further, assume $\sP_{\sMM}$ achieves the supremum over $\cuK_d$ above.
Then any
\begin{align}
S_{\sMM} \in \arg\min_{S\in \ocuA} \E_{\by,\bX\sim \prob(\sP_\sMM)}
R_{\#}(S;\by,\bX)
\end{align}
achieves the minimax error.
\end{theorem}

As is common in estimation theory, the minimax theorem can be difficult to apply 
since computing  the supremum over  $\sP\in  \cuK_d$ is in general very difficult.
Nevertheless, the theorem implies the following important insight. We should not perform data selection
by plugging in the surrogate conditional model for $y$ given $\bx$ 
for the actual one. Instead, we should optimize data selection
as if labels were distributed according to the `worst' conditional model in a neighborhood of
the surrogate.

 Below we work out a toy case to illustrate this insight.
Instead of studying the minimax problem for the finite-sample risk $R$, we will consider
its asymptotics defined in Proposition \ref{propo:LowDimAsymptotics}.
We define the asymptotic minimax coefficients $\rcoeff_{*}(S;\cuK_d)$ and
$\rcoeff_{\sMM}(\cuK_d)$ in analogy with Eqs.~\eqref{eq:Rstar} and \eqref{eq:RMM}.
\begin{example}[Discrete covariates]
Consider $x$ taking values in $[k]=\{1,\dots,k\}$, with probabilities $p_{\ell}=\prob(x=\ell)$,
and $y$ taking values in $\{0,1\}$, with $\prob(y=1|x)=\theta^*_x$, whereby 
$\btheta^*=(\theta^*_1,\dots,\theta^*_k)$ is a vector of unknown parameters. 
We estimate $\btheta^*$ using empirical risk minimization with respect to log-loss
\begin{align}
L(\btheta;y,x) = -y\log \theta_x-(1-y)\log(1-\theta_x)\, . 
\end{align}
We are interested in the quadratic estimation error $\|\hbtheta-\btheta^*\|_{\bQ}$ with 
$\bQ=\diag(q_1,q_2,\dots,q_k)$. We consider a non-reweighting subsampling scheme whereby
a sample with covariate $x$ is retained independently with probability $\pi_x$.
Either applying Proposition \ref{propo:LowDimAsymptotics}, or by a straightforward calculation, 
we obtain:
\begin{align}
\E\big\{\|\hbtheta^S-\btheta^*\|_{\bQ}^2\big\} & = \frac{1}{N}\, \rho(\pi;\btheta^*,\bQ) +o(1/N)\, ,\\
 \rho(\pi;\btheta^*,\bQ) & = \sum_{x=1}^k\frac{\theta^*_x(1-\theta^*_x)}{\pi_xp_x}q_x\, .
\end{align}
(Here we made explicit the dependence on $\btheta^*$.)
For $\cuK\subseteq \reals^k$ a convex set, we then 
define\footnote{Notice that we are not taking the infimum over all 
randomized data selection schemes as in  Theorem \ref{thm:Minimax}. For simplicity, we are restricting the minimization to non-reweighting schemes. 
The substance of Theorem \ref{thm:Minimax} does not change since this set is convex.}
\begin{align}
\rcoeff_{*}(\pi;\cuK,\bQ) & := \sup_{\btheta^*\in \cuK}  \rho(\pi;\btheta^*,\bQ) \, ,
\;\;\;\;\; \rcoeff_{\sMM}(\cuK,\bQ) := \inf_{\pi}\rcoeff_{*}(\pi;\cuK,\bQ)\, .\label{eq:MMaxSpecial}
\end{align}

We will assume that $\cuK$ is closed (hence compact) and convex. We can apply the minimax
theorem to Eq.~\eqref{eq:MMaxSpecial}:
\begin{align}
\rcoeff_{\sMM}(\cuK,\bQ) :=  \max_{\btheta^*\in \cuK}\min_{\pi}\rcoeff(\pi;\btheta^*,\bQ)\, .
\end{align}
Hence, the minimax optimal data selection strategy is obtained by selecting the 
optimum $\pi$ for the worst case $\btheta$, to be  denoted by $\btheta^{\sMM}$.
A simple calculation yields
\begin{align}
\pi_x^{\sMM} &= \min\Big(\frac{c(\gamma)}{p_x}\sqrt{q_x\theta^{\sMM}_x(1-\theta^{\sMM}_x)}; 1\Big)\, ,\label{eq:PiDiscrete}\\
\btheta^{\sMM}_s & = \arg\max_{\btheta\in\cuK} \sum_{x=1}^k
\max\Big(\frac{1}{c(\gamma)}\sqrt{q_x\theta_x(1-\theta_x)};
\frac{q_x}{p_x}\theta_x(1-\theta_x)\Big)\, ,\label{eq:SimpleMMax}
\end{align}
where $c(\gamma)$ is obtained by solving
\begin{align}
\sum_{x=1}^k 
\min\Big(c(\gamma)\sqrt{q_x\theta^{\sMM}_x(1-\theta^{\sMM}_x)}; p_x\Big) = \gamma\, .
\end{align}
The above formulas have a clear interpretation (for simplicity we neglect the factor $q_x$).  Note that 
$\theta^{\sMM}_x(1-\theta^{\sMM}_x)$ can be interpreted as measure of uncertainty in predicting 
$y_{\snew}$ at a test point $x_{\snew}=x$  under the minimax model $\btheta^{\sMM}$.
We select data so that the fraction of samples with 
$x_i=x$ is proportional to the square root of this uncertainty.
Assuming that the inner maximum  in Eq.~\eqref{eq:SimpleMMax} is achieved on
the first argument, the minimax model itself is the one that maximize total uncertainty within the set $\cuK$.

For instance, if $\cuK= [\theta_{\surr,1}-\eps,\theta_{\surr,1}+\eps]\times\cdots\times[\theta_{\surr,k}-\eps,\theta_{\surr,k}+\eps]\subseteq [0,1/2]^k$, then $\theta^{\sMM}_{x}=
\theta_{\surr,x}+\eps$ for all $x$. We then obtain
\begin{align}
\pi_x^{\sMM} &= \min\Big(\frac{c(\gamma)}{p_x}\sqrt{q_x(\theta_{\surr,x}+\eps)(1-\theta_{\surr,x}-\eps)}; 1\Big)
\end{align}
In other words, we should not use the uncertainties computed within the surrogate model,
but within the most uncertain model in its neighborhood.
\end{example}
%

%
\section{High-dimensional asymptotics: Generalized linear models}
\label{sec:HighDim}

In this section we study data selection within a high dimensional setting whereby
the number of samples $N$ and the dimension $p$ diverge simultaneously,
while their ratio converges to some $\delta_0\in (0,\infty)$. We are therefore studying the limit $n,N,p\to\infty$ with 
\begin{align}
\frac{n}{N}\to \gamma\, ,\;\;\;\;\;
\frac{N}{p}\to \delta_0\, ,
\end{align}
with $\gamma,\delta_0\in (0,\infty)$. 

We will restrict ourselves to the case of
(potentially misspecified) generalized linear models already introduced in Example \ref{ex:FirstGLM} and Section \ref{sec:NonMono}. In this case, we can
derive an asymptotically exact characterization of the risk using a well established strategy 
based on Gaussian comparisons inequalities.

In the next section, we will evaluate theoretical predictions
derived in this one, and compare them with numerical simulations on synthetic data.

As we will see in this section and the next,
the high-dimensional setting allow us to unveil a few interesting
phenomena, namely: 
$(i)$~Biasing data selection towards hard samples (those that are uncertain under the surrogate model) can be suboptimal;
$(ii)$~Even when biasing towards hard samples is effective, selecting the top hardest one can lead to poor behavior at small $\gamma$;
$(iii)$~A one parameter family of selection probabilities
introduced in the next section is broadly effective.

\subsection{Setting}
\label{sec:SettingHiDim}

We want to focus on the optimal use of the surrogate model,
and how this changes from low to high-dimension. With this motivation in mind,
we consider a model in which the covariates carry little or no information about the value 
of a sample. 
Namely, we assume isotropic covariates $\bx_i\sim \normal(0,\id_p)$, and 
responses $y_i$ that depend on a one-dimensional projection of the data:
\begin{align}
\prob(y_i\in A|\bx_i) = \sP(A|\<\btheta_0,\bx_i\>)\, ,
\end{align}
where, for each $z$,  $\sP(\,\cdot\, |z)$ is a probability distribution over $\reals$.
Note that this setting includes, as a special case, generalized linear
models (see Example \ref{ex:FirstGLM}), whereby, for $z_i=\<\btheta_*,\bx_i\>$):
$\sP(\de y_i|z_i)  = \exp\big\{y_i z_i-\phi(z_i)\big\}\, \nu_0(\de y_i)$.
It also includes the special misspecified binary response model of Section
\ref{sec:NonMono} in which case $y_i\in\{+1,-1\}$
and $\prob(y_i=+1|\bx_i) = f(\<\btheta_0,\bx_i\>)=1-\prob(y_i=-1|\bx_i)$.

We specialize the post-data selection empirical risk minimization
of Eq.~\eqref{eq:ER2} to
\begin{align}
\hR_N(\btheta):= \frac{1}{N}\sum_{i=1}^N S_i(\<\surrth,\bx_i\>) \, L(\<\btheta,\bx_i\>,y_i)
+\frac{\lambda}{2}\|\btheta\|_2^2\, .\label{eq:ER_HiDim}
\end{align}
where $L:\reals\times \reals\to\reals$ is a loss function (we abuse notations and replace
$L(\btheta;y_i,\bx_i)$ by $L(\<\btheta,\bx_i\>,y_i)$). This is a special case of Eq.~\eqref{eq:ER2}
in two ways: \emph{first,} the loss depends on $\btheta$, $\bx_i$ only through $\<\btheta, \bx_i\>$
and, \emph{second,} we focus on a ridge regularizer. Our characterization will require $L$
to be convex in its first argument (hence $L(\<\btheta,\bx_i\>,y_i)$ will be convex in $\btheta$).
As before, it is understood that $S_i(\<\surrth,\bx_i\>)$  depends on some additional 
i.i.d. randomness, namely $S_i(\<\surrth,\bx_i\>) = s(\<\surrth,\bx_i\>,U_i)$, where $(U_i)_{i\le N}\sim_{iid}\Unif([0,1])$.

Denoting by $\hbtheta_{\lambda}:= \argmin_{\btheta} \hR_N(\btheta)$, we will consider a test error
of the form 
\begin{align}
R_{\stest}(\hbtheta_\lambda) & = \E L_{\stest}(\<\hbtheta_{\lambda},\bx_{\snew}\>,y_{\snew})\, .
\end{align}
We will also consider the excess error $R_{\sexc}(\hbtheta_\lambda) :=
R_{\stest}(\hbtheta_\lambda)-\inf_{\btheta}R_{\stest}(\btheta)$.

Note that the data distribution is invariant under rotations in feature space
that leave $\btheta_0$ unchanged. In other words, if $\bQ\in \reals^{p\times p}$
is an orthogonal matrix such that $\bQ\btheta_0=\btheta_0$ (with $\bQ$
independent of the data),  then $(y_i,\bQ\bx_i)$ is distributed as  $(y_i,\bx_i)$.
As a consequence, and because of the rotational invariance of the ridge regularizer, the 
behavior of the empirical risk minimization \eqref{eq:ER_HiDim} 
depends on the surrogate model only through the following parameters:
\begin{align}
\beta_0&:= \lim_{N,p\to\infty}\frac{\<\surrth,\btheta_0\>}{\|\btheta_0\|},\;\;\;\;\;\;
\beta_s:= \lim_{N,p\to\infty}\big\|\bP_{0}^{\perp}\surrth\big\|_2\, ,
\end{align}
where $\bP_{0}^{\perp}$ is the projector orthogonal to $\btheta_0$.
In words, $\beta_0$ is the size of the projection
of the surrogate model onto the direction of the true model,   while
$\beperp$ is the size of the perpendicular direction.
\begin{remark}\label{rmk:PopMin}
Of course in this setting the population risk minimizer $\btheta_*$
does not coincide necessarily with $\btheta_0$. Instead, we
have $\btheta_* = c_*\btheta_0/\|\btheta_0\|$, where
\begin{align}
c_* = \argmin_{c\in \reals}\,  \E \,L(c\, G_0;Y)\, ,
\end{align}
and expectation is with respect to $G_0\sim \normal(0,1)$,
$Y\sim\sP(\;\cdot\;|\,\|\btheta_0\|G_0)$.
\end{remark}
%
%
\subsection{Asymptotics of the estimation error}

The high-dimensional asymptotics of the test error is determined by 
a saddle point of the following Lagrangian
(here and below $\balpha :=(\alpha_0,\alpha_s,\aperp)$, $\bbeta :=(\beta_0,\beta_s,0)$):
%
%
%
\begin{align}\label{eq:GeneralGLM}
\cuL(\balpha,\mu,\omega):= 
\frac{\lambda}{2}\|\balpha\|^2-\frac{1}{2\delta_0}\mu\aperp^2
+ \E\Big\{ \min_{u\in\reals}
\Big[S(\<\bbeta,\bG\>) \, L(\alpha_0G_{0}+\alpha_s G_{s}+u ,Y)+\nonumber\\
\frac{1}{2}\mu(\alphap G_{\perp}-u)^2\Big] \Big\}
\end{align}
Here expectation is with respect to 
\begin{align}
\bg=(G_0,G_s,G_{\perp})\sim\normal(\bfzero,\id_3),\;\;\;\;\;\;
Y\sim \sP(\;\cdot\; |\,\|\btheta_0\|_2G_0)\, .\label{eq:JointRV}
\end{align}
as well as the randomness in $S$. 

\begin{theorem}\label{thm:HiDim}
Assume $u\mapsto L(u,y)$ is convex, continuous, with at most quadratic growth,
and $\lambda>0$.
Further denote by $\balpha^*,\mu^*$ the solution of the following 
minimax problem ($\balpha^*$ is uniquely defined by this condition)
\begin{align}
\min_{\balpha}\max_{\mu\ge 0}\cuL(\balpha,\mu)\, .
\end{align}

Then the following hold in the limit $N,p\to \infty$, with $N/p\to\delta_0$:
\begin{enumerate}
    \item[$(a)$] If $(u,z)\mapsto \int L_{\stest}(u;y)\sP(\de y|z)$ is a continuous function with
at most quadratic growth, we have
\begin{align}
\plim_{N,p\to\infty}R_{\stest}(\hbtheta_{\lambda}) & = \E L_{\stest}\Big(\alpha^*_0 G_0+\sqrt{(\alpha^*_s)^2+(\alpha^*_{\perp})^2}G;Y\Big) \, ,
\end{align}
where expectation is taken with respect to the joint distribution of Eq.~\eqref{eq:JointRV}.
\item[$(b)$] If $\btheta_0\in\argmin R_{\stest}(\btheta)$, then the excess risk is given by
\begin{align}
\plim_{N,p\to\infty}R_{\sexc}(\hbtheta_{\lambda}) & = \E L_{\stest}\Big(\alpha_0^* G_0+\sqrt{(\alpha^*_s)^2+(\alpha^*_{\perp})^2}G;Y\Big)
- \E L_{\stest}\Big(c_* G_0;Y\Big)\, .
\end{align}
(See Remark \ref{rmk:PopMin} for a definition of $c_*$.)
\item[$(c)$] Letting $\bP_0^{\perp}$ be the projector orthogonal to
$\btheta_0$ and $\bP^{\perp}$ the projector orthogonal to ${\rm span}(\btheta_0, \surrth)$,
we have
\begin{align}
    \plim_{N,p\to\infty} \frac{\<\hbtheta_{\lambda},\btheta_0\>}{\|\btheta_0\|}=\alpha^*_0\, ,\;\;\;\;\;
      \plim_{N,p\to\infty} \frac{\<\hbtheta_{\lambda},\bP^{\perp}_0\surrth\>}{\|\bP^{\perp}_0\surrth\|}=\alpha^*_s\, ,\;\;\;\;\;
       \plim_{N,p\to\infty} \|\bP^{\perp}\hbtheta_{\lambda}\|=\alpha^*_{\perp}\, .
       \label{eq:LimitProj}
\end{align}
\item[$(d)$] The asymptotic subsampling fraction is given by
\begin{align}
\frac{n}{N}\to \gamma = \prob(S(\beta_0 G_0+\beta_s G_s)>0)\, .
\end{align}
\end{enumerate}
\end{theorem}

The proof of Theorem \ref{thm:HiDim} is deferred to Appendix \ref{sec:ProofHiDim}. We can further specialize the above formulas to the two classes of data selection schemes
studied before: unbiased  and non-reweighting schemes.

\paragraph{Unbiased data selection.} In this case $S_i(\<\surrth,\bx_i\>)= 1/\pi(\<\surrth,\bx_i\>)$
with probability $\pi(\<\surrth,\bx_i\>)$
and $S_i(\<\surrth,\bx_i\>)=0$ otherwise. The Lagrangian reduces to
\begin{align}
\cuL(\balpha,\mu):= &\frac{\lambda}{2}\|\balpha\|^2-\frac{1}{2\delta_0}\mu\aperp^2+
\E\Big\{\min_{u\in\reals}\big[L(\alpha_0 G_0+\alpha_s G_s+u;Y)
+\frac{1}{2}\mu \, \pi(\<\bbeta,\bg\>)\big(u-\aperp G_{\perp}\big)^2\big]\Big\}\, .\nonumber
\end{align}
%
%

\paragraph{Non-reweighting data selection.} In this case $S_i(\<\surrth,\bx_i\>)= 1$
with probability $\pi(\<\surrth,\bx_i\>)$
and $S_i(\<\surrth,\bx_i\>)=0$ otherwise. The Lagrangian reduces to

\begin{align}
\cuL(\balpha,\mu):= &
\frac{\lambda}{2}\|\balpha\|^2-\frac{1}{2\delta_0}\mu\aperp^2
+\E\Big\{\pi(\<\bbeta,\bg\>)\min_{u\in\reals}\big[L(\alpha_0 G_0+\alpha_s G_s+u;Y)
+\frac{1}{2}\mu\, \big(u-\aperp G_{\perp}\big)^2\big]\Big\}\, .
\label{eq:NR-Lagrangian}
\end{align}
%
%
%

\subsection{The case of misspecified linear regression}

We revisit the case of misspecified linear regression, already studied 
in Section \ref{sec:NonMono}. We assume a non-reweighting data-selection scheme,
whose asymptotic behavior is characterized by the Lagrangian
\eqref{eq:NR-Lagrangian}.

In the case of square loss, the inner minimization over $u$ is easily solved and
we can then perform the maximization over $\mu$
in Theorem \ref{thm:HiDim}
analytically. This calculation yields the following 
Lagrangian
\begin{align}
\cuL_{\sls}(\balpha):= &
\frac{1}{2}\left(
\sqrt{\E\Big\{\pi(\<\bbeta,\bg\>)\big[Y-\<\balpha,\bg\>\big]^2\Big\}}
-\frac{\aperp}{\sqrt{\delta_0}}\right)_+^2+
\frac{\lambda}{2}\|\balpha\|^2_2\, .\label{eq:LagrangianLS}
\end{align}
We then have the following consequence of Theorem \ref{thm:HiDim}.
\begin{corollary}
Assume the misspecified generalized linear model of 
Section \ref{sec:SettingHiDim}, and further consider the case of square loss 
$L(\btheta;y,\bx)=(y-\<\btheta,\bx\>)^2/2$.  
Let  $\balpha^*=(\alpha^*_0,\alpha^*_s,\aperp^*)$ be the unique minimizer
of the Lagrangian $\cuL_0$.
Then claims $(a)$ to $(d)$ of Theorem \ref{thm:HiDim} hold.
\end{corollary}

The next statement gives a particularly simple expression in the case of perfect surrogate, and ridgeless limit $\lambda\to 0$. 
For $n>p$ this is standard least squares, while for $n\le p$,
this is minimum $\ell_2$ norm interpolation. Its proof is deferred to Appendix \ref{sec:LinearHighD}.
\begin{proposition}\label{propo:LinearHighD}
Assume the misspecified generalized linear model of 
Section \ref{sec:SettingHiDim}, and further consider the case of square loss  $L(\btheta;y,\bx)=(y-\<\btheta,\bx\>)^2/2$,
and further consider the case of a perfect surrogate 
which, without loss of generality, we assume normalized:
$\surrth = \btheta_0/\|\btheta_0\|$.
Define the quantities 
\begin{align}
A_{\pi} := \frac{1}{\gamma}
\E[G^2\pi(G)]\, ,\;\;\;\;\; B_{\pi} :=  \frac{1}{\gamma}\E[G Y\pi(G)]
\, ,\;\;\;\;\; C_{\pi} :=  \frac{1}{\gamma}\E[Y^2\pi(G)]\, ,
\end{align}
where expectation is with respect to $G\sim\normal(0,1)$ and
$Y\sim\sP(\,\cdot\,|G)$. In particular,
we let $A_1$, $B_1$, $C_1$ be the quantities defined above with 
$\pi(t) = 1$ identically.

Then the asymptotic excess risk of ridgeless regression is
as follows

\textup(here $\delta:=\delta_0\gamma=\lim_{n,p\to\infty} \left(n/p\right)$\textup):
\begin{enumerate}
\item For $\delta>1$:
\begin{align}
   \lim_{\lambda\to 0}\plim_{N,p\to\infty}R_{\sexc}(\hbtheta_{\lambda})
   = \lt( \frac{B_1}{A_1}-\frac{B_{\pi}}{A_{\pi}} \rt)^2+
   \frac{1}{\delta-1}\cdot \lt(C_{\pi }-\frac{B^2_{\pi}}{A_{\pi}}\rt)\, .\label{eq:Risk_LS_Interp_Underp_delgt1}
\end{align}
\item For $\delta<1$:
\begin{align}
   \lim_{\lambda\to 0}\plim_{N,p\to\infty}R_{\sexc}(\hbtheta_{\lambda})
   = &\lt(\frac{B_1}{A_1} - \frac{B_{\pi}\delta}{1-\delta+A_{\pi}\delta}\rt)^2
    +   \frac{B_\pi^2}{A_\pi}\cdot \frac{\delta(1-\delta)}{(1-\delta+A_{\pi}\delta)^2}
   \label{eq:Risk_LS_Interp_Underp_dellt1}\\
   &+
   \frac{\delta}{1-\delta}\cdot \lt(C_{\pi }-\frac{B^2_{\pi}}{A_{\pi}}\rt)
   \, .\nonumber
\end{align}
\end{enumerate}
\end{proposition}

\begin{remark}[Random data selection]
    In the case of no data selection, $\pi(\bx)=\gamma$,
we have $A_{\pi}=A_1=1$,  
$B_{\pi}= B_1$, $C_{\pi}=C_1$, and 
we recover the result of ordinary ridgeless regression \cite{advani2020high, hastie2022surprises}
\begin{align}
\delta>1  \; &:\;\;\;\;\;
\lim_{\lambda\to 0}\plim_{N,p\to\infty}R_{\sexc}(\hbtheta_{\lambda})
   = 
   \frac{1}{\delta-1}\cdot \lt(C_{1}-\frac{B^2_{1}}{A_{1}}\rt)\, ,\\
\delta<1  \; &:\;\;\;\;\;
\lim_{\lambda\to 0}\plim_{N,p\to\infty}R_{\sexc}(\hbtheta_{\lambda})
   = B_1^2(1-\delta)+
   \frac{\delta}{\delta-1}\cdot \lt(C_{1}-\frac{B^2_{1}}{A_{1}}\rt)\, .
\end{align}
\end{remark}
%
%
\section{Numerical results: Synthetic data}
\label{sec:NumericalSynth}

\begin{figure}
\begin{center}
\includegraphics[width=0.9\linewidth]{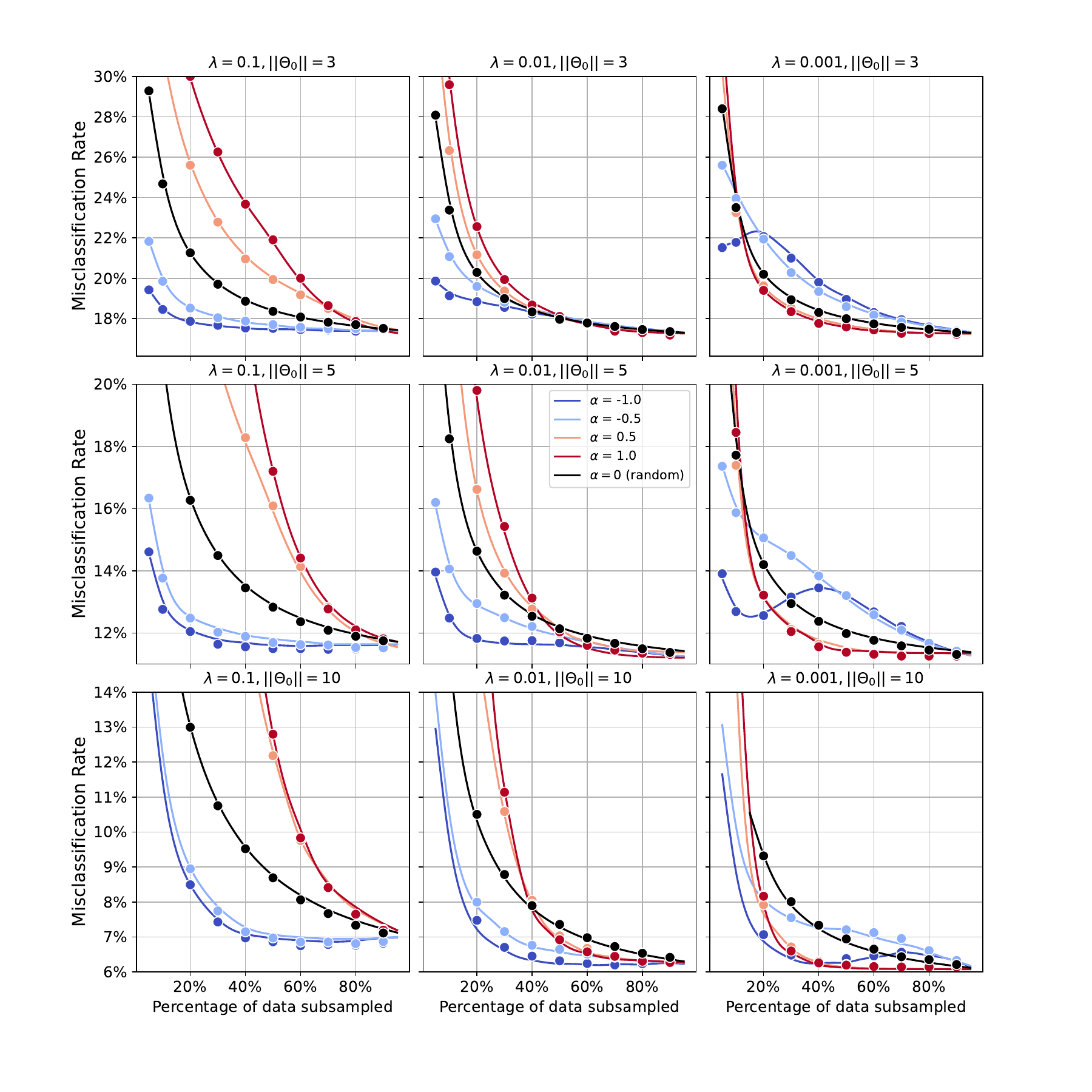}
\end{center}
\caption{Misclassification error for logistic regression  
after non-reweighted subsampling under the scheme of 
Eq.~\eqref{eq:Scheme}. Here $N=34345$, $p=932$.
Circles: simulations. Continuous lines: theory. Each panel
corresponds to a different choice of $\|\btheta_0\|$, 
$\lambda$. Different colors represent different values of the exponent $\alpha$ in Eq.~\eqref{eq:Scheme}. }
\label{fig:Synthetic-Mis-Error-Lowdim}
\end{figure}

In this section we present numerical simulations within the synthetic data model introduced in Section \ref{sec:SettingHiDim}.
We consider the case of
binary labels $y_i\in\{+1,-1\}$. Summarizing,
we generate isotropic feature vectors $\bx_i\sim\normal(0,\id_p)$,
and labels with $\prob(y_i=+1|\bx_i) = f(\<\btheta_0,\bx_i\>)$.
We then perform  data selection,  with selection
probability
\begin{align}
\pi(\bx_i) &= \min\Big(c(\gamma)\, \phi''(\<\surrth,\bx_i\>) ^{\alpha};\;1\Big)\, ,\label{eq:Scheme}
\end{align}
where we remind the reader that $\phi$ is the log-moment generating function. In the binary case
$\phi(t) = \log \left(2\cosh(t)\right)$ and hence $\phi''(t)=1-\tanh(t)^2$
is label variance under the logistic model.
We fix the constant  $c(\gamma)$, chosen such that $\sum_{i\le N}\pi(\bx_i)  = n$.
We can rewrite the above formula as
\begin{align}
\pi(\bx_i) &= \min\Big(c(\gamma)\, p(\<\surrth,\bx_i\>)^{\alpha}
(1-p(\<\surrth,\bx_i\>))^{\alpha};\; 1\Big)
\, ,\label{eq:SchemeB}
\end{align}
where  $p(\<\surrth,\bx_i\>)=(\tanh(\<\surrth,\bx_i\>) + 1) / 2 = 
(1+e^{-2\<\surrth,\bx_i\>})^{-1}$ is the probability of $y_i=+1$ under the surrogate model (note: binary labels are defined over $y_i\in\{+1,-1\}$ and not $y_i\in\{+1, 0\}$).
Hence, $\alpha>0$ upsamples data points with higher uncertainty
under the surrogate model (`difficult' data),
while $\alpha<0$ upsamples data points with lower uncertainty
 (`easy' data). We then fit ridge regularized logistic regression 
 to the selected 
 data (cf. Eq.~\eqref{eq:ER_HiDim}) and evaluate misclassification 
error on a hold-out test set. 

It is instructive to compare the above scheme with influence function-based
data selection, cf. Example \ref{ex:FirstGLM}. Within the present data model,
the population Hessian takes the form $\bH = a_+ \id+b_+ \btheta_0\btheta_0^{\sT}/\|\btheta_0\|^2$
where $a_+= \E\{\phi''(\|\btheta_0\|G)\}$, $b_+= \E\{\phi''(\|\btheta_0\|G)(G^2-1)\}$.
The score of Example \ref{ex:FirstGLM} (cf. Eq.~\eqref{eq:GLM_Score_Unbiased})
reads (an overall factor $da_-$ is immaterial and introduced for convenience)
\begin{align}
Z(\bx_i) = \phi''(\<\btheta_0,\bx_i\>)
\Big\{\frac{\|\bx_i\|^2}{d}+b_-\frac{\<\btheta_0,\bx_i\>^2}{da_-\|\btheta_0\|^2}\Big\}\, ,
\end{align}
where $a_-=1/a_+$, $b_- = (a_++b_+)^{-1}-1/a_+$. For large dimension $d$,
$\|\bx_i\|^2/d=1+O_P(1/\sqrt{d})$, and  $\<\btheta_0,\bx_i\> =O_P(1)$. 
We therefore get
\begin{align}
\label{eq:inf_fn_exp}
Z(\bx_i) =  \phi''(\<\btheta_0,\bx_i\>)\cdot\big(1+O_P(1/\sqrt{d})\big)\, .
\end{align}
Therefore influence-function based subsampling essentially  corresponds to the 
case $\alpha=1/2$ of Eq.~\eqref{eq:Scheme}.

\begin{figure}
\begin{center}
\includegraphics[width=0.9\linewidth]{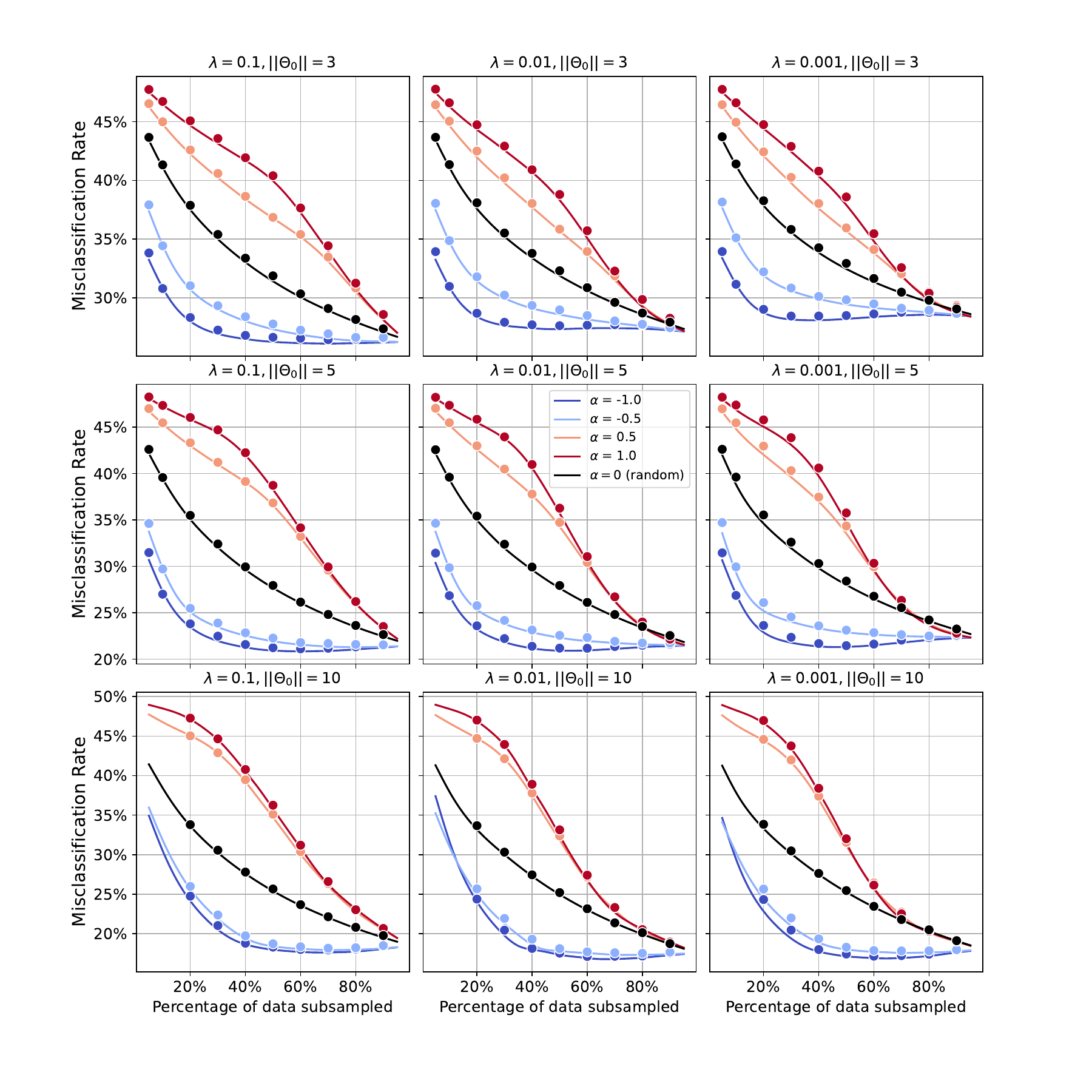}
\end{center}
\vspace{-1cm}
\caption{Same as Figure \ref{fig:Synthetic-Mis-Error-Lowdim},
with $N=6870$ and $p=3000$.}
\label{fig:Synthetic-Mis-Error-Highdim}
\end{figure}

\begin{figure}
\centering
\includegraphics[width=0.9\linewidth]{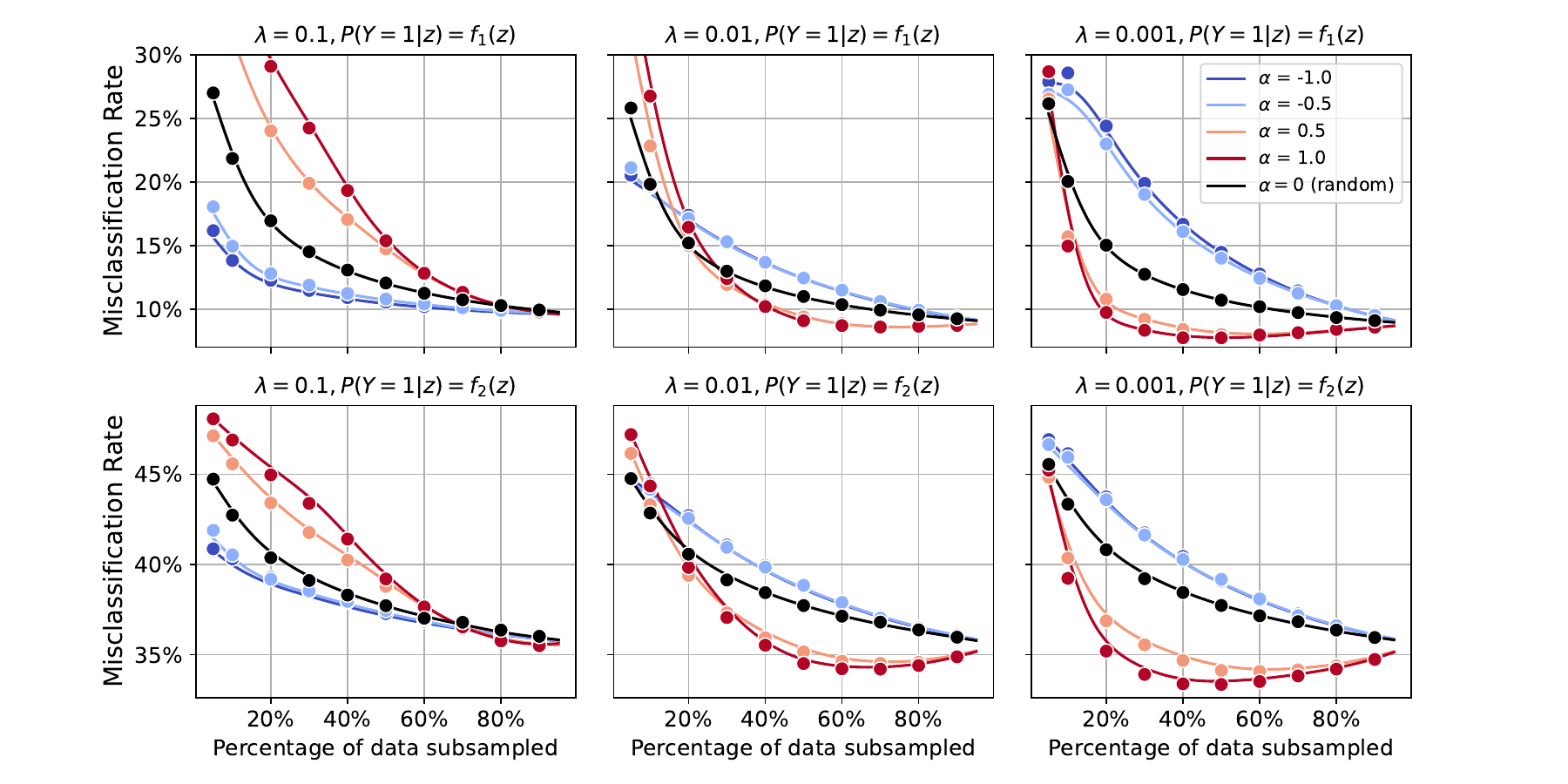}
\vspace{0.5cm} 
\includegraphics[width=0.9\linewidth]{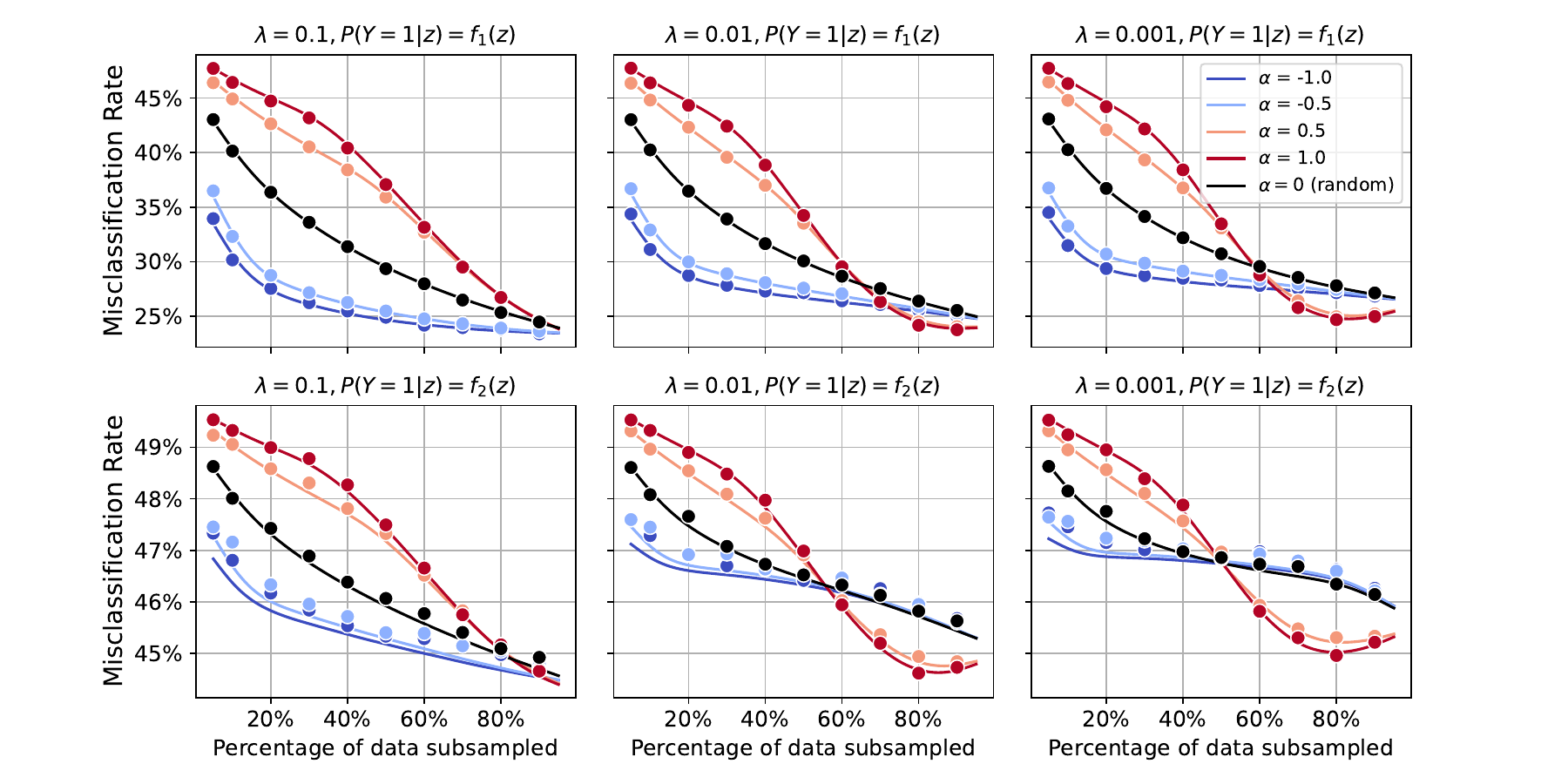}
\caption{Misclassification error for logistic regression  
after non-reweighted subsampling under the scheme of 
Eq.~\eqref{eq:Scheme}. Circles: simulations. Continuous lines: theory. Data are generated according to 
a misspecified model $\prob\big(\by=1\big|z\big) := f(z)$, where $z = \<\btheta_0, \bx\>$. Top plot: $N=34345$, $p=932$; bottom plot: $N=6870$, $p=3000$. Each panel corresponds to a different choice of $f$, $\lambda$. First and second row in each plot corresponds to $f=f_1$ and $f=f_2$ respectively where $f_1(z) = \eta\bfone_{z\geq 0 } + (1-\eta)\bfone_{z<0}$, $f_2(z) =  (1-\zeta) \bfone_{z < -0.5 } + (1-\eta)\bfone_{(-0.5 \leq z < 0)} + \eta \bfone_{(0 \leq z < 0.5)} + \zeta \bfone_{z \geq 0.5}$; $(\eta, \zeta) = (0.95, 0.7)$ and $\|\btheta_0\|=5$ for $f_2(z)$ ($f_1(z)$ only depends on the sign of $z$).}
\label{fig:Synthetic-Mis-Error-misspec}
\end{figure}

\begin{figure}
\centering
\includegraphics[width=0.9\linewidth]{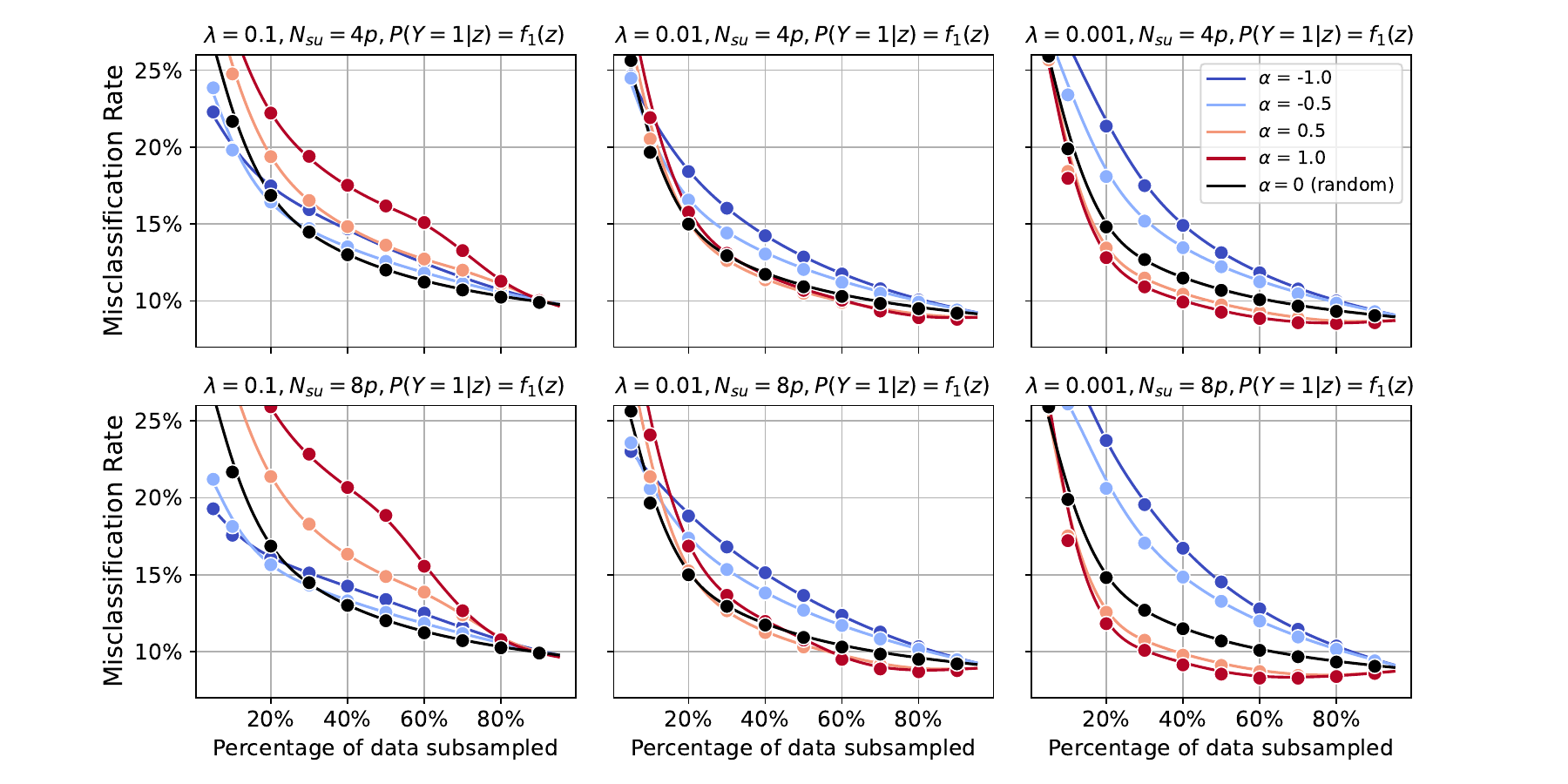}
\caption{Misclassification error for logistic regression  
after non-reweighted subsampling with data generated according to 
the  same misspecified model as in Figure \ref{fig:Synthetic-Mis-Error-misspec}.
Circles: simulations. Continuous lines: theory. Here $N=34345$, $p=932$.
Unlike in Figure \ref{fig:Synthetic-Mis-Error-misspec}, we use an imperfect
surrogate $\surrth$ that is fit on $N_{\surr}$ samples from the same distribution.
Top row: $N_{\surr}=4p$. Bottom row:  $N_{\surr}=8p$. The values of $\lambda$
indicated in the plot titles are used when learning on the selected subsample.}
\label{fig:Synthetic-Mis-Error-misspec-Imperf}
\end{figure}

Figures \ref{fig:Synthetic-Mis-Error-Lowdim} and \ref{fig:Synthetic-Mis-Error-Highdim} report results of simulations (circles) and theoretical 
predictions (continuous lines), respectively in a  lower-dimensional setting ($N=34345$, $p=932$) and in a higher-dimensional setting ($N=6870$, $p=3000$). Simulation results are median over $10$ realizations. Theoretical predictions are computed by evaluating the saddle point formula of Theorem \ref{thm:HiDim}. 

A few remarks are in order:
\begin{enumerate}
    \item The agreement between theory and simulations is excellent over the whole range of settings we investigated. 
    \item Theory correctly predicts that, for a suitable choice of $\alpha$, the non-reweighted  data selection
    scheme of Eq.~\eqref{eq:Scheme} achieves nearly identical test error as full data, while using as few as $40\%$ of the samples.
    \item The optimal choice of $\alpha$ is highly dependent on the setting, with $\alpha<0$ broadly outperforming $\alpha>0$
    when the number of samples per dimension is smaller
    (Figure \ref{fig:Synthetic-Mis-Error-Highdim}). 
    \item For small $\lambda$ and certain choices of $\delta_0$, $\alpha$ we observe interesting non-monotonicities of the error: smaller sample sizes lead to lower misclassification error. This is of course related to suboptimality of that value of $\alpha$ and of the scheme \eqref{eq:Scheme} (see Figure \ref{fig:Synthetic-Mis-Error-Lowdim}). Indeed, as we saw in Section
    \ref{sec:NonMono}, the optimal data selection scheme is necessarily monotone in $n$.
\end{enumerate}

In Figure \ref{fig:Synthetic-Mis-Error-misspec} we repeat the above 
experiments within a misspecified linear model, 
whereby $\prob(y_i=+1|\bx_i) = f(\<\btheta_0,\bx_i\>)$. We 
carry out experiments with two choices of $f$:
\begin{align*}
f_1(z) &:= \eta\bfone_{z\geq 0 } + (1-\eta)\bfone_{z<0}\, ,\\ 
f_2(z) &:=  (1-\zeta) \bfone_{z < -0.5 } + (1-\eta)\bfone_{(-0.5 \leq z < 0)} + \eta \bfone_{(0 \leq z < 0.5)} + \zeta \bfone_{z \geq 0.5}\, .
\end{align*}
with $\eta=0.95$, $\zeta = 0.7$. Note that $f_1(z)$ depends only on the sign of $z$. For $f_2(z)$ we present results with $\|\btheta_0\|=5$.

Finally, while all previous figures use a perfect surrogate, Figure 
\ref{fig:Synthetic-Mis-Error-misspec-Imperf} explores the impact of an 
imperfect surrogate. Namely, we estimate $\surrth$ by using $N_{\surr}$
independent samples from the same distribution as the training samples. 
We train $\surrth$ using ridge-regularized logistic regression, with an oracle choice 
of the regularization parameter $\lambda$. This setting
gives an intuitive understanding of the `cost' of the surrogate model. 
We choose $N_{\surr}= 4p$ (top row) or $N_{\surr}= 8p$
(bottom row), corresponding to $N_{\surr}/N\approx 10.9\%$
or $N_{\surr}/N\approx 21.7\%$, respectively.

The agreement between theoretical predictions and simulation results is again excellent. Also, we observe behaviors
that are qualitatively new with respect to the previous
setting that assumes  well-specified data and a perfect surrogate. 
Most notably:
\begin{enumerate}
\item Learning after data selection often outperforms learning on the full sample. 
\item Upsampling `hard' datapoints (i.e. using $\alpha>0$) is often the optimal strategy.
This appears to be more common than in the well-specified case.
\item  As shown in Figure \ref{fig:Synthetic-Mis-Error-misspec-Imperf},
the performance of data selection-based learning degrades gracefully with the quality of the surrogate. 
\item In particular, we observe once more the striking phenomenon of Figure \ref{fig:Summary},
cf. bottom row, rightmost plot of Figure \ref{fig:Synthetic-Mis-Error-misspec-Imperf}. At subsampling fraction $n/N=60\%$, learning on
selected data outperforms learning on the full data, even if the surrogate model only used
additional $N_{\surr}/N\approx 21.7\%$ samples. As shown in next section, this effect is even stronger with real data.
\end{enumerate} 

%
\section{Numerical results: Real data}

For our real-data experiments we used an Autonomous Vehicle (AV) dataset, and a binary classification task. 
As in the synthetic data simulations, we use
ridge regularized logistic regression. This model was trained 
and tested on unsupervised features extracted from image data. Below we first provide details of the dataset and experiments, followed by empirical results.

\label{sec:NumericalReal}
\subsection{Dataset}
We use a subset of images obtained from the KITTI-360 train set \cite{Liao2022PAMI}. The KITTI-360 train set comprises $1408\times376$ dimensional 8-bit stereo images, with 2D semantic and instance labels. These images are sourced from $9$ distinct continuous driving trajectories. We only consider the left stereo image for our dataset. To adapt this dataset for a binary classification task, we initially center-crop the images to dimensions of $224\times224$ pixels. Subsequently, we 
assign binary labels, by setting $y_i=+1$ if the count of pixels containing the semantic label in a certain class
surpasses a predefined threshold. We choose `car' as the label and the pixel cutoff threshold is set at $50$, resulting in a chance accuracy of approximately $0.69$. 

We then extract SwAV embeddings of the images to serve as 
feature vectors \cite{caron2021unsupervised}. We use \code{torch.hub.load(`facebookresearch/swav:main', `resnet50')} as the base model and we use the $2048$ dimensional outputs from the penultimate layer (head) as the SwAV features. 
Following common practice, we normalize the images before computing feature vectors, using mean and standard 
deviations of the images calculated on ImageNet. This results in a dataset with a total of $N=61,280$ images with $p=2048$ dimensional features with
binary labels indicating the presence or absence of a car.

\subsection{Experiments}

We randomly partition this dataset into four disjoint sets:
$N_{\strain}=34,345$ images to perform sub-sampling and train 
models, $N_{\ssurr}=14,720$ images to train surrogate models,
$N_{\sval}=3665$ images for validation and $N_{\stest}=8550$ images
for reporting the final experiment results. Prior to model
training, we center and normalize each of the features using mean
and standard deviation calculated from $N_{\strain}$ and
$N_{\ssurr}$.

We proceeded by training a ridge-regularized logistic regression
model without intercept. The training utilized the L-BFGS 
optimization algorithm with a cap of 10,000 iterations implemented 
using the scikit-learn library \cite{scikit-learn}. Surrogate 
models are trained on a fraction $k$ of the surrogate set ($N_{\ssurr}$) 
where $k \in \{10\%, 50\%, 100\%\}$ using 5-fold cross validation 
to choose the regularization parameter $\lambda$. Note that, the
sample sizes used for these 
surrogate models correspond to $\{4.2\%, 21.4\%, 42.8\%\}$ of 
$N_{\strain}$ (but we use a disjoint set of data). 

We use the data selection procedure introduced in
the previous section, cf. Eqs.~\eqref{eq:Scheme}, \eqref{eq:SchemeB}.
We show empirical results for $\alpha \in \{-2, -1, -0.5, 0, 0.5, 1, 2\}$, 
where $\alpha=0$ corresponds to random subsampling and positive 
(negative) values of $\alpha$ correspond to upsampling hard (easy) 
examples respectively. To ensure numerical stability for negative 
values of $\alpha$ modify the 
definition of $\pi$ in  Eqs.~\eqref{eq:Scheme}, by limiting\footnote{Formally, we replace 
$\<\surrth,\bx_i\>$
by $T(\<\surrth,\bx_i\>)$, where $T(x) = \min(\max(x,-10),10).$}  $\<\surrth,\bx_i\>$ to be in $[-10, 10]$. 

We also perform experiments in which we select the 
$n$ samples with the largest (or smallest) values of 
$p(\<\surrth,\bx_i\>)(1-p(\<\surrth,\bx_i\>))$.
This corresponds to the limit cases 
$\alpha\to\infty$ and $\alpha\to -\infty$ respectively. 
We will refer to these limit cases as `Hard topK' and 
`Easy topK', respectively. 

\subsection{Results}

All experiments are repeated across five random  
subsamplings when the data selection scheme is probabilistic 
and across three different surrogate models when the surrogate model is trained an a strict subset of the 
$N_{\ssurr}$ samples reserved for this.

We vary four 
different parameters in our experiments:
\begin{itemize}
    \item The ridge regularization parameter $\lambda$.
    This is either fixed or selected optimally by taking $\lambda^* = \argmin_{\lambda \in \Lambda} 
    R_{\mbox{\tiny\rm val}}(\hbtheta_\lambda)$, where
     $R_{\mbox{\tiny\rm val}}$ is the risk on the validation set and 
     $\Lambda :=$ $\{0.001, 0.01, 0.03,\\ 0.06, 0.1, 1, 10\}$.
    \item The exponent $\alpha$ that parameterizes the  subsampling probabilities. 
    \item The `strength' of the surrogate model, $\tsurr$, namely, the sample size $n_{\ssurr}\in [0,N_{\ssurr}]$
    used to learn $\surrth$. We will report the ratio
    $n_{\ssurr}/N_{\strain}$, as this provides a direct measurement of how much information is required to train the surrogate model. In particular, we will qualitatively refer to surrogate models in results and figures as `weak', `medium' and `strong' for the cases where: $n_{\ssurr}/N_{\strain} = 4.2\%$, $n_{\ssurr}/N_{\strain} = 21.4\%$ and $n_{\ssurr}/N_{\strain} = 42.8\%$ respectively.
    \item A binary parameter $\tbias$
    indicating whether we are using 
    unbiased or non-reweighted subsampling (referred to as biased sampling in Fig.~\ref{fig:Summary}). 
\end{itemize}

Figure~\ref{fig:Summary} shows the misclassification error on test set for optimal $\lambda$ and $\alpha=0.5$ fixed, for both unbiased and non-reweighted (biased) subsampling, using `weak' and `strong' surrogate models.

Figure \ref{fig:Summary_Opt} shows test error for the optimal parameters $\{\lambda, \alpha, \tsurr, \tbias\}$, as selected
by minimizing the misclassification rate on the validation set. The 
reported results are computed on the test set. We compare this
optimal choice by a `constant strategy' that uses $\lambda=0.01$, $\alpha=0.5$, non-reweighted subsampling and weak surrogate models. As evident from the figure, this constant strategy performs almost optimally when $n$ is large, and still provides consistent improvements over random subsampling when $n$ is low. 
As a reminder to the reader, this strategy reflects influence-function based subsampling, although without reweighting and using a weak surrogate model.

\begin{figure}  
\begin{center}
\includegraphics[width=0.48\linewidth]{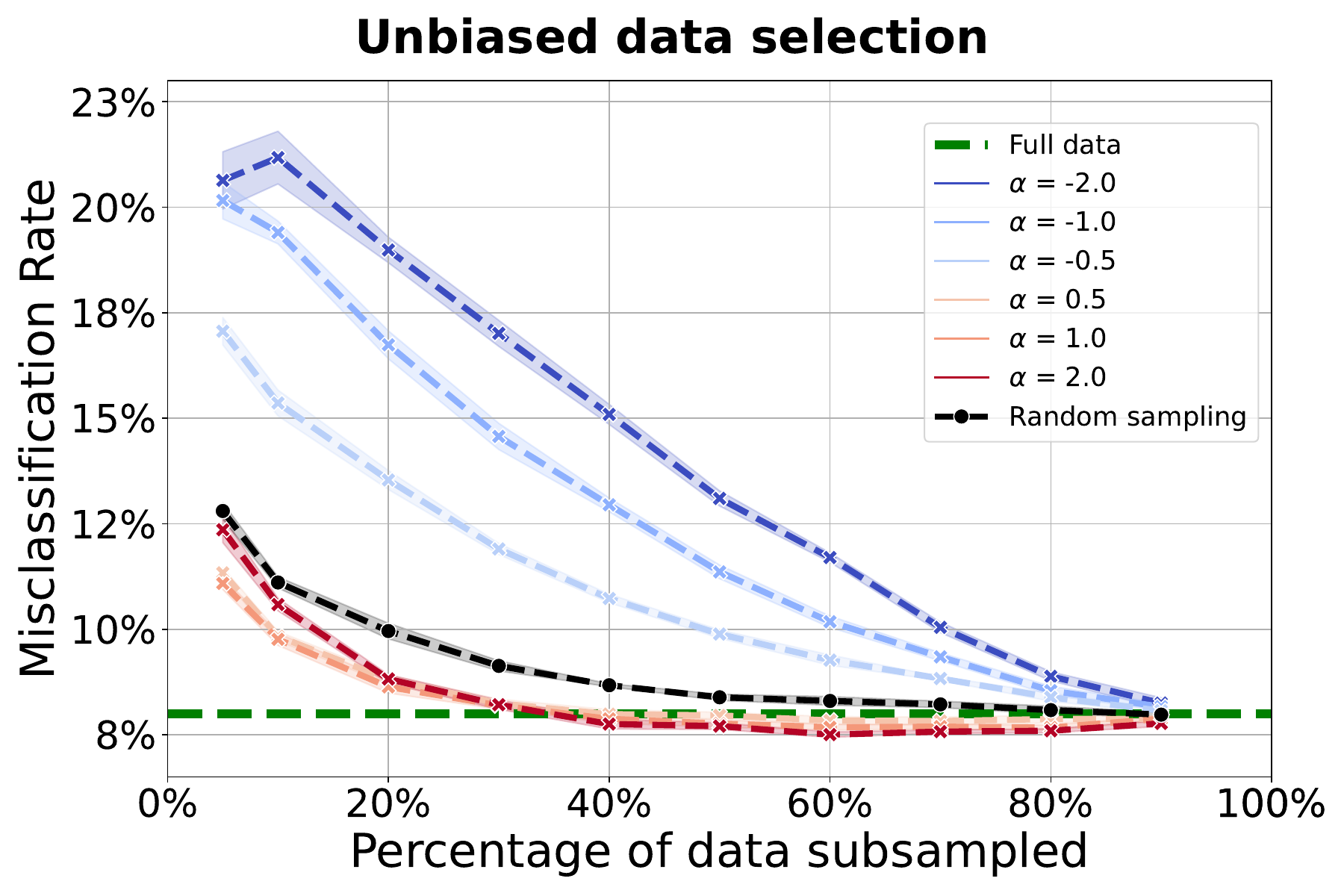}
\includegraphics[width=0.48\linewidth]{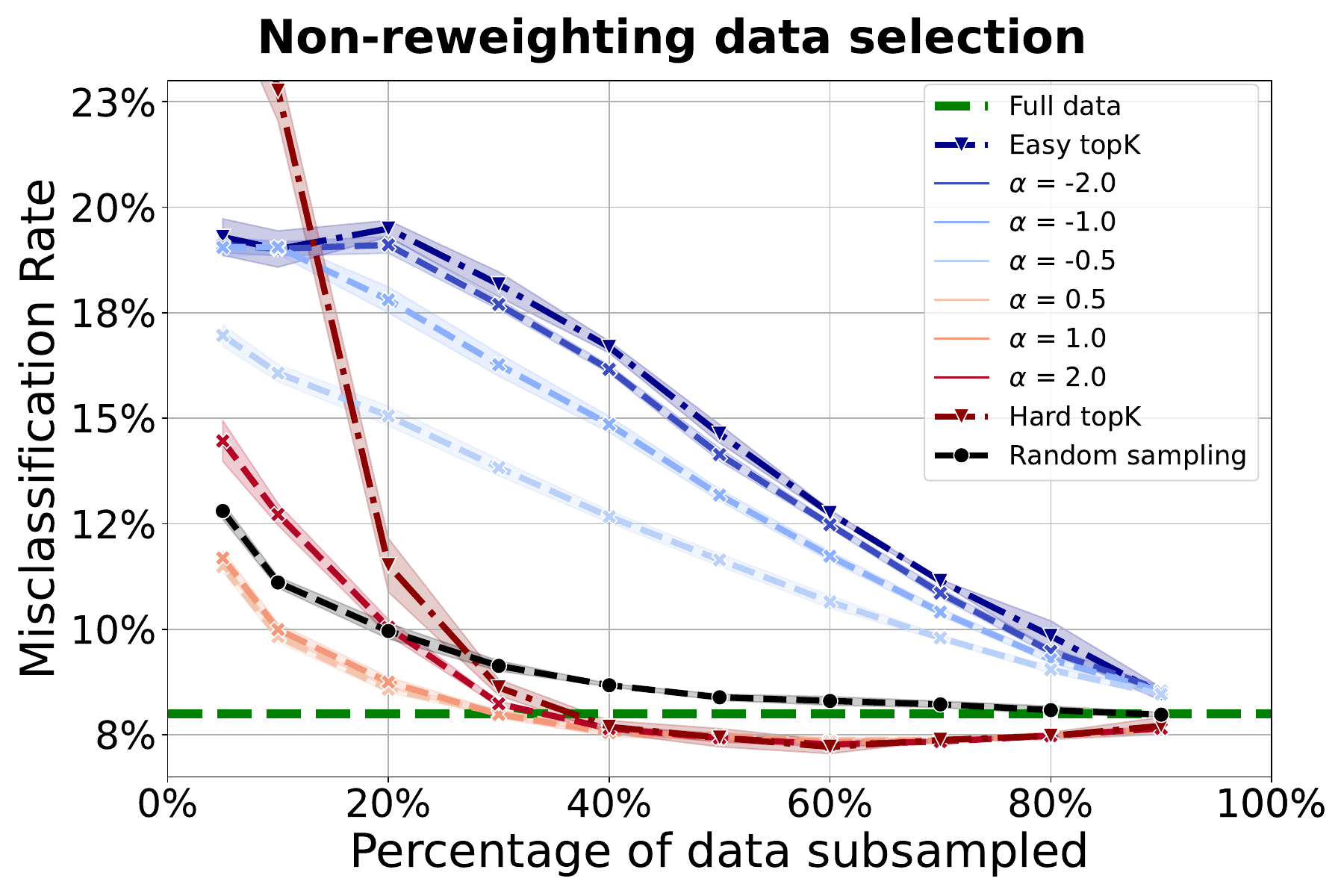}
\end{center}
\caption{Test error on image classification task, for a model trained after data subsampling. Effect of changing $\alpha$ in the subsampling probabilities, cf. Eq.~\eqref{eq:Scheme}. Here we use both
 unbiased (left) and non-reweighting (right) 
subsampling schemes with 
$n_{\ssurr}/N_{\strain} = 4.2\%$. $N=34345$, $p=2048$.
}
\label{fig:EffectAlpha}
\end{figure}

\begin{figure}  
\begin{center}
\includegraphics[width=0.48\linewidth]{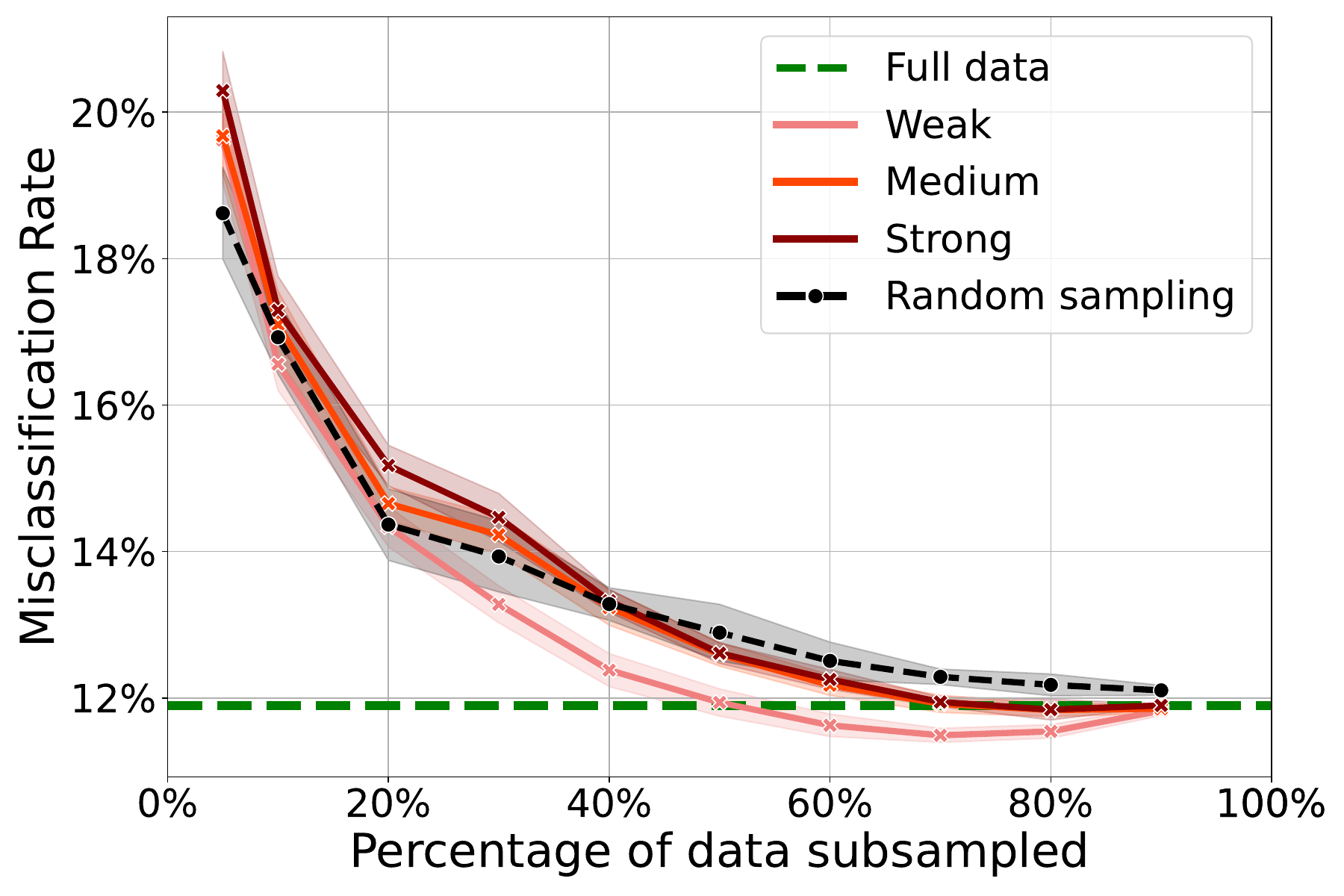}
\includegraphics[width=0.48\linewidth]{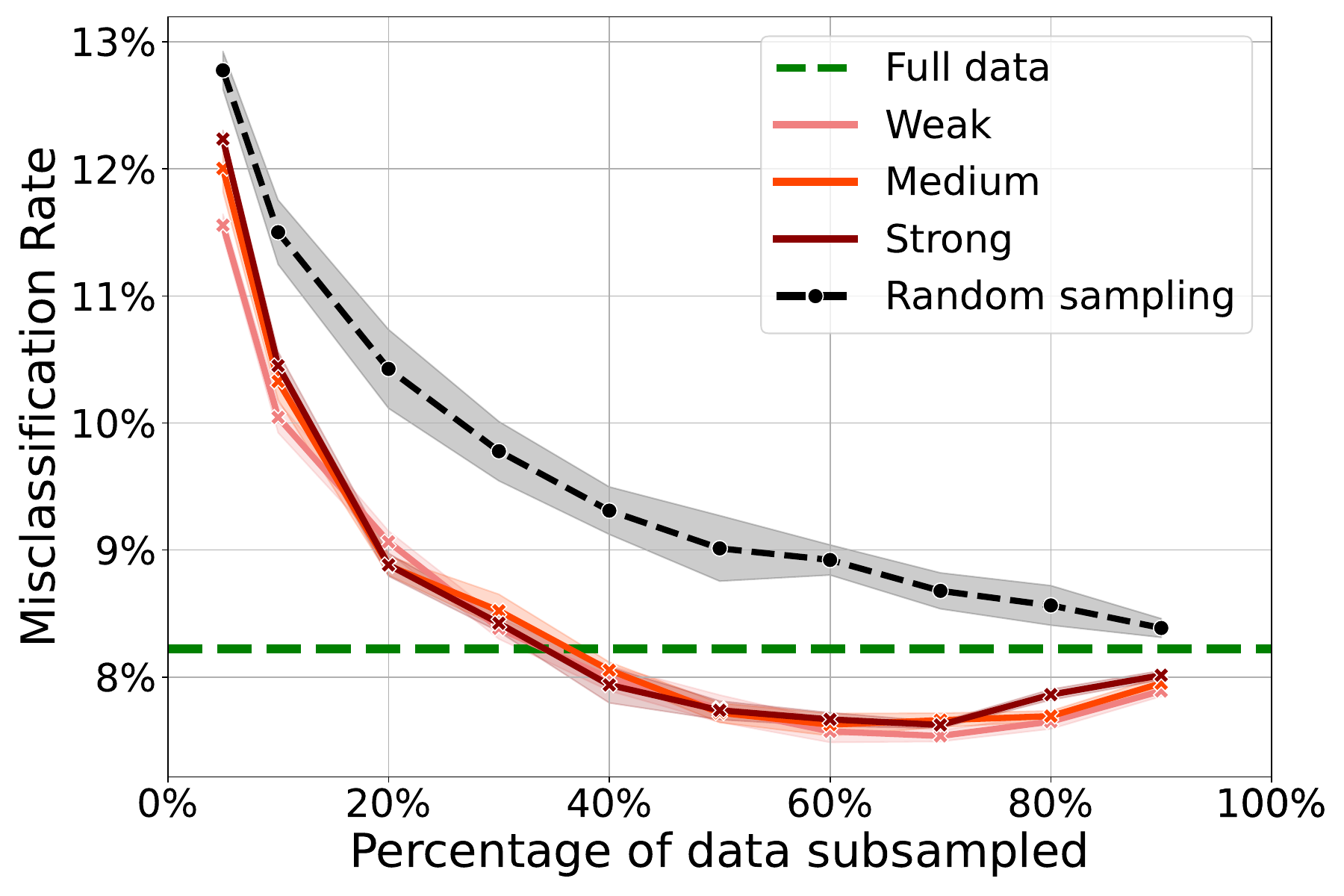}
\end{center}
\caption{
Test error on image classification task. 
Here $\lambda = 0.001$, $\alpha = 0.5$, 
and non-reweighting subsampling.
Left plot: $N=3434$, $p=2048$; right 
plot: $N=34345$, $p=2048$.
Various curves correspond to different
surrogate models.}
\label{fig:EffectSurr}
\end{figure}

Figure~\ref{fig:EffectAlpha} reports the misclassification rate on the test set as a function
of the subsampling fraction $\gamma$ for various values of the exponent $\alpha$. We consider both unbiased and non-reweighting subsampling
and use `weak' surrogate models.
In this case, we select $\lambda$  optimally, as 
described above. We observe that subsampling with $\alpha>0$ outperforms
training on the full sample down to subsampling ratios $\gamma\gtrsim 0.4$. This is most significant with non-reweighing subsampling, as anticipated by the asymptotic theory of Section \ref{sec:UnbiasedFirst}.
Further, above this value of $\gamma$,
the test error is fairly insensitive to the choice of $\alpha>0$.  
The situation changes dramatically at smaller
subsampling fractions. In particular, for non-reweighting 
subsampling and for $\gamma<0.3$,
soft subsampling $\alpha=0.5$ outperforms
substantially $\alpha=2$ and $\alpha=\infty$. 
Negative $\alpha$ (upsampling easy examples) 
always underperforms with respect to random in this case.

Figure \ref{fig:EffectSurr} investigates the effect of the strength of the surrogate model. In both subplots, we fix $\lambda = 0.001$, $\alpha = 0.5$, and  subsampling with no-reweighting. 
The two subplots correspond to different
regimes of the number of samples-to-parameters ratio.
The left subplot uses 10\% of the total available training set (i.e. $N=0.1\times N_{\strain}=3434$) as $100\%$ train data,
while the right subplot uses the entire training set (i.e. $N=N_{\strain}=34,345$). Each subplot shows subselection performance for three different surrogate models -- `weak', `medium' and `strong' described above.

We observe that at larger sample size $N$
(right subplot),
the test error of the model learnt after subsampling is insensitive to the accuracy of the surrogate model. We recover the results of Figure 
\ref{fig:EffectAlpha} irrespective
of the strength of the surrogate.
This is encouraging because it indicates that weak supervision is sufficient for effective data selection. Even more surprising is the behavior at
smaller sample size (left plot).
In this case the weak surrogate outperforms
medium and strong surrogates. A similar phenomenon was derived in a  minimax
setting in  Section \ref{sec:LowDimImperf}.

\section{Conclusion}

Data selection has been studied over the years 
from several point of views. Both heuristics and mathematically
motivated approaches have been put forward, sometimes with 
conflicting conclusions regarding their effectiveness. 

In this paper, we studied data selection within 
a statistical setting, using both low-dimensional and high-dimensional
asymptotics. Our main conclusions are as follows:
\begin{enumerate}
\item Restricting to `unbiased' data selection is unnecessary and sometimes harmful.
\item Data selection criteria based on the `uncertainty' associated to the label of a data point are effective. However both upsampling `hard'
and `easy' datapoints can be beneficial in different cases, and the form of the weights used does matter.
\item The efficacy of these method is not highly sensitive on the quality of the surrogate.  
\item Learning after data selection can outperform learning on the full sample.
\item Finally, we introduce a one-parameter family of data selection schemes, depending on parameter $\alpha\in\reals$, cf. Section \ref{sec:NumericalSynth}. By optimizing over $\alpha$, we obtain consistently good data selection across a varity of settings.
\end{enumerate}

\section*{Acknowledgements}

We are grateful to Joseph Gardi, Evan Gunter, Marc Laugharn, Kaleigh Mentzer, Rahul Ponnala, 
Eren Sasoglu and Eric Weiner for several conversations about this work. This work was carried out while Andrea Montanari was on leave from
Stanford and a Chief Scientist at Granica (formerly known as Project N). The present research is unrelated to
AM’s Stanford research.

\newpage

\bibliographystyle{alpha}
\newcommand{\etalchar}[1]{$^{#1}$}

\newpage

\addcontentsline{toc}{section}{References}

\newpage

\appendix

 %
\section{Proof of Proposition \ref{propo:LowDimAsymptotics}}
 \label{sec:ProofLowDim}

Write $S_i(\bx_i)= s(\bx_i,U_i)$ where $(U_i)_{i\le N}\sim_{iid}\Unif([0,1])$ are
the independent random seed used to compute
$S_i$ (we omit the dependence on the surrogate model). 
For $\|\bxi\|<1$, define 
$\btheta(\bxi)=c(\|\bxi\|)\, \bxi$, where $c(r) =1/(1-r^2)$,
and, with an abuse of notation, $R_S(\bxi) = R_S(\btheta(\bxi))$.
Finally
\begin{align}
L_S(\bxi;Z_i) &:= 
\begin{dcases}
    s(\bx_i,U_i)\, L(\btheta(\bxi);y_i,\bx_i)
&\mbox{ if $\|\bxi\|<1$,}\\
  s(\bx_i,U_i)\, L_{\infty}(\bxi;y_i,\bx_i)
 &\mbox{ if $\|\bxi\|=1$,}
\end{dcases}\\
Z_i&:= \big(U_i,y_i,\bx_i\big)\, ,
\end{align}
so that $R_S(\bxi) = \E L_S(\bxi;Z)$ and $\hR_N(\bxi) = N^{-1}\sum_{i=1}^NL_S(\bxi;Z_i)$ are defined for 
$\bxi\in\Ball^p(1)$ (the unit ball in $\reals^p$).
If $\hbxi:=\arg\min_{\bxi\in\Ball^p(1)}\hR_N(\bxi)$,
and $\bxi_*:= \arg\min_{\bxi\in\Ball^p(1)}R_S(\bxi)$
(the latter is unique by Assumption {\sf A1}),
by \cite[Theorem 5.14]{van2000asymptotic}, we have
$\hbxi\to \bxi_*$ almost surely.
By Assumption {\sf A2}, we further have $\|\bxi\|<1$
strictly. 
Therefore, almost surely, $\hbtheta\to \btheta_*=\btheta(\bxi_*)$.

%
%

Using \cite[Theorem 5.14]{van2000asymptotic} (whose assumptions follow from {\sf A3}, 
 {\sf A4}), we get
 \begin{align}
 \hbtheta-\btheta_* = -\frac{1}{N}\bH_S^{-1}\sum_{i=1}^N\nabla_{\btheta}L_S(\btheta_*;Z_i)+o_P(N^{-1/2})\,,
 \end{align}
 whence 
 \begin{align}
 \rho(S;\bQ) = \E\,\Tr\Big(\nabla_{\btheta}L_S(\btheta_*;Z_1)\nabla_{\btheta}L_S(\btheta_*;Z_1)^{\sT}
 \bH_S^{-1}\bQ\bH_S^{-1}\Big)\, .
 \end{align}
 The claim follows simply by substituting the expression for $\nabla_{\btheta}L_S(\btheta_*;Z_1)$.
 %
 %
 \section{Proof of Proposition \ref{propo:Unbiased}}
 \label{sec:ProofBasicUnbiased}

 As mentioned in the main text, this result (in slightly different form) appears already
 in the literature \cite{ting2018optimal,wang2018optimal,ai2021optimal}. We nevertheless present a 
 proof for the reader's convenience. 

 First of all notice that, for $S$ unbiased we have $\E\{S(\bx)|\bx\}=1$
 and therefore $\bH_S=\bH$. Eq.~\eqref{eq:GeneralAsymp} yields
 \begin{align*}
 \rho(S;\bQ)  =  \E\{S(\bx)^2Z(\bx)\}\, ,\;\;\;\; Z(\bx) = \Tr\big(\bG(\bx)\bH^{-1}
\bQ\bH^{-1}\big)\, ,
 \end{align*}
 We can always write $S(\bx) = S_+(\bx) \, I(\bx)$, where $S_+(\bx)>0$ almost surely  and, conditionally on $\bx$,  $S_+(\bx)$ is
 independent of $I(\bx)\in\{0,1\}$ with $\prob(I(\bx) = 1|\bx) = \pi(\bx)$.
 The unbiasedness constraint translates into $\E(S_+(\bx) |\bx) = 1/\pi(\bx)$. 
 Hence
 \begin{align*}
 \rho(S;\bQ)  &=  \E\{S_+(\bx)^2\pi(\bx)Z(\bx)\}\\
 & \ge \E\{\E\{S_+(\bx)|\bx\}^2\pi(\bx)Z(\bx)\} \\
 &= \E\Big\{\frac{Z(\bx)}{\pi(\bx)}\Big\} \, ,
 \end{align*}
 where the lower bound holds with equality if and only if $S_+(\bx) = 1/\pi(\bx)$ almost surely.

 The optimal $\pi$ is determined by the convex optimization problem
 \begin{align}
     \mbox{minimize} &\;\;\; \E\Big\{\frac{Z(\bx)}{\pi(\bx)}\Big\} \, ,\\
     \mbox{subj. to} &\;\;\; \E\{\pi(\bx)\} = \gamma\, ,\\
     &\;\;\; \pi(\bx)\in [0,1]\; \forall \bx\, .
 \end{align}
 By duality, there exists a constant $\lambda=\lambda(\gamma)$, such that the optimum of the above problem is
 the solution of
 \begin{align}
     \mbox{minimize} &\;\;\; \E\Big\{\frac{Z(\bx)}{\pi(\bx)}-\lambda\,\pi(\bx)\Big\} \, ,\\
     \mbox{subj. to} &\;\;\; \pi(\bx)\in [0,1]\;\;\; \forall \bx\, .
 \end{align}
This yields the claimed optimum $\pi_{\sunb}$. 
 %
 %
 \section{Proof of Lemma \ref{lemma:Simple}}
 \label{sec:ProofSimple}

  As in the previous section, we can always  
 $S(\bx) = S_+(\bx) \, I(\bx)$, where $S_+(\bx)>0$ almost surely and,
 conditionally on $\bx$, $S_+(\bx)$ is
 independent of $I(\bx)\in\{0,1\}$. Further
 \begin{align}
 \prob(I(\bx) = 1|\bx) &= \pi(\bx)\, ,\\
\E(S_+(\bx) |\bx) &= w(\bx)\, .
\end{align}
Simple schemes correspond to the case in which $S_+(\bx) = w(\bx)$ is non-random 

Recall the formula \eqref{eq:GeneralAsymp} for the asymptotic error coefficient 
 $\rho(S;\bQ)$, which 
 we  rewrite here as 
 \begin{align}
 \rho(S;\bQ) & =  \Tr\Big(\E\big\{S^2_+(\bx)\pi(\bx)\bG(\bx)\big\}
 \tbH_{w,\pi}^{-1}
\bQ\tbH_{w,\pi}^{-1}\Big)\, ,\\
\tbH_{w,\pi}&:=\E\{w(\bx)\pi(\bx)\bH(\bx)\}\, .
\end{align}
By Jensen inequality (using the fact that $\bG(\bx), \tbH_{w,\pi}, \bQ\succeq\bfzero$),
we get 
 \begin{align}
\rho(S;\bQ) & \ge   \Tr\Big(\E\big\{w(\bx)^2\pi(\bx)\bG(\bx)\big\}
 \tbH_{w,\pi}^{-1}
\bQ\tbH_{w,\pi}^{-1}\Big)\, ,
\end{align}
and simply note that the right hand side is achieved by the simple scheme.
 %
 %
 \section{Proof of Proposition \ref{propo:Biased}}
 \label{sec:ProofBiased}

Recall the general formula \eqref{eq:GeneralAsymp} for the asymptotic error coefficient 
 $\rho(S;\bQ)$. For a non-reweighting scheme with selection probability $\pi$, with an abuse of notation
 we write $\rho(S;\bQ)$ as  $\rho(\pi;\bQ)$ (we also used $\rho(\pi,1;\bQ)$ in the main text). 
 Explicitly
 \begin{align}
 \rho(\pi;\bQ) & =  \Tr\Big(\E\big\{\pi(\bx)\bG(\bx)\big\}
 \E\big\{\pi(\bx)\bH(\bx)\big\}^{-1}
\bQ\E\big\{\pi(\bx)\bH(\bx)\big\}^{-1}\Big)\, .\label{eq:NRrho}
\end{align}
Notice that this is defined for $\E\big\{\pi(\bx)\bH(\bx)\big\}\succ \bfzero$.
We extent it to $\E\big\{\pi(\bx)\bH(\bx)\big\}\succeq \bfzero$
by letting
 \begin{align}
 \rho(\pi;\bQ) & = \lim_{\lambda\to 0+} \Tr\Big(\E\big\{\pi(\bx)\bG(\bx)\big\}
 \big(\lambda\id+\E\big\{\pi(\bx)\bH(\bx)\big\}\big)^{-1}
\bQ\big(\lambda\id+\E\big\{\pi(\bx)\bH(\bx)\big\}\big)^{-1}\Big)\, .\label{eq:NRrhoinf}
\end{align}
We want to minimize this function over $\pi$ subject to the convex constraints $\E\pi(\bx)=\gamma$,
$\pi(\bx)\in[0,1]$ for all $\bx$. 

We claim that a minimizer $\pi_{\snr}$ always exists. To this end, we view $\rho(\pi;\bQ)  = F(\nu)$ 
as a function of the probability measure $\nu(\de\bx) = \pi(\bx)\prob(\de \bx)/\gamma$. In other words $F$
is a function of the space of probability measures, whose Radon-Nikodym derivative with respect to 
$\prob$ is upper bounded by $1/\gamma$.
This domain is uniformly tight. Further, if $\nu_n$ is a sequence in this space and $\nu_n\Rightarrow\nu_{\infty}$
(weak convergence), it follows by the Portmanteau's theorem that $\nu_{\infty}$
has also Radon-Nikodym derivative with respect to 
$\prob$ that is upper bounded by $1/\gamma$.
Hence this domain is compact by Prokhorov's
theorem.
Finally $\nu \mapsto \int \bG(\bx)\, \nu(\de\bx)$ and 
$\nu \mapsto \int \bH(\bx)\, \nu(\de\bx)$ are continuous
in the topology of weak convergence  (because $\bG(\bx)$, $\bH(\bx)$
are continuous by assumption), and therefore
$\nu\mapsto F(\nu)$ is lower
semi-continuous. Hence, there exists a minimizer 
$\nu_{\snr}(\de\bx) = \pi_{\snr}(\bx)\prob(\de \bx)/\gamma$, with 
$\pi_{\snr}(\bx)\in[0,1]$.

Given any minimizer $\pi_{\snr}$, and any other feasible $\pi$,
let $\pi_t :=(1-t)\pi_{\snr}+t\pi$. By assumption 
$\bH_{\pi_{\snr}}\succ \bfzero$ strictly.
Then 
 \begin{align}
 \rho(\pi_t;\bQ) = \rho(\pi_{\snr};\bQ)+t \int \big(\pi(\bx)-\pi_{\snr}(\bx)\big) Z(\bx;\pi_{\snr})\,\prob(\de\bx)+o(t)\,.
 \end{align}
 Therefore it must be true that, for any feasible $\pi$, 
 \begin{align}
 J(\pi;\pi_{\snr}):=\int \big(\pi(\bx)-\pi_{\snr}(\bx)\big) Z(\bx;\pi_{\snr}) \,\prob(\de\bx) \ge 0\, .
 \label{eq:Stationarity}
 \end{align}
Let $Q_{\eps}:= \{\bx\in\reals^d:\, \pi_{\snr}(\bx)\in (\eps,1-\eps)\}$. The claim \eqref{eq:piNR}
is is implied by the following statement: $Z(\bx;\pi_{\snr})$ is almost surely
constant on $Q_{\eps}$ for each $\eps>0$. Assume by contradiction that there exists $\eps>0$
and $z_0\in\reals$
such that $\prob(\bx\in Q_{\eps}:\, Z(\bx;\pi_{\snr})\ge z_0)=p_+>0$, $\prob(\bx\in Q_{\eps}: Z(\bx;\pi_{\snr})<z_0)=p_->0$.
Let $Q_{+}:=\{\bx\in Q_{\eps}:\, Z(\bx;\pi_{\snr})\ge z_0\}$,  $Q_{-}:=\{\bx\in Q_{\eps}:\, Z(\bx;\pi_{\snr})< z_0\}$.
Define
\begin{align*}
\pi(\bx) = \begin{cases}
\pi_{\snr}(\bx)-p_-\eps & \mbox{ if $\bx\in Q_+$,}\\
\pi_{\snr}(\bx)+p_+\eps & \mbox{ if $\bx\in Q_-$,}\\
\pi_{\snr}(\bx) & \mbox{ otherwise.}
    \end{cases}
\end{align*}
It is easy to check that $\pi$ is feasible and
 \begin{align*}
 J(\pi;\pi_{\snr})&=-p_-\eps \int_{Q_+} Z(\bx;\pi_{\snr}) \,\prob(\de\bx)
 +p_+\eps \int_{Q_-} Z(\bx;\pi_{\snr}) \,\prob(\de\bx)\\
 &= -p_+p_-\eps\Big\{\E\big(Z(\bx;\pi_{\snr})\big|\bx\in Q_+\big)-\E\big(Z(\bx;\pi_{\snr})\big|\bx\in Q_-\big)
 \Big\}<0\, ,
\end{align*}
thus yielding a contradiction with Eq.~\eqref{eq:Stationarity}.

Finally, we note that Eq.~\eqref{eq:VariationalOptBiased} follows from Eq.~\eqref{eq:NRrho}.
 %
 %
 \section{Proof of Theorem \ref{thm:GeneralDerivative} and Theorem \ref{thm:NonMono}} 
 \label{sec:ProofNonMono}

\subsection{Proof of Theorem \ref{thm:GeneralDerivative}}

Write $\tbG_{\pi} = \E\{\bG(\bx)\pi(\bx)\}$, and similarly for $\tbH_{\pi}$.
Defining $\opi(\bx):=1-\pi(\bx)$, and using $\E\{\opi(\bx)\}=1-\gamma$, we get 
\begin{align}
\tbG_{\pi}  = \bG- \E\{\bG(\bx)\opi(\bx)\}\, ,\;\;\;\; \|\E\{\bG(\bx)\opi(\bx)\}\|_{\op}\le 
\E\{\|\bG(\bx)\|_{\op}^{4}\}^{1/4}(1-\gamma)^{3/4}\, ,
\end{align}
and similarly for $\tbH_{\pi}$.

Using Eq.~\eqref{eq:VariationalOptBiased} and Taylor expansion
(recall $\bH\succ\bfzero$ by assumption), we get, for any $\pi$
satisfying $\E\pi(\bx)=\gamma$, 
\begin{align*}
\rho_{\snr}(\bQ;\gamma) &\le  \rho_{\snr}(\bQ;1) +\E\big\{Z_{\bQ}(\bx;1)\opi(\bx)\big\} + 
O\big(\|\E\{\bG(\bx)\opi(\bx)\}\|_{\op}^2+\|\E\{\bH(\bx)\opi(\bx)\}\|^2_{\op}\big)\\
 &=\rho_{\snr}(\bQ;1) +\E\big\{Z_{\bQ}(\bx;1)\opi(\bx)\big\} + 
O\big((1-\gamma)^{3/2}\big)\, ,
\end{align*}
where $Z_{\bQ}(\bx;1)$ is defined in the statement of the theorem.

For $\lambda$ as in the statement, and all $(1-\gamma)$ small enough, we
can take $\opi(\bx) = (1-\gamma)/\prob(Z_{\bQ}(\bx;1)<\lambda)$ if $Z_{\bQ}(\bx;1)<\lambda$, and  $\opi(\bx)=0$ otherwise. This immediately implies 
Eq.~\eqref{eq:MainGenDeriv} whence point $(b)$ follows.

In order to prove point $(a)$, note that, by definition, for any
$\lambda>\ess\inf Z_{\bQ}(\bx)$, we have $\prob(Z_{\bQ}(\bx)<\lambda)>0$
and therefore Eq.~\eqref{eq:MainGenDeriv} implies, for all $1-\gamma$ small enough
\begin{align}
\rho_{\snr}(\bQ;\gamma) &\le \rho_{\snr}(\bQ;1)+\lambda(1-\gamma)
+C(1-\gamma)^{3/2}\, .
\end{align}
This implies $\partial_{\gamma}\rho_{\snr}(\bQ;1)\ge -\lambda$
and therefore $\partial_{\gamma}\rho_{\snr}(\bQ;1)\ge -\ess\inf Z_{\bQ}(\bx)$.

Alternatively, assume $\ess\inf Z_{\bQ}(\bx) = \lambda_*>-\infty$.
Let $\pi_{\snr, \gamma}$ achieve the infimum in Eq.~\eqref{eq:VariationalOptBiased}
(Proposition \ref{propo:Biased} ensures that such $\pi_{\snr, \gamma}$ exists).
By the above argument (letting $\opi_{\snr,\gamma}(\bx)= 1-\pi_{\snr,\gamma}(\bx)$)
\begin{align*}
\rho_{\snr}(\bQ;\gamma) &=  \rho_{\snr}(\bQ;1) +\E\big\{Z_{\bQ}(\bx;1)\opi_{\snr,\gamma}(\bx)\big\} + 
O\big(\|\E\{\bG(\bx)\opi(\bx)\}\|_{\op}^2+\|\E\{\bH(\bx)\opi(\bx)\}\|^2_{\op}\big)\\
& \ge \rho_{\snr}(\bQ;1) +\lambda_*(1-\gamma) + 
O\big((1-\gamma)^{3/2}\big)\, .
\end{align*}
This yields $\partial_{\gamma}\rho_{\snr}(\bQ;1)\le -\lambda_*$.

\subsection{Proof of Theorem \texorpdfstring{\ref{thm:NonMono}(a)}{}}

Both  claims of the theorem follow if we can provide an example such that 
$Z_{\bQ}(\bx;1)<0$ with
strictly positive probability, in the case $\bQ=\bH$. Specializing to that case,
we have
\begin{align}
Z_{\bH}(\bx;1) := -\Tr\big\{\bG(\bx)\bH^{-1}\big\}+
2 \Tr\big\{\bH(\bx)\bH^{-1}\bG\bH^{-1}\big\}\, .\label{eq:ZHX}
\end{align}
In the following we will simplify notations and write $Z(\bx)=Z_{\bH}(\bx;1)$.

We next consider the linear regression setting of point $(a)$.
Without loss of generality, we will assume $\|\btheta_0\|=1$.
By rotation invariance, the population risk minimizer has the form
$\btheta_*=\alpha_*\btheta_0$. The coefficient $\alpha_*$ is fixed by
\begin{align*}
\bfzero= \nabla R_S(\alpha_*\btheta_0) = \E\big\{(y-\alpha_*\<\btheta_0,\bx\>)\bx\big\}\, .
\end{align*}
The only non-zero component of this equation is the one along $\btheta_0$.
Projecting along this direction, we get (for $G\sim\normal(0,1)$, $Y\sim\sP(\,\cdot\,|G)$):
\begin{align*}
\E\big\{(Y-\alpha_*G)G\big\}=0\, .
\end{align*}
We next compute 
\begin{align}
\bG(\bx) = \E\big\{(y-\<\btheta_*,\bx\>)^2\big|\bx\big\}\, \bx\bx^{\sT}\, ,\;\;\;
\bH(\bx) = \bx\bx^{\sT}\, . 
\end{align}
Note that we can rewrite the first equation as 
\begin{align}
\bG(\bx) &= f(\<\btheta_0,\bx\>)\, \bx\bx^{\sT}\, ,\\
f(t)& := \int(y-\alpha_*t)^2\, \sP(\de y|t)\, . 
\end{align}
Taking expectation with respect to $\bx$
\begin{align}
\bG&= a\id_d+b\btheta_0\btheta_0^{\sT}\, ,\;\;\;
\bH = \id_d\, ,\\
a& := \E\big\{(Y-\alpha_*G)^2\big\}\, ,\;\;\;\;\;b:= \E\big\{(Y-\alpha_*G)^2
(G^2-1)\big\}\, .
\end{align}
Substituting in Eq.~\eqref{eq:ZHX}, we get
\begin{align}
Z(\bx) &= -f(\<\btheta_0,\bx\>)\|\bx\|^2+2a\|\bx\|^2+2b\<\btheta_0,\bx\>^2
\nonumber\\
& = \<\btheta_0,\bx\>^2\big[-f(\<\btheta_0,\bx\>) +2\,\E\{ 
f(G) G^2\}\big]\label{eq:ZxRegression}\\
&\phantom{AAA}+\big\|\bP_0^{\perp}\bx\big\|_2^2\big[-f(\<\btheta_0,\bx\>) +
2\,\E \{f(G) \}\big]\, .\nonumber
\end{align}
It is easy to construct examples in which $\ess\inf Z(\bx)=-\infty$.
For instance, take $\sP(\,\cdot\,|t) = \delta_{h(t)}$, $h(t)= t + c(t^3-3t)$
for some $c>0$ (no noise). Then we get $\alpha_*=1$ and $f(t)=c^2(t^3-2t)^2$.
Let $t_0$ be such that $f(t)>2\E\{G^2 f(G)\}$ for all $t\ge t_0$.
Then the claim follows since 
\begin{align}
\prob\Big(\bx:\; \<\btheta_0,\bx\>\in [t_0,t_0+1], \big\|\bP_0^{\perp}\bx\big\|_2 >M
\Big)>0\, ,
\end{align}
for any $M$, and Eq.~\eqref{eq:ZxRegression} yields $Z(\bx)<0$
on the above event for all $M$ large enough. In fact, we also have 
$Z(\bx)<-c$ for any $c$, by taking $M$ sufficiently large. 

\subsection{Proof of Theorem \texorpdfstring{\ref{thm:NonMono}(b)}{}}

We next consider a  misspecified generalized linear model
with $y_i\in\{+1,-1\}$ and
\begin{align}
    \sP(+1|z)= 1- \sP(-1|z) = f(z)\, .
\end{align}
We set $u(t):= 2f(t)-1$.
It is simple to compute
 \begin{align*}
 \nabla R(\btheta) & = -\E\big\{(u(\<\btheta_0,\bx\>)-\phi'(\<\btheta,\bx\>))\bx\big\}\, .
 \end{align*}
In particular  $\nabla R(\btheta)  = -\E\big\{(u(G)-\phi'(G))G\big\}\btheta_0$,
where expectation is with respect to $G\sim\normal(0,1)$.   We will impose the condition
 \begin{align*}
 \E\big\{Gu(G)\big\}= \E\big\{G\phi'(G))\big\}\, .\label{eq:ThetaStarCond}
 \end{align*}
so that the empirical risk minimizer is $\btheta_*=\btheta_0$.

Next we compute
\begin{align}
\bG(\bx)  = g(\<\btheta_*,\bx\>)\,\bx\bx^{\sT}\, ,\;\;\;\;\;\;
\bH(\bx)  = h(\<\btheta_*,\bx\>)\,\bx\bx^{\sT}\, ,
\end{align}
where
\begin{align}
g(t) := (u(t)-\phi'(t))^2 + 1-u(t)^2\, ,\;\;\;\;
h(t) := 1-\phi'(t)^2\, .
\end{align}
Performing the Gaussian integral, we get
\begin{align}
    \bG =  a_G\id_d+ b_G\btheta_*\btheta_*^{\sT}\, , \,\;\;\;\;\;
    \bH =  a_H\id_d+ b_H\btheta_*\btheta_*^{\sT}\, ,
\end{align}
where $a_G:=\E g(G)$, $b_G:=\E g''(G)$, and similarly for $\bH$.
We thus get 
\begin{align}
    \bH^{-1}  = c_1\id_d+ c_1'\btheta_*\btheta_*^{\sT}\, , \,\;\;\;\;\;
    \bH^{-1}\bG\bH^{-1} = c_2\id_d+ c_2'\btheta_*\btheta_*^{\sT}\, ,
\end{align}
for some constants $c_i,c_i'$ that are dimension-independent functions of $a_H, b_H, a_G, b_G$.
Substituting in the formula for $Z(\bx)=Z_{\bH}(\bx;1)$,
cf. Eq.~\eqref{eq:ZHX}, we get
\begin{align}
Z(\bx) := -g(\<\btheta_*,\bx\>)\, \<\bx,(c_1\id_d+ c_1'\btheta_*\btheta_*^{\sT})\bx\>
+2 h(\<\btheta_*,\bx\>)\, \<\bx,(c_2\id_d+ c_2'\btheta_*\btheta_*^{\sT})\bx\>\, .
\end{align}
For large $d$, we have
\begin{align}
\<\bx,(c\id_d+ c'\btheta_*\btheta_*^{\sT})\bx\> = c\, d +o_P(1)\, ,
\end{align}
and therefore 
\begin{align}
\frac{1}{d} Z(\bx) = -c_1 g(\<\btheta_*,\bx\>) +c_2  h(\<\btheta_*,\bx\>) +o_P(d)\, .
\end{align}
Note that $G=\<\btheta_*,\bx\>\sim\normal(0,1)$. In particular, its distribution
is $d$-independent.
It is therefore sufficient to prove that  $-c_1 g(G) +2c_2  h(G)  <0$  with strictly positive probability, since this implies $Z(\bx)$ with strictly positive
probability for all $d$ large enoug.
It is simple to compute $c_1= 1/a_H$, $c_2= a_G/a^2_H$.
Therefore it is sufficient to prove the following
\begin{itemize}
\item[] {\bf Claim:} We can choose $f$ so that Eq.~\eqref{eq:ThetaStarCond} is satisfied and 
$-a_H g(G) + 2a_G h(G)<0$ with strictly positive probability.
\end{itemize}

To prove this claim, it is convenient to define the random variables $W=\phi'(G)$,
$M=u(G)$ and note that the distribution of $W$ is symmetric around $0$, and has support $(-1,1)$.

We have $g(G) = (M-W)^2+1-M^2$, $h(G) = 1-W^2$.
Therefore, the conditions of the claim are equivalent to
(the second condition is understood to hold with strictly positive probability)
\begin{align}
    \E\big\{(M-W)\, \psi(W)\big\}& =0\, ,\\
    a_H\cdot \big[(M-W)^2+1-M^2\big] &> 2a_G\cdot\big[1-W^2\big]\, ,
\end{align}
where $\psi=(\phi')^{-1}$ (functional inverse), and
\begin{align}
    a_G & = \E\big\{(M-W)^2+1-M^2\big\}\, ,\;\;\; a_H =  \E\big\{1-W^2\big\}\, .
\end{align}

We will further restrict ourselves to construct these random variables so that
$2a_G=a_H$, i.e.
\begin{align}
     \E\{1-W^2-4W(M-W)\}=0\, .
\end{align}
Next define $\varphi:\reals\to\reals$ by $\varphi(x) = u(x)-\phi'(x)$,
whence $M-W=\varphi(G)$, $W=\phi'(G)= \tanh(G)$. 
Therefore, it is sufficient to construct $\varphi:\reals\to\reals$ such that,
for some $x_0\in \reals$, $\eps>0$,
\begin{align}
\E\{\varphi(G)\tanh(G) \} &= b_0\, ,\\
\E\{\varphi(G)G \} &= 0\, ,\\
\varphi(x)^2+1-(\tanh(x)+\varphi(x))^2 & >  1-\tanh(x)^2 \;\;\;\mbox{ for all }x\in(x_0-\eps,x_0+\eps)\, ,
\label{eq:SomeX0}\\
-1-\tanh(x)<\varphi(x)&<1-\tanh(x)\;\;\;\;\;\;\;\mbox{ for all }x\in \reals\, .
\end{align}
where $b_0= \E\{1-\tanh(G)^2\}/4$ is a constant $b_0\in(0,1/4)$.
Note that Eq.~\eqref{eq:SomeX0} is sufficient because the law of $G$ is supported on the whole
real line. Simplifying Eq.~\eqref{eq:SomeX0}, and assuming $\varphi$ is continuous, we are led to
\begin{align}
\E\{\varphi(G)\tanh(G) \} &= b_0\, ,\label{eq:CondPhi1}\\
\E\{\varphi(G)G \} &= 0\, ,\label{eq:CondPhi2}\\
x\varphi(x)& <0    \;\;\;\;\;\;\;\;\;\;\;\;\;\;\;\;\;\;\;\;\;\mbox{ for some }x\in\reals\, ,
\label{eq:CondPhi3}\\
-1-\tanh(x)<\varphi(x)&<1-\tanh(x)\;\;\;\;\;\mbox{ for all }x\in \reals\, .\label{eq:CondPhi4}
\end{align}
We complete the proof by the following result.
\begin{lemma}
    There exists $\varphi:\reals\to\reals$ continuous satisfying conditions \eqref{eq:CondPhi1} to \eqref{eq:CondPhi4} above.
\end{lemma}
\begin{proof}
We define $\sF: L^2(\normal(0,1))\to\reals^2$ by $\sF(\varphi) =(\E\{\varphi(G)\tanh(G) \},\E\{\varphi(G)G \})$, and 
\begin{align}
\cC:=\Big\{\varphi\in C(\reals):\;  \forall x\in\reals\;\;
\max(-1-\tanh(x),-1) <\varphi(x)<\min(1-\tanh(x),1)\Big\}\, .
\end{align}
In particular, any $\varphi\in \cC$ satisfies condition \eqref{eq:CondPhi4}. 
We
claim that there exists an open set $B\subseteq \reals^2$, such that $(b_0,0)\in B$ 
and $\sF(\cC)\supseteq B$ (i.e., for every $\bx\in B$, there exists $\varphi\in\cC$
such that $\sF(\varphi)=\bx$).

In order to prove this claim, note that $\sF$ is a continuous linear map and $\cC$ is convex,
whence $\sF(\cC)$ is convex.
Hence, it is sufficient to exhibits points $\varphi_1,\varphi_2,\varphi_3\in\overline{\cC}$
(the closure is in  $L^2(\normal(0,1))$ of $\cC$), such that $(b_0,0)$ is in the interior of the convex hull 
of $\{\sF(\varphi_i):\; i\le 3\}$. We use the following functions:
\begin{enumerate}
\item $\varphi_1(x) = 0$: $\sF(\varphi_1) = (0,0)$.
\item $\varphi_2(x) = \sign(x)(1-\tanh(|x|))$: 
\begin{align}
\sF_1(\varphi_2) &= \E\{\tanh|G|(1-\tanh |G|)\}=:b_1>b_0\, ,\\
\sF_2(\varphi_2) &= \E\{|G|(1-\tanh|G|)\}>0\, ,
\end{align}
(Numerically, $b_1\approx 0.16168$, $b_0\approx 0.15143$.)
\item $\varphi_3(x) = \varphi_2(x) - M^{-1/2}\, q(x-M)$, where $q(x)$ is a continuous non-negative
function supported on $[-1,1]$, with $\int q(x) \,\de x>0$. We have $\sF_1(\varphi_3) = \sF_1(\varphi_2)-\Theta(M^{-1/2})$
and $\sF_2(\varphi_3) = \sF_2(\varphi_2)-\Theta(M^{1/2})$.
Hence for a sufficiently large $M$, $\sF_1(\varphi_3)>b_0$,
$\sF_2(\varphi_3)<0$.
\end{enumerate}

This proves the claim, and in particular we can satisfy conditions \eqref{eq:CondPhi1},
\eqref{eq:CondPhi2}, \eqref{eq:CondPhi4}, since $(b_0,0)\in B$. In order to show that we can satisfy also condition \eqref{eq:CondPhi3}, let $q$ be defined as above and further such that $q(x)\le 1/2$
for all $x$.
Consider the sequence
of functions indexed by $k\in\naturals$:
\begin{align}
\varphi_k(x) = \ophi_k(x) - q(x-k)\, .
\end{align}
First note that, if $\ophi_k\in\cC$, then $\varphi_k$
satisfies conditions \eqref{eq:CondPhi3}, \eqref{eq:CondPhi4}.
Therefore, we are left to prove that, for some $k$,
there exists $\ophi_k\in\cC$ that satisfies 
\begin{align}
\sF_1(\ophi_k) &= b_0+\E\{q(G-k)\tanh(G) \} =: b_0+\delta_{1,k}\, ,\\
\sF_2(\ophi_k) &= \E\{q(G-k)\, G \} =: \delta_{2,k}\, .
\end{align}
By dominated convergence, we have $\bdelta_k=(\delta_{1,k},\delta_{2,k})\to 0$ as $k\to\infty$,
and therefore $(b_0,0)+\bdelta_k\in B$ for all $k$  sufficiently large, whence the existence of $\ophi_k$
for such $k$ follows from  the first part of the proof.
\end{proof}

%
%
\section{Proof of Theorem \ref{thm:NeverOptimal}}
\label{app:NeverOptimal}

Let $\pi_{\sunb}$ be the optimal unbiased sampling probability (see Proposition \ref{propo:Unbiased}),
with corresponding weight $w_{\sunb}(\bx) = 1/\pi_{\sunb}(\bx)$. 
Since $\bG(\bx)=\bH(\bx)$ is almost surely bounded (Assumption {\sf A2}),  it follows that $Z(\bx)=\Tr(\bG(\bx)\bH^{-1})$
is bounded as well. Therefore, there exists $\gamma_0>0$ such that, for all $\gamma\in (0,\gamma_0)$,
$\pi_{\sunb}(\bx) = c(\gamma) Z(\bx)^{1/2}$.

For a bounded function 
$\varphi$, consider the alternative weight  $w_{\eps}(\bx) = (1+\eps \varphi(\bx))/\pi_{\sunb}(\bx)$,
which is well defined for all $\eps$ small enough.
Recall that we denote by $\rcoeff(\pi,w;\bQ)$ the asymptotic estimation error coefficient 
for the simple scheme $(\pi,w)$. To linear order in $\eps$, we have
\begin{align}
\rcoeff(\pi_{\sunb},w_{\eps};\bH)- \rcoeff(\pi_{\sunb},w_{0};\bH) = -2\eps\, \E\{\cW(\bx)\, \varphi(\bx)\}\, + o(\eps)\, .\label{eq:ExpansionUnb}
\end{align}
where 
\begin{align}
\cW(\bx):=  
\Tr\Big(\E_{\bx'}\Big(\frac{\bG(\bx')}{\pi_{\sunb}(\bx')}\Big) \bH^{-1}\bH(\bx) \bH^{-1} \Big)-
\Tr\Big(\frac{\bG(\bx)}{\pi_{\sunb}(\bx)}\bH^{-1}\Big)\, .
\end{align}
Assume by contradiction $\rcoeff(\pi_*,w_{\eps};\bH)\ge \rcoeff(\pi_*,w_{-};\bH)$ for every $\varphi$,
$\eps$ such that $w_{\eps}\ge 0$. By Eq.~\eqref{eq:ExpansionUnb}, this
implies $\cW(\bx)=0$ for $\prob$-almost every $\bx$.
Using assumption ${\sf A1}$, and defining $\bM(\bx):= \bH^{-1/2}\bG(\bx)\bH^{-1/2}$, we get
\begin{align}
\cW(\bx) &= \Tr\Big(\E_{\bx'}\Big(\frac{\bM(\bx')}{\pi_{\sunb}(\bx')}\Big) 
\bM(\bx) \Big) - \frac{1}{\pi_{\sunb}(\bx)}\Tr\big(\bM(\bx)\big)\, .
\end{align}
we have $\pi_{\sunb}(\bx)=c(\gamma) \Tr\big(\bM(\bx)\big)^{1/2}$ and therefore 
$\cW(\bx)  =\cW_0(\bx) \cdot \Tr\big(\bM(\bx)\big)^{1/2}/c(\gamma)$, where
\begin{align}
\cW_0(\bx) &=\Tr\Big(\E_{\bx'}\Big(\frac{\bM(\bx')}{\Tr\big(\bM(\bx')\big)^{1/2} }\Big) 
\cdot \frac{\bM(\bx)}{\Tr\big(\bM(\bx)\big)^{1/2} }
\Big) - 1\label{eq:W0MM}\\
& = \Tr\big(\obW\cdot\bW(\bx)\big)-1\, .\nonumber
\end{align}
Here we defined $\bW(\bx):=\bM(\bx)/\Tr\big(\bM(\bx)\big)^{1/2}$, $\obW := \E\bW(\bx)$. 

Since $\cW(\bx)=0$ almost surely, and $\bG(\bx)\neq \bfzero$ almost surely (by assumption {\sf A2}), 
we must have  $\Tr\big(\obW\cdot\bW(\bx)\big)=1$
almost surely, which contradicts the assumption of $\bG(\bx)/\Tr(\bG(\bx)\bH^{-1})^{-1/2}$
not lying on an affine subspace (Assumption {\sf A3}).
%
%
\section{Proof of Theorem \ref{thm:HiDim}}
\label{sec:ProofHiDim}

This proof is based on Gordon Gaussian comparison inequality, following
a well established technique, see \cite{thrampoulidis2015regularized,thrampoulidis2018precise,miolane2021distribution}. Our presentation will be succinct, 
emphasizing novelties with respect to earlier derivations of this type.

Define
\begin{align}
G_{0,i} := \frac{\<\bx_i,\btheta_0\>}{\|\btheta_0\|_2}\, ,
\;\;\;\;\;
G_{s,i} := \frac{\<\bx_i,\bP_0^{\perp}\surrth\>}{\|\bP_0^{\perp}\surrth\|_2}\, ,
\;\;\;\;\;
\bg_i := \bP^{\perp}\bx_i\, ,
\end{align}
where we recall that $\bP_0^{\perp}$ is the projector orthogonal to $\btheta_0$ and $\bP^{\perp}$ is the projector orthogonal to 
$\spn(\btheta_0,\surrth)$. Further define
\begin{align}
\alpha_{0} := \frac{\<\btheta,\btheta_0\>}{\|\btheta_0\|_2}\, ,
\;\;\;\;\;
\alpha_{s} := \frac{\<\btheta,\bP_0^{\perp}\surrth\>}{\|\bP_0^{\perp}\surrth\|_2}\, ,
\;\;\;\;\;
\balphap := \bP_{0,s}^{\perp}\btheta\, ,
\end{align}
With a slight abuse of notation, we can then identify the empirical 
risk
\begin{align}
\hR_N(\alpha_0,\alpha_s,\balphap)= \frac{1}{N}\sum_{i=1}^N s(\beta_0G_{0,i}+\beta_sG_{s,i},U_i) \, 
L(\alpha_0G_{0,i}+\alpha_sG_{s,i}+\<\balphap,\bg_i\> ,y_i)
+\nonumber\\\frac{\lambda}{2}\big(\alpha_0^2+\alpha_s^2+\|\balphap\|^2\big)\, .
\end{align}
We can identify $\balphap$ and $\bg_i$ with $(p-2)$-dimensional
vectors.
For any closed set $\Omega\subseteq\reals^p$, let
\begin{align}
\hR_N(\Omega):=\min\Big\{\hR_N(\alpha_0,\alpha_s,\balphap) \,:\;\;
(\alpha_0,\alpha_s,\balphap)\in \Omega\Big\}\, .
\end{align}
Further define
\begin{align}
\hR^G_N(\Omega)&=\min_{(\alpha_0,\alpha_s,\balphap)\in \Omega}\min_{\bv\in\reals^n} \max_{\bxi\in\reals^N}\hcuL^{(0)}_N(\alpha_0,\alpha_s,\balphap;\bv;\bxi)
\label{eq:LagrangianCvx1}\\
&=\min_{(\alpha_0,\alpha_s,\balphap)\in \Omega} \max_{\bxi\in\reals^N}\min_{\bv\in\reals^n}\hcuL^{(0)}_N(\alpha_0,\alpha_s,\balphap;\bv;\bxi) \, ,\label{eq:LagrangianCvx2}
\end{align}
where 
\begin{align}
\hcuL^{(0)}_N(\alpha_0,\alpha_s,\balphap;\bv;\bxi):=&
\|\balphap\|\<\bg_{\perp},\bxi\>+\|\bxi\|\<\bh,\balphap\>-\<\bv,\bxi\>
+\frac{\lambda}{2}\big(\alpha_0^2+\alpha_s^2+\|\balphap\|^2\big)\\
&+ \frac{1}{N}\sum_{i=1}^N s(\beta_0G_{0,i}+\beta_sG_{s,i},U_i) \, 
L(\alpha_0G_{0,i}+\alpha_s G_{s,i}+v_i ,y_i)\, ,\nonumber
\end{align}
and the identity of the two lines \eqref{eq:LagrangianCvx1},
\eqref{eq:LagrangianCvx2} holds by the minimax theorem.
By an application of Gordon's inequality and using Gaussian concentration
\cite{boucheron2013concentration}, we obtain the following.
\begin{lemma}\label{lemma:BasicGordon}
There exist subgaussian error terms, $\err_{1}(N,\Omega),\err_{2}(N,\Omega)$
(with $\|\err_{2}(N,\Omega)\|_{\psi_2}\le CN^{-1/2}$, such that:
\begin{enumerate}
\item For any closed set $\Omega$:
\begin{align*}
\hR_N(\Omega)\ge \hR^G_N(\Omega) +\err_{1}(N)\, .
\end{align*}
\item For any closed convex set $\Omega$:
\begin{align*}
\hR_N(\Omega) = \hR^G_N(\Omega) +\err_{2}(N)\, .
\end{align*}
\end{enumerate}
\end{lemma}
\begin{proof}
This follows from an application of Gordon's inequality and using Gaussian concentration
\cite{thrampoulidis2015regularized,thrampoulidis2018precise,miolane2021distribution}.
In applying Gordon's inequality, we need to check that the minimizer $\balphap$
lie with high probability in a compact set $B_N$. This is immediate for $\lambda>0$
by strong convexity of $\hR_N$.
\end{proof}

Note that, subject to the constraints $\|\balphap\|=\alphap$, $\|\bxi\|=\mu/\sqrt{N}$,
the minimization over $\balphap$ and maximization over $\bxi$
can be performed before the other optimizations. 
This yields the reduced Lagrangian, with argument $\balpha := (\alpha_0,\alpha_s,\alphap)$:
\begin{align}
\hcuL_N^{(1)}(\balpha;\mu,\bv):=&
-\frac{\|\bh\|}{\sqrt{N}}\alphap\mu+\frac{\mu}{\sqrt{N}}\big\|\alphap\bg_{\perp}-\bv\|
+\frac{\lambda}{2}\|\balpha\|^2\\
&+ \frac{1}{N}\sum_{i=1}^N s(\beta_0G_{0,i}+\beta_s G_{s,i},U_i) \, 
L(\alpha_0G_{0,i}+\alpha_s G_{s,i}+v_i ,y_i)\, .\nonumber
\end{align}

For $A\in \R^2\times\R_{\ge 0}$,  let $\Omega_A:= \{(\alpha_0,\alpha_s,\balphap)\in \reals^p:
(\alpha_0,\alpha_s,\|\balphap\|)\in A \}$, and
write 
\begin{align}
\hR_{\#,N}(A) &:=\hR_{N}(\Omega_A)=
\min\Big\{\hR_N(\alpha_0,\alpha_s,\balphap) \,:\;\;
(\alpha_0,\alpha_s,\|\balphap\|)\in A\Big\}\, ,\\
\hR_{\#,N}^G(A)&:=\hR_{N}^G(\Omega_A)\, .
\end{align}
We then have
\begin{align}
\hR^G_{\#,N}(A)
&=\min_{(\alpha_0,\alpha_s,\alphap)\in A} \max_{\mu\in\reals_{\ge 0}}\min_{\bv\in\reals^n}
\hcuL^{(1)}_N(\alpha_0,\alpha_s,\alphap;\mu,\bv) \, ,
\end{align}

Finally, we can take the limit $N,p\to\infty$. In this limit,
the minimization over $\bv$ is replaced by minimization over a random variable 
$V$. Namely, let $(\cS,\cF,\sP)$ be a probability space on which the random variables 
$(G_0,G_s,G_{\perp},U,Y)$ are defined with the same joint law of 
$(G_{0,1},G_{s,1},g_{\perp,1},U_1,y_1)$. Namely 
$(G_0,G_s,G_{\perp},U)\sim\normal(0,1)^{\otimes 3}\otimes \Unif([0,1])$,
and $Y|G_0,G_s,G_{\perp},U\sim \sP(\cdot\, |\|\btheta_0\|G_0)$.
For $V$ another random variable in the same space, taking values in the extended
real line $\obR$, and letting $\balpha:=(\alpha_0,\alpha_s,\alphap)$, define 
\begin{align}
\hcuL(\balpha;\mu,V):=&
-\frac{1}{\sqrt{\delta_0}}\alphap\mu+\E\{(\alphap G_{\perp}-V)^2\}^{1/2}\mu\\
&+\frac{\lambda}{2}\|\balpha\|^2
+ \E\Big\{ s(\beta_0G_{0}+\beta_sG_{s},U) \, 
L(\alpha_0G_{0}+\alpha_s G_{s}+V ,Y)\Big\}\, .\nonumber
\end{align}

\begin{theorem}\label{thm:HiDimApp}
With the definitions given above, and under the assumptions of
Theorem \ref{thm:HiDim}, the following hold:
\begin{enumerate}
\item For any compact set $A\subseteq\R^2\times \R_{\ge 0}$
\begin{align}
\lim_{N,p\to\infty}
\hR^G_{\#,N}(A) = \hR^G_{\#}(A) 
:=\min_{\balpha\in A} \max_{\mu\in\R_{\ge 0}}
\min_{V\in m\cF}
\hcuL(\balpha;\mu,V) \, .\label{eq:GordonConstrained}
\end{align}
\item For any closed set $A$,
\begin{align}
\lim\inf_{N\to\infty}\hR_{\#,N}(A) \ge \hR^G_{\#}(A)\, .
\end{align}
\item For any closed convex set $A$,
\begin{align}
\lim\inf_{N\to\infty}\hR_{\#,N}(A)  = \hR^G_{\#}(A)\, .
\end{align}
\item 
Further, denoting by
$\balpha^*:=(\alpha^*_0,\alpha^*_s,\alphap^*)$ the minimizer of  
$$\hR^G_{\#}(\balpha):=\max_{\mu\in\R_{\ge 0}}
\min_{V\in m\cF} \hcuL(\balpha;\mu,V)\nonumber$$
Conclusions $(a)$ to $(d)$ of Theorem \ref{thm:HiDim} hold.
\end{enumerate}
\end{theorem}
\begin{proof}
Note that we can rewrite $\hR_{\#,N}(A)=\min\{\hR_{\#,N}(\balpha):\balpha\in A\}$,
where
\begin{align}
\hR_{\#,N}(\balpha)=
\begin{dcases}\min &\;\;\;\;\;
 \frac{\lambda}{2}\|\balpha\|^2+ \frac{1}{N}\sum_{i=1}^N s(\beta_0G_{0,i}+\beta_s G_{s,i},U_i) \,  L(\alpha_0G_{0,i}+\alpha_s G_{s,i}+v_i ,y_i)\\
 \mbox{subj. to} &\;\;\;\;\; \big\|\alphap\bg_{\perp}-\bv\|\le \|\bh\|\alphap\, .
\end{dcases}
\end{align}
Further, this can be written as 
as a function of the joint empirical distribution 
of $\{(G_{0,i},G_{s,i},$ $g_{\perp,i},$ $U_1,y_i,v_i)\}$.
Namely, defining 
\begin{align}
\hp_N := \frac{1}{N}\sum_{i=1}^N \delta_{(G_{0,i},G_{s,i},g_{\perp,i},U_1,y_i,v_i)}\, ,
\end{align}
we have (with $\E_{\hp_N}$ denoting expectation with respect to $\hp_N$)
\begin{align}
\hR_{\#,N}(\balpha) =
\begin{dcases}\min &\;\;\;\;\;
 \frac{\lambda}{2}\|\balpha\|^2+ \E_{\hp_N} \big\{s(\beta_0G_{0}+\beta_s G_{s},U) \,  L(\alpha_0G_{0}+\alpha_s G_{s}+V ,y_i)\big\}\\
 \mbox{subj. to} &\;\;\;\;\; \E_{\hp_N}\big\{[\alphap G_{\perp}-V]^2
\big\}\le \frac{\|\bh\|^2}{N}\alphap^2\, .
\end{dcases}
\end{align}
Let $\hR_{\#}(\balpha)$ be the same quantity, in which minimization over $\bv$
is replaced by minimization over random variables $(G_{0},G_{s},G_{\perp},U,Y,V)$
with $(G_{0},G_{s},G_{\perp},U,Y)$ having the prescribed $N=\infty$
distribution, and $\|\bh\|^2/N$ is replaced by $1/\delta_0$. 
In fact, se define a slight generalization
\begin{align}
\hR_{\#}(\balpha;\eps) =
\begin{dcases}\min &\;\;\;\;\;
 \frac{\lambda}{2}\|\balpha\|^2+ \E \big\{s(\beta_0G_{0}+\beta_s G_{s},U) \,  L(\alpha_0G_{0}+\alpha_s G_{s}+V ,y_i)\big\}\\
 \mbox{subj. to} &\;\;\;\;\; \E\big\{[\alphap G_{\perp}-V]^2
\big\}\le (1-\eps)\alphap^2\, .
\end{dcases}
\end{align}
Let $\hp^*_N$ be the joint distribution of that achieves the
above minimum at finite $N$. By tightness, $\hp_N^*\Rightarrow \hp^*$
along subsequences, and further  the limit satisfies
$\E\{[\alphap G_{\perp}-V]^2\}\le \alphap^2/\delta_0$ by Fatou's. Therefore
\begin{align}
\lim\inf_{N\to\infty}\hR_{\#,N}(\balpha) \ge \hR_{\#}(\balpha;0)\, ,
\end{align}
almost surely.
On the other hand, if $(G_{0},G_{s},G_{\perp},U,Y,V)$ achieves the minimum 
to define $\hR_{\#}(\balpha;\eps)$, $\eps>0$,
let $(G_{0,i},G_{s,i},G_{\perp,i},U_i,y_i,v_i)$,
$i\le N$ be i.i.d.'s vectors from this distribution.
Of course this $\bv$ is almost surely feasible for problem $\hR_{\#,N}(\balpha)$ for all  $N$ large enough. Therefore
\begin{align}
\lim\sup_{N\to\infty}\hR_{\#,N}(\balpha) \le \hR_{\#}(\balpha;\eps)\, ,
\end{align}
Finally, we claim that $\lim_{\eps\to 0+} \hR_{\#}(\balpha;\eps) = \hR_{\#}(\balpha;0)$,
thus yielding 
\begin{align}
\lim_{N\to\infty}\hR_{\#,N}(\balpha) \le \hR_{\#}(\balpha;0)=:\hR_{\#}(\balpha)\, ,
\end{align}
To prove the claim notice that, if $V$ satisfies $\E\big\{[\alphap G_{\perp}-V]^2
\big\}\le \alphap^2$, then $V_{\delta} = (1-\delta) V+\alphap\delta G_{\perp}$
satisfies $\E\big\{[\alphap G_{\perp}-V_{\delta}]^2
\big\}\le (1-\eps)\alphap^2$, with $\eps=2\delta-\delta^2$.
Further, the objective is continuous as $\delta\to 0$ by continuity of $L$, 
whence the claim follows.

Next we claim that $\balpha \mapsto \hR_{\#,N}(\balpha)$ is Lipschitz
continuous for $\ba\in A$, where $A$ is a compact set, on the high probability event
\begin{align}
\cG:=\Big\{\sum_{i=1}^N(G_{0,i}^2+G_{s,i}^2+G_{\perp,i}^2)\le 10 N, \;\;\;
\frac{N}{2}\le\|\bh\|^2\le 2N\Big\}\, .
\end{align}

To prove this claim, note that on this event,
define $s_i=s(\beta_0G_{0,i}+\beta_s G_{s,i},U_i)$,
$a_i(\balpha)=\alpha_0G_{0,i}+\alpha_s G_{s,i}$,  
$ L_i(u)=L_i(u;y_i)$, $r = \|\bh\|/\sqrt{N}$.
such that $\|\bs\|$, $\|\ba\|\le C\sqrt{N}$,
$1/2\le r\le 2$, 
it is sufficient to prove that $\balpha\mapsto F(\balpha)$ is Lipschitz on $A$
where
\begin{align}
F(\balpha)&:= \min\Big\{ H(\balpha,\bv) \;\; :\;\;
\bv\in S(\balpha)\Big\}\, ,\\
H(\balpha,\bv) &:= \frac{1}{N}\sum_{i=1}^N s_i \,  L_i(a_i(\balpha)+v_i)\, ,
\;\;\;\;
S(\balpha):=\{\bv:\; \|\bv-\alphap\bg_{\perp}\|\le \alphap\sqrt{N}\}\, .
\end{align}
On the event $\cG$ above, 
$|H(\balpha_1,\bv_1)-H(\balpha_2,\bv_2)|\le C\|\balpha_1-\balpha_2\|+
 C\|\bv_1-\bv_2\|/\sqrt{N}$ and \\
${\rm dist}(S(\balpha_1),S(\balpha_2))\le C\sqrt{N}\|\balpha_1-\balpha_2\|$\footnote{For $S_1,S_2\subseteq \R^N$,
 we let ${\rm dist}(S_1,S_2):=\sup_{\bx_1\in S_1}\sup_{\bx_2\in S_2}\|\bx_1-\bx_2\|$.}.
The claim follows from these bounds.

Since $\hR_{\#,N}(\balpha) \to \hR_{\#}(\balpha)$ for each $\alpha\in A$,
$\balpha\mapsto \hR_{\#,N}(\balpha)$ is almost surely  Lipschitz continuous with 
Lipshitz constant $C$ independent of $N$, for all $N$ large enough,
we have $\sup_{\balpha\in A}|\hR_{\#,N}(\balpha) - \hR_{\#}(\balpha)|\to 0$,
and therefore, for any compact set $\bA$, 
\begin{align}
\lim_{N\to\infty}\hR_{\#,N}(A) =\hR_{\#}(A)\, .
\end{align}
For $\lambda>0$, both $\hR_{\#,N}$ and $\hR_{\#,N}$ are uniformly strongly convex
and therefore the last claim extend to any closed set $A$.
(Because eventually almost surely $\arg\min \hR_{\#,N}(\balpha)\in B$
for some compact $B$.)

Finally, by introducing a Lagrange multiplier for the
constraint 
\begin{align*}
\hR_{\#}(\balpha)&= 
\min_{V\in m\cF}\max_{\mu\in\R_{\ge 0}}
\hcuL(\balpha;\mu,V) \\
&=\max_{\mu\in\R_{\ge 0}} \min_{V\in m\cF}
\hcuL(\balpha;\mu,V)\, .
\end{align*}
This concludes the proof of point 1 (Eq.~\eqref{eq:GordonConstrained}).

Points 2 and 3 follow from the previous one by applying Lemma 
\ref{lemma:BasicGordon}.

Finally, the proof of point 4 is also straightforward.
Indeed by taking $A=\{\balpha\in \R^2\times\R_{\ge 0}:
\|\balpha-\balpha^*\|\ge \eps\}$ and $A= \R^2\times\R_{\ge 0}$
and applying points 2 and 3 , for arbitrary $\eps>0$, implies
Eq.~\eqref{eq:LimitProj}, thus establishing claim $(c)$ of Theorem \ref{thm:HiDim}.

Claims $(a)$, $(b)$, to $(d)$ of Theorem \ref{thm:HiDim} follow
from the previous one by noting that $\hbtheta_{\lambda}$
is uniformly random, conditional to $\<\btheta_0,\hbtheta_{\lambda}\>$,
$\<\surrth,\hbtheta_{\lambda}\>$, $\|\bP^{\perp}\hbtheta_{\lambda}\|$.
\end{proof}

The proof of Theorem \ref{thm:HiDim} is completed by showing the following.
\begin{lemma}
Under the assumptions of Theorem \ref{thm:HiDim}, we have the following 
equivalent characterizations of $\hR_{\#}(\balpha)$
(where we use the shorthand $S(b)=s(b,U)$ for $U\sim\Unif([0,1])$):
\begin{align}
\hR^G_{\#}(\balpha)&=\max_{\mu\in\R_{\ge 0}}
\min_{V\in m\cF} \hcuL(\balpha;\mu,V)\, ,\label{eq:Gordon-v1}\\
\hR_{\#}(\balpha) &=
\begin{dcases}\min &\;\;\;\;\;
 \frac{\lambda}{2}\|\balpha\|^2+ \E \big\{S(\beta_0G_{0}+\beta_s G_{s}) \,  L(\alpha_0G_{0}+\alpha_s G_{s}+V ,y_i)\big\}\\
 \mbox{\rm subj. to} &\;\;\;\;\; \E\big\{[\alphap G_{\perp}-V]^2
\big\}\le \alphap^2\, ,
\end{dcases}\label{eq:Gordon-v2}\\
\hR^G_{\#}(\balpha)&= \max_{\mu\ge 0}\cuL(\balpha,\mu)\, .
\label{eq:Gordon-v3}
\end{align}
\end{lemma}
\begin{proof}
The equivalence of Eq.~\eqref{eq:Gordon-v1} and \eqref{eq:Gordon-v2} was already 
established in the proof of Theorem \ref{thm:HiDimApp}. 
As for Eq.~\eqref{eq:Gordon-v3}, the equivalence with 
\eqref{eq:Gordon-v2} follows by introducing a Lagrange multipliier for 
the inequality $\E\big\{[\alphap G_{\perp}-V]^2
\big\}\le \alphap^2$, whence $\hR^G_{\#}(\balpha)= \max_{\mu\ge 0}\cuL(\balpha,\mu)$
\begin{align*}
\cuL(\balpha,\mu)=& \frac{\lambda}{2}\|\balpha\|^2
-\frac{1}{2\delta_0}\mu\alphap^2\\
&+\min_{V\in m\cF}\lt\{
\frac{1}{2}\, \mu \cdot \E\big\{(\alphap G_{\perp}-V)^2\big\}
+ \E\Big\{ S(\<\bbeta,\bG\>) \, L(\alpha_0G_{0}+\alpha_s G_{s}+V ,Y)\Big\}
\rt\}\\
=& \frac{\lambda}{2}\|\balpha\|^2-\frac{1}{2\delta_0}\mu\alphap^2
+ \E\Big\{ \min_{u\in\reals}
\Big[S(\<\bbeta,\bG\>) \, L(\alpha_0G_{0}+\alpha_s G_{s}+u ,Y)+\frac{1}{2}\mu(\alphap G_{\perp}-u)^2\Big] \Big\}\, ,
\end{align*}
which yields the desired claim.
\end{proof}

%
%
\section{Proof of Proposition \ref{propo:LinearHighD}}
\label{sec:LinearHighD}

Since we are assuming a perfect surrogate, $\alpha_s=0$, by applying Theorem \ref{thm:HiDim}, we get 
\begin{align}
    \plim_{N,p\to\infty}R(\hbtheta_{\lambda}) &= 
    \E\big\{(Y-\alpha^*_0G_0-\aperp^* G_{\perp})^2\big\}
    \nonumber\\
    & = C_1-2B_1 \alpha_0^*+ A_1(\alpha_0^*)^2+(\aperp^*)^2\, ,
\end{align}
whence the excess error is
\begin{align}
    \plim_{N,p\to\infty}\min_{\btheta}R_{\sexc}(\hbtheta_{\lambda}) &=  
   \lt(\frac{B_1}{A_1}-\alpha_0^*\rt)^2+(\aperp^*)^2\, 
   .\label{eq:ExcessSolve}
\end{align}
Simplifying the Lagrangian \eqref{eq:LagrangianLS}
in the case of perfect surrogate, we get 
(recalling that $\delta=\delta_0\gamma$)
\begin{align}
\frac{1}{\gamma}\cuL_{\sls}(\alpha_0,\aperp)&= 
\frac{1}{2}\left(
\sqrt{C_{\pi}-2B_{\pi} \alpha_0+ A_{\pi}\alpha_0^2+\aperp^2}
-\frac{\aperp}{\sqrt{\delta}}\right)_+^2+
\frac{\lambda}{2\gamma}\|\balpha\|^2_2\nonumber\\
&=:\frac{1}{2} \, G(\balpha)^2+
\frac{\lambda}{2\gamma}\|\balpha\|^2_2
\, .\nonumber
\end{align}
In the limit  $\lambda\to 0$, $\balpha^*=(\alpha_0^*, \aperp^*)$ 
is given by
\begin{align}
\balpha^* = \argmin\Big\{\|\balpha\|^2: \; \balpha\in \argmin_{\ba\in\reals\times\reals_{\ge 0}}G(\ba)^2\Big\}\, .
\end{align}

Depending on the value of $\delta$, the solution of this problem
is achieved in different domains of the plane:
\begin{itemize}
\item For $\delta>1$, 
$\argmin_{\ba\in\reals\times\reals_{\ge 0}}G(\ba)^2$
is uniquely achieved when $G(\balpha)>0$,
and hence satisfies $\nabla G(\balpha)=0$, $G(\balpha)>0$.
Simple calculus yields
\begin{align}
    \alpha^*_0&=\frac{B_{\pi}}{A_{\pi}}\, ,\\
    \aperp^* &=\sqrt{\frac{1}{\delta-1}\Big(C_{\pi}-\frac{B_{\pi}^2}{A_{\pi}}\Big)}\, .
\end{align}
\item For $\delta<1$,  $\argmin_{\ba\in\reals\times\reals_{\ge 0}}G(\ba)^2$ is  $S_0:=\{\balpha:G(\balpha)=0\}$, and
it is easy to see that (for $r_0 := C_{\pi}-B_{\pi}^2/A_{\pi}$):
\begin{align}
S_0= \Big\{(\alpha_0,\aperp)\reals\times\reals_{\ge 0}: \;\; Q(\balpha):=\Big(\frac{1}{\delta}-1\Big)\aperp^2-A_{\pi}
\Big(\alpha_0-\frac{B_{\pi}}{A_{\pi}}\Big)^2- r_0\ge 0\Big\}\, .
\end{align}
Hence we $\balpha_*$ solves (for a Lagrange multiplier $\xi$)
\begin{align}
\nabla Q(\balpha) =\xi \balpha\, ,\;\; Q(\balpha)=0\, ,
\end{align}
which are easily solved.
\end{itemize}
The proof is completed by substituting this solution in Eq.~\eqref{eq:ExcessSolve}.
%
%
\section{Some useful formulas for binary classification}

In line with the model introduced in the main text, 
we consider $y_i\in\{+1,-1\}$ and
\begin{align}
\prob(y_i=+1|\bx_i)= f(\<\btheta_0,\bx_i\>\, .
\end{align}
(In other words, $\sP(+1|z) = 1-\sP(-1|z) = f(z)$.
We use logistic loss
\begin{align}
L(y;z) = -yz+\log(1+e^z)\, .
\end{align}
Formulas below are obtained by specializing the 
results of Section \ref{sec:HighDim}.

\paragraph{Unbiased subsampling.} The Lagrangian takes the 
form
\begin{align}
\cuL(\balpha,\mu)&:= \frac{\lambda}{2}\|\balpha\|^2-\frac{1}{2\delta_0}\mu\aperp^2\\
&+
\E\Big\{f(\|\btheta_0\|G_0)\min_{u\in\reals}\big[L(\alpha_0 G_0+\alpha_s G_s+u;+1)
+\frac{1}{2}\mu \, \pi(\<\bbeta,\bg\>)\big(u-\aperp G\big)^2\big]
\nonumber\\
&+ (1-f(\|\btheta_0\|G_0))\min_{u\in\reals}\big[L(\alpha_0 G_0+\alpha_s G_s+u;-1)
+\frac{1}{2}\mu \, \pi(\<\bbeta,\bg\>)\big(u-\aperp G\big)^2\big]\Big\}\, .\nonumber
\end{align}

\paragraph{No-reweigthing data selection.} In this case  Lagrangian reduces to
\begin{align}
\cuL(\balpha,\mu)&:= 
\frac{\lambda}{2}\|\balpha\|^2-\frac{1}{2\delta_0}\mu\aperp^2
\label{eq:NR-Lagrangian-binary}\\
&+\E\Big\{f(\|\btheta_0\|G_0)\pi(\<\bbeta,\bg\>)\min_{u\in\reals}\big[L(\alpha_0 G_0+\alpha_s G_s+u;+1)
+\frac{1}{2}\mu\, \big(u-\aperp G_{\perp}\big)^2\big]
\nonumber\\
&+(1-f(\|\btheta_0\|G_0))\pi(\<\bbeta,\bg\>)\min_{u\in\reals}\big[L(\alpha_0 G_0+\alpha_s G_s+u;-1)
+\frac{1}{2}\mu\, \big(u-\aperp G_{\perp}\big)^2\big]\Big\}\, .
\nonumber
\end{align}

\paragraph{Misclassification error.}  For $L_{\stest}(y;z)= \bfone(yz<0)$:
\begin{align}
R(\lambda;\btheta) & = \frac{1}{2}-\frac{1}{2}\E \Big\{\big(2f(\|\btheta_0\|_2G)-1\big)
\big(2\Phi(q\, G)-1\big)
\Big\}\, ,\;\;\; q:= \frac{\alpha_0}{\sqrt{\alpha_s^2+\alpha_{\perp}^2}}\, .
\end{align}
%

%
%
%
%
%
%
%

\end{document}